\documentclass{article} 
\usepackage{iclr2027_conference,times}

\usepackage{graphicx} 
\usepackage{natbib}
\usepackage{subcaption}
\usepackage{xcolor}
\usepackage{hyperref}
\usepackage{booktabs}
\usepackage{tabularx}
\usepackage{makecell}

\usepackage{algorithm}
\usepackage{algpseudocode}
\usepackage{thmtools}
\usepackage{thm-restate}

\usepackage{url}
\usepackage{paralist}
\usepackage{amsthm}
\usepackage{amsmath,amsfonts,bm}
\usepackage{mathtools}

\usepackage{cleveref}
\usepackage{comment}
\usepackage{todonotes}
\usepackage{placeins}
\usepackage{tcolorbox}
\usepackage{multirow}
\usepackage{longtable}

\def\eqref#1{equation~\ref{#1}}

\def\1{\bm{1}}

\DeclareMathAlphabet{\mathsfit}{\encodingdefault}{\sfdefault}{m}{sl}
\SetMathAlphabet{\mathsfit}{bold}{\encodingdefault}{\sfdefault}{bx}{n}

\newcommand{\R}{\mathbb{R}}

\DeclareMathOperator{\sign}{sign}

\newcommand{\vecop}{\operatorname{vec}}
\newcommand{\polar}{\operatorname{polar}}

\newtheorem{theorem}{Theorem}
\newtheorem{lemma}{Lemma}
\newtheorem{corollary}{Corollary}
\newtheorem{definition}{Definition}
\newtheorem{proposition}{Proposition}

\newtheorem{assumption}{Assumption}

\usepackage{xargs}                      

\usepackage{mathrsfs}

\newcommandx{\del}[2]{\dfrac{\mathrm{d} #1}{\mathrm{d} #2}}

\newcommand{\lp}{\left(}
\newcommand{\rp}{\right)}

\newcommand{\diag}{\mathrm{diag}}

\definecolor{darkblue}{RGB}{0,0,139} %

\Crefname{equation}{Eq.}{Eqs.}
\Crefname{figure}{Fig.}{Figs.}
\Crefname{appendix}{App.}{App.}
\Crefname{section}{Sec.}{Sec.}
\Crefname{assumption}{Assumption}{Assumptions}
\usepackage{xcolor}

\newtheorem*{theorem*}{Theorem}

\newcommand{\important}[1]{\textcolor{blue}{#1}}

\definecolor{applegreen}{rgb}{0.01, 0.75, 0.24}
\definecolor{brightmaroon}{rgb}{0.66, 0.13, 0.24}
\definecolor{warningyellow}{rgb}{0.86, 0.56, 0.14} 
\definecolor{skyblue}{rgb}{0.15, 0.4, 0.91} 

\title{Optimizer Geometry Transforms the Spectral Low-Rank Bias in Deep Matrix Factorization}
\title{Every Mode Has Its Time: Optimizer Geometry in Deep Matrix Factorization}
\title{Beyond the Final Rank: Optimizer Dynamics in Deep Matrix Factorization}
\title{Geometry Sets the Pace: Spectral Learning Dynamics in Deep Matrix Factorization}
\title{Optimizer Geometry Sets the Pace: Spectral Learning Dynamics in Matrix Factorization}

\author{Mahalakshmi Sabanayagam $\qquad$ Simon Lucey \\
Australian Institute for Machine Learning, Adelaide University \\
\texttt{\{mahalakshmi.sabanayagam, simon.lucey\}@adelaide.edu.au} \\
}

\iclrfinalcopy 
\begin{document}

\maketitle

\begin{abstract}

Recent successes of matrix- and curvature-based optimizers have renewed interest in how update geometry shapes learning. These methods normalize or precondition updates, changing how different components progress during training. In deep matrix factorization, the geometry that slows gradient descent (GD) favors low-rank solutions by delaying the emergence of small singular modes. This raises the question of what remains of that spectral bias when normalization or curvature correction weakens or removes the slowdown. We study how optimizer geometry and depth jointly govern this behavior through a common framework for singular mode dynamics. Under explicit balance and alignment assumptions, we derive birth, saturation, and decay laws for Euclidean GD, coordinate-wise updates of SignGD and an instantaneous Adam approximation, spectral updates of Muon and cumulative Shampoo, and a block curvature model of K-FAC.  The resulting picture is not a simple ordering from stronger to weaker low-rank bias.  To highlight, SignGD and ideal Muon eliminate the divergent birth barrier and drive unsupported modes to zero in finite time. Cumulative Shampoo initially retains GD's depth-dependent barrier, then accumulated gradients produce a catch-up phase while making previously active modes increasingly persistent. Undamped K-FAC cancels the factorization-induced slowdown while preserving the ordering of the target singular values, whereas positive damping introduces a spectral threshold below which the slow GD phase laws reappear. These results give normalization, accumulated state, damping, and depth a direct interpretation as controls determining when modes emerge, persist, and disappear during training.

\end{abstract}

\section{Introduction}
\label{sec:intro}

For much of modern deep learning, large scale optimization was dominated by SGD and Adam-like coordinate-wise methods, while matrix preconditioning and approximate curvature methods remained comparatively limited in practice. Their additional memory, numerical, and systems requirements made their advantages difficult to realize at scale \citep{martens2015optimizing,gupta2018shampoo}. This landscape is beginning to change. Recent foundation models with multi-billion and trillion parameters \citep{team2025kimi,liu2025muon,khona2026soap,ji2026quantization} are effectively trained with spectral optimizers such as Muon \citep{jordan2024muon} 
and SOAP \citep{vyas2025soap} at substantially large model, data, and batch scales.
%
This empirical shift raises a theoretical question. The implicit bias analysis of gradient descent (GD) shows that Euclidean geometry slows the progress of weak singular modes and thereby promotes low-rank solutions in deep matrix factorization \citep{arora2019implicit,li2021towards}.
Whereas spectral normalization and curvature correction are designed to reduce precisely this kind of anisotropy, in idealized limits they can weaken or even cancel the singular value dependent slowdown.
Their recent success therefore raises the question of whether GD's spectral bias induced by depth is lost, preserved in a weaker form, or replaced by a qualitatively different mechanism of mode selection.
More broadly, we ask:

\begin{quote}
    \emph{When do modern optimizers merely accelerate a common learning trajectory, and when do they change which spectral components emerge, persist, or disappear?}
\end{quote}


We investigate these questions in deep linear networks, where depth changes the parameterization and optimization dynamics without changing the end-to-end linear functions. This makes it a particularly clean setting for separating optimizer-induced dynamics from parameterization-induced bias. Our theoretical analysis focuses on deep matrix factorization, a canonical special case in which the evolution of singular modes can be characterized explicitly.
Specifically, in full observation regression, it exposes transient singular mode learning \citep{saxe2013exact,soudry2018implicit}, while in underdetermined matrix sensing and completion, it permits direct study of which interpolating matrix is selected \citep{arora2019implicit,li2021towards,rhee2026muon,parviz2026reassessing}.


%
In contrast to the Euclidean geometry of GD,
non-Euclidean optimizers transform the factorization-induced geometry in qualitatively different ways. 
SignGD and Adam apply coordinate-wise adaptive scaling \citep{bernstein2018signsgd,kingma2014adam}; Muon acts directly on the spectrum of matrix updates \citep{jordan2024muon} while Shampoo introduces history-dependent matrix preconditioning \citep{gupta2018shampoo}; and K-FAC attempts to compensate for the parameterization through curvature using approximate natural gradients \citep{martens2015optimizing}. 
In our study, we categorize these geometries based on their updates into \emph{Euclidean} (GD), \emph{coordinate-wise} (SignGD and Adam), \emph{spectral} (Muon and Shampoo) and \emph{curvature} based (K-FAC). 
These methods can reach nearly identical training losses while producing weights with different spectra, representations, and performance.
Recent studies have begun to identify optimizer-dependent feature and simplicity biases, including differences between Adam and GD in nonlinear models \citep{cattaneo2023implicit,zhang2024implicit,xie2024implicit} and between spectral and Euclidean updates in deep linear networks \citep{dragutinovic2026use}. 
These findings further motivate the central question we posed, and answering this requires distinguishing how quickly modes emerge from how they settle or disappear, since an optimizer can affect these phases differently.



%


\begin{figure}[t]
    \centering
    \begin{subfigure}{0.32\textwidth}
        \centering
        \includegraphics[width=\linewidth]{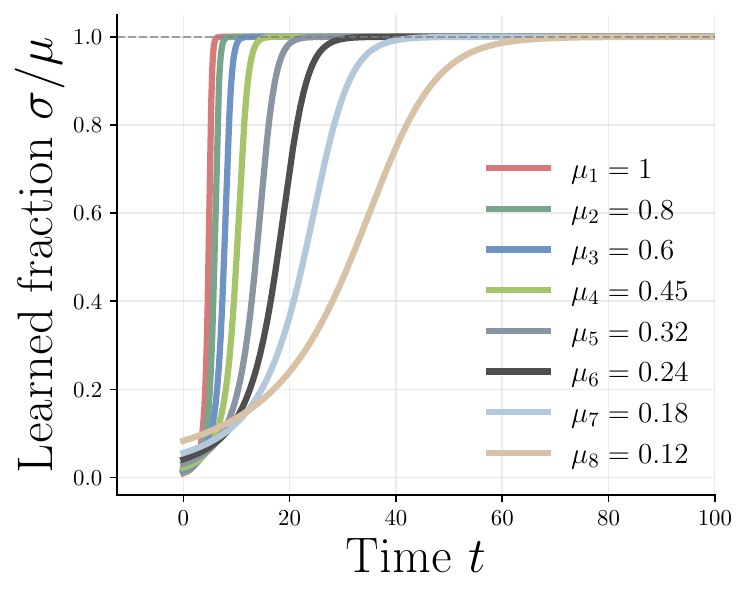}
        \vspace{-0.3cm}
        \caption{GD}
        \label{fig:gd_intro}
    \end{subfigure}
    \begin{subfigure}{0.32\textwidth}
        \centering
        \includegraphics[width=\linewidth]{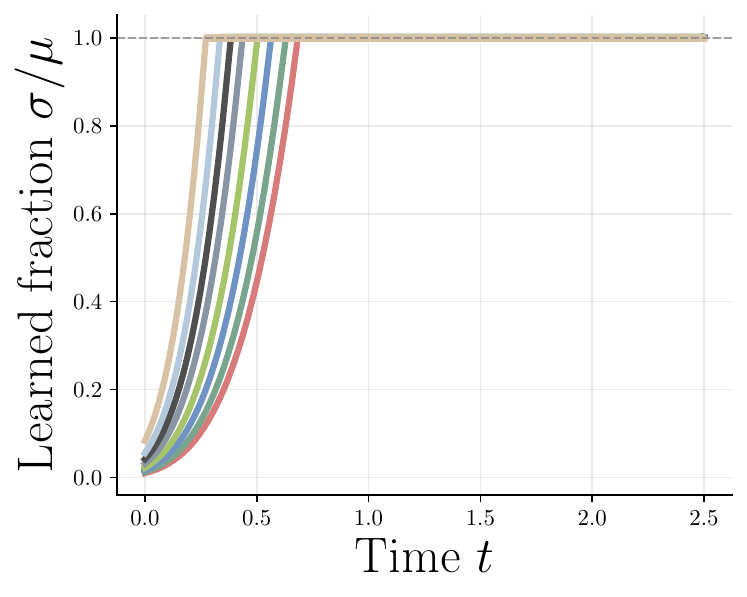}
        \vspace{-0.3cm}
        \caption{Muon}
        \label{fig:muon_intro}
    \end{subfigure}
    \begin{subfigure}{0.32\textwidth}
        \centering
        \includegraphics[width=\linewidth]{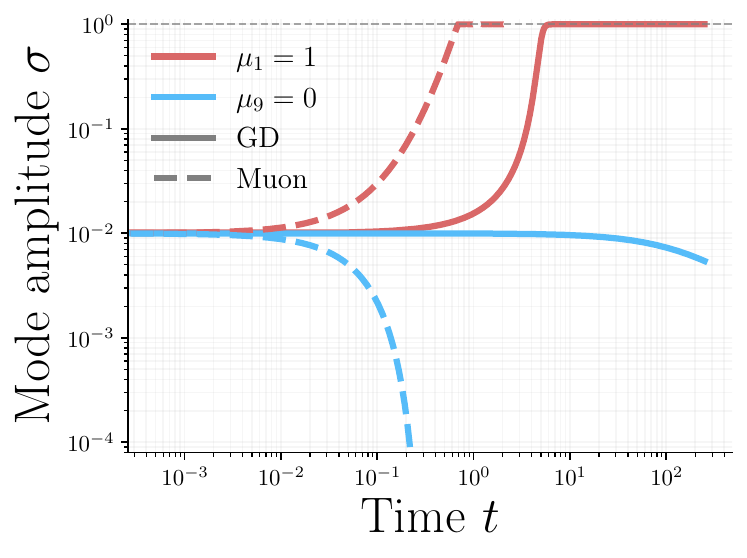}
        \vspace{-0.3cm}
        \caption{GD vs Muon: Birth and decay}
        \label{fig:muon_intro}
    \end{subfigure}
    \vspace{-0.2cm}
    \caption{\textbf{Optimizer geometry sets the pace of spectral learning.}
    In balanced aligned deep matrix factorization, GD and Muon produce distinct spectral learning trajectories (experimental details in \Cref{app:exp_settings}).
\textbf{(a)} Under GD, larger target modes are learned earlier, while smaller modes remain suppressed for longer.
\textbf{(b)} Under Muon, factor amplitudes grow at a common speed, allowing smaller targets to reach full recovery earlier.
\textbf{(c)} From the same initial amplitude $0.01$, Muon accelerates the growth of a positive target mode (red) and, importantly, removes the zero target mode in finite time (blue).
GD instead exhibits delayed growth and slow decay.
Thus, faster emergence of weak modes does not imply weaker removal of unsupported modes. Optimizer geometry determines which spectral components emerge, persist, and disappear. Other optimizers are discussed in later sections.
    }
    \label{fig:intro_new}
    \vspace{-0.4cm}
\end{figure}

Our approach separates the geometry induced by the factorization from the update rule and state maintained by the optimizer.
The product dynamics first express how each layer updates 
combine into the dynamics of the effective matrix. 
Under explicit balance, target alignment, and state compatibility assumptions, this description reduces to a common equation for each singular mode.
The factorization determines how factor dynamics changes the mode amplitude, while the optimizer determines the update through its normalization, accumulated statistics, and damping.
This gives a controlled setting in which these mechanisms can be compared without attributing differences to changing singular directions.
%
Our analysis separates three aspects of spectral learning that are often conflated: the \emph{birth} of initially small nonzero modes to reach a prescribed amplitude, \emph{saturation} near their targets, and \emph{decay} of unsupported modes whose target is zero.
These phases need not respond to an optimizer in the same way.
A method can permit rapid emergence while retaining a slow decay tail, and accumulated gradient history can change the rate at which an already learned component settles.
Consequently, the emergence of small modes alone does not describe the spectral bias of an entire trajectory as studied in prior works. Our results make these distinctions explicit as shown in \Cref{fig:intro_new}.  
Moreover, it also brings out how optimizer state parameters, including adaptive stabilization and curvature damping, interpolate between distinct rank selection regimes, and verifies that these differences persist in underdetermined matrix factorization and regression.
\textbf{Our contributions} are: \\ 
    \textbf{ $\mathbf{(1)}$ A common framework for spectral learning.}
    We express effective matrix dynamics through the layer updates and derive an exact mode reduction under explicit balance and alignment 
    assumptions. 
    The reduction separates the contribution of the factorization from the optimizer update states.

    \textbf{$\mathbf{(2)}$ Phase laws revealing how optimizer settings shape the bias of depth.}
    We derive birth, saturation, and decay laws that distinguish normalization, accumulated history, and damping.
    GD and SignGD dynamics establish contrasting reference behaviors. Instantaneous Adam identifies how a positive stabilizer restores slow dynamics at small gradient magnitudes. Ideal Muon accelerates both emergence and removal of modes in finite time. 
    For cumulative Shampoo, an exact accumulator identity quantifies the transition from initialization dominated to history dominated dynamics. 
    For the specified K-FAC block model, an exact birth time decomposition separates an undamped contribution from a damping-weighted GD contribution.
    Together, these results identify when optimizer geometry preserves, modifies, or cancels the suppression of small modes induced by depth.

    \textbf{$\mathbf{(3)}$ Numerical validation of phase laws and regime transitions.}
    Controlled experiments in deep matrix factorization and regression 
    verify the derived phase laws. 
    These experiments demonstrate that settings commonly associated with numerical stability can change relative spectral progress, determining which components emerge, persist, and disappear during training.

\vspace{-0.15cm}
\section{Related Work}
\label{sec:related_work}
\vspace{-0.1cm}
\paragraph{Optimization geometry and solution bias.}  
Optimization geometry influences which solution is selected among those that fit the training data \citep{wilson2017marginal,gunasekar2018characterizing}.
For Adam, implicit bias analyses establish margin characterizations in separable classification under specified optimizer and learning rate assumptions \citep{zhang2024implicit,gronich2026implicit}.
Natural gradient provides another illustration of the role of geometry: under suitable rank conditions, its effective matrix dynamics can eliminate the dependence on depth that arises under Euclidean gradient flow \citep{kerekes2021depth,bernacchia2018exact}.
These results motivate examining both the symmetries respected by an optimizer and the trajectory it induces.
Our framework complements this perspective by expressing changes in unobserved matrix directions through accumulated optimizer induced motion, while using mode dynamics to isolate the effects of different optimizers. 

\paragraph{Training dynamics in matrix factorization.}
Analyses of gradient flow and gradient descent connect spectral dynamics to low rank recovery and incremental rank growth under specified initialization and measurement assumptions \citep{gunasekar2017implicit,arora2019implicit,li2021towards,jin2023understanding}.
While depth strengthens the tendency of low rank bias, it is not always equivalent to minimizing a single Schatten norm, indicating that trajectory-level geometry is essential \citep{razin2020implicit}.  
Recent Muon analyses characterize how polar normalization changes spectral learning and simplicity bias in deep linear networks and linear attention \citep{dragutinovic2026use}, and establish approximately uniform singular value growth under assumptions on LoRA factorization \citep{kang2026uniform}.
For fully observed bilinear matrix factorization, \citet{rhee2026muon} derive uniform factor singular value growth after alignment, 
and analyze oscillations caused by finite optimization steps.
Complementary empirical comparisons show that Muon's relative performance depends on the factorization problem and hyperparameter tuning \citep{parviz2026reassessing}.
Shampoo has been interpreted as an adapted version of spectral descent that retains information about stochasticity and the parameter trajectory \citep{eschenhagen2026clarifying}.
Building on these results, we compare mode birth, decay, and saturation within a common balanced aligned model.
While most analyses study one optimizer in isolation, our analysis includes SignGD and simplified Adam,
derives an exact accumulator identity for cumulative Shampoo, and quantifies the effect of additive damping in K-FAC.
We give a broader discussion on related works in \Cref{app:related_work}.


\vspace{-0.15cm}
\section{Preliminaries}
\label{sec:prelims}
\vspace{-0.1cm}
\paragraph{Notations.}
For matrices $X$ and $Y$ of the same dimensions, $\langle X,Y\rangle_F=\operatorname{tr}(X^\top Y)$ denotes the Frobenius inner product.
We use $\|\cdot\|_F$ for the Frobenius norm of a matrix, which reduces to the Euclidean norm for a vector, and matrix $I_d$ denotes the identity matrix of size $d\times d$.
The notation $\nabla_W f(W)$ denotes the gradient of $f$ with respect to $W$. 
The symbols $\odot$ and $\oslash$ denote entrywise multiplication and division, respectively, with $H^{\odot 2}=H\odot H$.
In entrywise expressions, square roots are also applied entrywise.
A dot denotes differentiation with respect to time, so $\dot a=da/dt$.

\paragraph{Deep matrix factorization.}

Let $W_\ell\in\R^{d_\ell\times d_{\ell-1}}$ denote the parameters of layer $\ell$ in a depth $L$ linear network with $L\geq2$. The effective matrix is
\begin{align}
W=W_L W_{L-1} \cdots W_2 W_1. 
\label{eq:eff_W}
\end{align}
We assume $d_\ell\geq\min(d_0,d_L)$ for every hidden layer, so that the architecture does not impose a rank bottleneck.
Let $S_1,\ldots,S_m\in\R^{d_L\times d_0}$ be known measurement matrices and let $y=(y_1,\ldots,y_m)\in\R^m$ be the observed values.
We define $\mathcal{A}(W)=(\langle S_j,W\rangle_F)_{j=1}^m$, and 
$$
\mathcal{L}(W)=\frac12\sum_{j=1}^m((\mathcal{A}(W))_j-y_j)^2,
\qquad
G=\nabla_W\mathcal{L}(W)
=\sum_{j=1}^m((\mathcal{A}(W))_j-y_j)S_j.
$$
The model is trained by minimizing the mean squared loss $\mathcal{L}(W)$. In noiseless recovery, $M\in\R^{d_L\times d_0}$ denotes the unknown target matrix and $y_j=\langle S_j,M\rangle_F$.
Matrix completion corresponds to selecting individual entries of $M$ through the measurement matrices.
Full observation gives $\mathcal{L}(W)=\|W-M\|_F^2/2$.
The gradient with respect to the effective matrix is $G$.
In particular, $G=W-M$ under full observation.
Define the prefix and suffix products
\begin{align}
A_\ell=W_{\ell-1}\cdots W_1,
\qquad
B_\ell=W_L\cdots W_{\ell+1},
\qquad
A_1=I_{d_0},
\quad B_L=I_{d_L}.
\label{eq:AB_l}
\end{align}
The chain rule gives the layer gradient
\begin{align}
H_\ell=\nabla_{W_\ell}\mathcal{L}(W)
=B_\ell^\top G A_\ell^\top. 
\label{eq:grad_l}
\end{align}
For full observation, $\mathcal{L}(W)=\|W-M\|_F^2/2$ has the unique minimizer $M$ for any depth.
Consequently, different spectra at finite time describe transient learning and cannot establish different interpolating solutions. However, the training trajectories differ across optimizers.
Therefore, we focus on how depth and optmizer geometry affect the spectrum before convergence in the full observation setting and distinguish finite time spectral bias from asymptotic solution selection throughout.

\begin{table}[t]
\centering
\small
\renewcommand{\arraystretch}{1.7}
\begin{tabular}{@{}lll@{}}
\toprule
\textbf{Optimizer} & ${Z_\ell}$ & \textbf{Update category} \\
\midrule
GD
& $H_\ell$
& Euclidean \\

SignGD \citep{bernstein2018signsgd}
& $\sign(H_\ell)$
&  Coordinate-wise \\

Adam \citep{kingma2014adam}
& \makecell{$ {m_\ell}\oslash({\sqrt{ v_\ell}+\epsilon})$ \,\, with \\   $\tau_m\dot m_\ell=H_\ell-m_\ell, \,\,
\tau_v\dot v_\ell=H_\ell^{\odot2}-v_\ell$} 
& Coordinate-wise (adaptive) \\ 

Ideal Muon \citep{jordan2024muon}
& $U_\ell V_\ell^\top$ \,\, where \,\, $H_\ell = U_\ell \Sigma_\ell V_\ell^\top$ 
& Spectral \\ 

Shampoo \citep{gupta2018shampoo}
& \makecell{$C_\ell^{-1/4}H_\ell D_\ell^{-1/4}$ \, \, with \,\, $\lambda_{\mathrm{Sh}} >0, $ \\ 
$C_\ell(t)=\lambda_{\mathrm{Sh}}I+
\int_0^t H_\ell(u)H_\ell(u)^\top\,du,$ \\
$
D_\ell(t)=\lambda_{\mathrm{Sh}}I+
\int_0^t H_\ell(u)^\top H_\ell(u)\,du
$ 
}
& Spectral \\ 


K-FAC \citep{martens2015optimizing}
& \makecell{$\operatorname{vec}(Z_\ell) = (\mathcal F_\ell+\lambda_{\mathrm{KF}} I )^{-1}\operatorname{vec}(H_\ell),$ \\ 
$\mathcal F_\ell= (A_\ell A_\ell^\top) \otimes (B_\ell^\top B_\ell)$, \,\,
$\lambda_{\mathrm{KF}}\geq 0$}
& Curvature based \\ 
\bottomrule
\end{tabular}
\caption{Layer update directions in the flow $\dot W_\ell=-Z_\ell$, where $H_\ell=B_\ell^\top G A_\ell^\top$.
SignGD and Adam applies coordinate-wise updates, Muon uses polar normalization, Shampoo uses cumulative statistics, and K-FAC uses additive block damping. Detailed account of these optimizers is in \Cref{app:optimizer}.}
\vspace{-0.2cm}
\label{tab:optimizer_conventions}
\end{table}

\paragraph{Optimizer conventions.}
We study deterministic flows driven by the full gradient, without weight decay or explicit regularization. For each layer, the continuous time dynamics of the matrix $W_\ell$ is 
$$
\dot W_\ell(t)=-Z_\ell(t),
$$
where $t$ is continuous time, a dot denotes its derivative, 
The flow uses unit learning rate, with any constant common learning rate absorbed into time and the stated memory times. All moment matrices have the same dimensions as $W_\ell$.
The equations for $Z_\ell$ for all the optimizers 
analyzed in this paper are given in \Cref{tab:optimizer_conventions} and each optimizer is discussed in detail in \Cref{app:optimizer}. We present the exact effective matrix dynamics in \Cref{prop:effective_dyn} (proof in \Cref{app:eff_W_dyn}) which is the main object in our analysis.

\begin{proposition}[Exact effective matrix $W$ dynamics]
\label{prop:effective_dyn}
Using the definitions of $A_\ell$ and $B_\ell$ from \Cref{eq:AB_l}, for every layer trajectory satisfying $\dot W_\ell=-Z_\ell$, the effective matrix satisfies
$$
\dot W=-\sum_{\ell=1}^L B_\ell Z_\ell A_\ell.
$$
\end{proposition}

\section{A Common Framework for Optimizer Geometry}
\label{sec:framework}
We develop a common framework to compare how optimizer geometry controls the emergence, growth, and decay of singular modes. 
To derive explicit phase laws, we adopt the training dynamics analysis that connects low rank learning to the evolution of the effective matrix spectrum and consider a balanced, aligned setting in which singular directions remain fixed \citep{gunasekar2017implicit,arora2019implicit,li2021towards}.
This construction follows the exact mode analysis of deep linear networks \citep{saxe2013exact} and isolates the effects of different optimizer geometries. 

We consider full observation with
$\mathcal{L}(W) =\frac12\|W_L \cdots W_1 -M\|_F^2$, and the gradient $G=W-M.$
While the minimizer is unique, this setting fixes the target of every mode and allows its learning time to be compared across optimizers. We make the following assumptions.

\begin{assumption}[Balanced initialization]
\label{assump:balancedness}
The factors are balanced at initialization:
$
W_\ell(0)W_\ell(0)^\top
=
W_{\ell+1}(0)^\top W_{\ell+1}(0),
\text{ for }
\ell=1,\ldots,L-1.
$
\end{assumption}

Balanced initialization is a standard condition in the analysis of deep matrix factorization \citep{arora2019implicit,rhee2026muon}.
For Euclidean gradient flow,  $W_\ell(t) W_\ell(t)^\top-W_{\ell+1}(t)^\top W_{\ell+1}(t)$ is conserved for all $t$, so initial balance persists throughout training.
For the other optimizers, preservation follows within the aligned construction from the compatibility condition stated below. 

\begin{assumption}[Alignment with the target]
\label{assump:aligned_modes}
There exist orthogonal matrices $Q_0,\ldots,Q_L$ with  $Q_\ell \in\mathbb{R}^{d_\ell \times d}$ where $d = \min(d_0, d_L)$ such that 
$$
M=Q_L\operatorname{diag}(\mu_1,\ldots,\mu_d)Q_0^\top,
\qquad
\text{ and } \qquad
W_\ell(0) =Q_\ell \operatorname{diag}(a_{\ell 1}(0),\ldots,a_{\ell d}(0)) Q_{\ell-1}^\top. 
$$
where $\mu_i\geq0$ and $a_{\ell i}(0)>0$.
$Q_0,\ldots,Q_L$ remain the reference bases for the mode analysis.
\end{assumption}

Together, \Cref{assump:balancedness,assump:aligned_modes} imply $a_{\ell i}(0)^2=a_{\ell+1,i}(0)^2$, and positivity therefore gives a common initial amplitude
$a_{\ell i}(0)=a_i(0)$ across layers.
While equal layer strengths and decoupled singular directions underlie the explicit learning curves of \citet{saxe2013exact},
fixing the directions gives every optimizer the same initial spectral configuration, so differences in mode evolution can be attributed to their update geometry and state.

\begin{assumption}[Compatible optimizer state and updates]
\label{assump:compatible_updates}
All layers use the same optimizer hyperparameters and have identical initial states when expressed in their respective mode bases.
The joint evolution of factors $W_\ell$ and optimizer states preserves the aligned form.
In particular, whenever the factor amplitudes and mode states agree across layers, the update directions satisfy
$$
Z_\ell
=
Q_\ell\operatorname{diag}(z_1,\ldots,z_d)Q_{\ell-1}^\top,
\qquad
\ell=1,\ldots,L,
$$
\end{assumption}

This condition makes the reduction a property of a specified optimizer flow.
It is satisfied by GD, Muon, and cumulative Shampoo 
in the aligned setting.
For Adam and SignGD, diagonality must hold in the actual parameter coordinates, which can be imposed by taking $Q_\ell=I$ for every $\ell$.
For K-FAC, we use the full observation block curvature model with common additive damping.
These choices retain the moment and preconditioner states responsible for the phase behavior studied below and shown in \Cref{app:aligned-mode-dynamics}.
Under \Cref{assump:compatible_updates}, equal initial amplitudes have equal updates across layers.
Consequently, the factors remain balanced and admit the representation
$$
W_\ell(t)
=
Q_\ell\operatorname{diag}(a_1(t),\ldots,a_d(t))Q_{\ell-1}^\top.
$$
The following reduction proved in \Cref{app:aligned-mode-dynamics} applies on intervals where the relevant amplitudes remain positive. At a zero crossing, any continuation is specified in the corresponding optimizer result.

\begin{theorem}[Exact dynamics of aligned modes]
\label{prop:aligned_dynamics}
Under \Cref{assump:balancedness,assump:aligned_modes,assump:compatible_updates}, define
$
\sigma_i=a_i^L,
q=1-\dfrac1L,
p=2q=2-\dfrac2L,
r_i=\sigma_i-\mu_i.
$
Then $W$ and every layer gradient $H_\ell$ has the form
$$
W=Q_L\operatorname{diag}(\sigma_1,\ldots,\sigma_d)Q_0^\top, \qquad
H_\ell
=
Q_\ell\operatorname{diag}(g_1,\ldots,g_d)Q_{\ell-1}^\top,
\,\, \text{ where }
g_i=r_i\sigma_i^q.
$$
The effective singular amplitudes satisfy $\dot\sigma_i=-L\sigma_i^qz_i.$
Here, $i$ labels a fixed mode and does not require the amplitudes to remain sorted by magnitude.
\vspace{-0.2cm}
\end{theorem}

The proof follows by substituting the aligned factors into the layer gradient and differentiating $\sigma_i=a_i^L$, as detailed in \Cref{app:aligned-mode-dynamics}.
The coefficient $L\sigma_i^q$ is common to every optimizer and comes from the factorization.
The optimizer determines $z_i$ through the current gradient and its evolving state.
For example, GD uses $z_i=g_i$ and yields $\dot\sigma_i=-L\sigma_i^p(\sigma_i-\mu_i)$ as also shown in \citet{arora2019implicit}, whereas normalization and preconditioning modify this dependence. 
\\[1ex]
\textbf{Comparisons across depth.}
Unless stated otherwise, we compare by fixing the target matrix $M$ and initial effective mode amplitudes $\sigma_i(0)\!=\!s_i$, implying balanced mode amplitudes $a_i(0)\!=\!s_i^{1/L}$.
Optimizer settings and initial states are specified separately for each comparison.
A common rescaling of time is used only to distinguish uniform acceleration from changes in relative mode evolution.



\vspace{-0.1cm}
\section{Optimizer Specific Results}
\label{sec:results}
\vspace{-0.1cm}
In this section, we ignore mode index $i$ for better readability. For a positive target $\mu$, \emph{birth} is the first passage from $\sigma(0)\!=\!s>0$ to a fixed threshold $b\in(s,\mu)$. We denote it by $T_{\mathrm{opt}}^{\mathrm{birth}}(s,b).$
\emph{Saturation} is the time required for the residual $r(t)\!=\!|\sigma(t)-\mu|$ to enter and remain below a tolerance $\tau>0$.
For an unsupported mode $\mu\! =\! 0$, \emph{decay} is the passage from $\sigma(0)\!=\!s>0$ to a tolerance $\delta \in (0,s)$, and is denoted by $T_{\mathrm{opt}}^{\mathrm{decay}}(s,\delta).$
Exact extinction means that the mode reaches $\sigma\!=\!0$ in finite time and remains there.
Birth always refers to escape from a positive seed, rather than spontaneous activation of an exactly zero factor mode.
From \Cref{prop:aligned_dynamics}, for $s<b$, time is $T_{\mathrm{opt}}(s,b)\!=\!-\frac1L\int_s^b\frac{x^{-q}}{z}\,dx$. 

\vspace{-0.2cm}
\paragraph{Birth and decay exponents.} 
Suppose the mode dynamics satisfies 
$\dot\sigma = c \sigma^{\alpha}$, where $c>0$. Then $T(a,b) = \frac1c  \int_a^b x^{-\alpha} \,dx$ with $0 < a < b$. Consequently,
\begin{equation}
    T(a,b)= \dfrac{1}{c}\log \dfrac{|b|}{|a|} \text{ for } \alpha=1, \qquad  \text{ and }  \qquad T(a,b)= \dfrac{ b^{1-\alpha}-a^{1-\alpha}}{c (1-\alpha) } \text{ for } \alpha \ne 1.
    \label{eq:bd_time}
\end{equation}

All the optimizers birth dynamics near $\sigma(0)\!=\!s\!=\!0$ satisfies $\dot\sigma \!\sim\! c_{\mathrm b}\sigma^{\alpha_{\mathrm b}}$, with $c_{\mathrm b}\!>\!0$ for a target mode $\mu \!>\!0$. 
Thus, $T^{\mathrm{birth}}(s,b)$ from \Cref{eq:bd_time} as $s \to 0$, $\alpha_{\mathrm b}\!<\!1$ gives finite escape from a vanishing positive seed, $\alpha_{\mathrm b}\!=\!1$ gives a logarithmic birth barrier, and $\alpha_{\mathrm b}\!>\!1$ gives a polynomial birth barrier. 
Similarly, the decay dynamics for an unsupported mode $\mu\!=\!0$ satisfies $\dot\sigma \!\sim\! - c_{\mathrm d}\sigma^{\alpha_{\mathrm d}}$, with $c_{\mathrm d}\!>\!0$. Thus, $T^{\mathrm{decay}}(s,\delta)$ can be similarly computed.
It gives that $\alpha_{\mathrm d}\!<\!1$ reaches zero in finite time, 
$\alpha_{\mathrm d}\!=\!1$ has exponential decay time,
$\alpha_{\mathrm d}\!>\!1$ has algebraic decay and zero is reached only asymptotically.

\vspace{-0.2cm}
\paragraph{Saturation exponent.} 
Suppose that near a nonzero target, the residual is $r_0=|\sigma(t_0)-\mu|$ at time $t_0$, satisfies $\dot r = -\beta r^\gamma$ with $\beta>0.$  
For $\gamma=0$, $r(t) = \max{\{ 0,(r_0-\beta t) \}}$, thus saturates (decreases) linearly at rate $\beta$. 
For $\gamma=1$, $r(t) = \exp{(-\beta t)} r_0$ saturates (decreases) exponentially at rate $\beta$. 

We derive the phase laws for the optimizers and present them using these exponents in \Cref{thm:theory_phase_laws}. All proofs are provided in \Cref{app:phase_laws}.

\begin{theorem}[Optimizer phase laws]
Based on the layer updates $Z_\ell$ in \Cref{tab:optimizer_conventions}, the singular mode dynamics $\dot\sigma$ for each optimizer are derived, and the following table summarizes them with the phase laws expressing the birth, decay, and saturation exponents $\sigma_\mathrm{b}$, $\sigma_\mathrm{d}$, and $\gamma$, respectively.

\resizebox{\linewidth}{!}{%
\begin{tabular}{@{}lllll@{}}
\toprule
\textbf{Optimizer} & 
\textbf{Dynamics}   & 
\begin{tabular}[c]{@{}l@{}}\textbf{Birth} \\  $\dot\sigma \sim c_{\mathrm b}\sigma^{\alpha_{\mathrm b}}$ \end{tabular} & 
\begin{tabular}[c]{@{}l@{}}\textbf{Decay} \\  $\dot\sigma \sim - c_{\mathrm d}\sigma^{\alpha_{\mathrm d}}$ \end{tabular}  & 
\begin{tabular}[c]{@{}l@{}}\textbf{Saturation} \\  $\dot r \sim -\beta r^\gamma$ \end{tabular} \\ \midrule

\begin{tabular}[c]{@{}l@{}}GD \end{tabular}   & 
$\dot\sigma=L\sigma^p(\mu-\sigma)$         & 
\begin{tabular}[c]{@{}l@{}} $\alpha_{\mathrm{b}} = p$ \\ $c_{\mathrm{b}} = L\mu$\\ \textbf{Time:} \color{orange}{Finite for $L=1$}, \\ \color{orange}{Asymptotic for $L>1$}\end{tabular}                                              & 
\begin{tabular}[c]{@{}l@{}}$\alpha_{\mathrm{d}} = p+1$ \\ $c_{\mathrm{d}} = L$ \\ \textbf{Time:} \color{orange}{Asymptotic} \end{tabular}                       & 
\begin{tabular}[c]{@{}l@{}}$ \gamma=1$ \\ $\beta = L\mu^p$ \\ \color{orange}{Exponential} \end{tabular}                                                                                         \\ \midrule

\begin{tabular}[c]{@{}l@{}}SignGD \& \\ Muon \end{tabular}  & 
$\dot\sigma=L\sigma^q\sign(\mu-\sigma)$ & 
\begin{tabular}[c]{@{}l@{}}$\alpha_{\mathrm{b}} = q$ \\ $c_{\mathrm{b}} = L$ \\  \textbf{Time:} \color{blue}{Finite}\end{tabular} & 
\begin{tabular}[c]{@{}l@{}}$\alpha_{\mathrm{d}} = q$ \\ $c_{\mathrm{d}} = L$ \\ \textbf{Time:} \color{blue}{Finite}\end{tabular}    & 
\begin{tabular}[c]{@{}l@{}}$ \gamma=0$ \\ $ \beta = L\mu^q$ \\ \color{blue}{Linear}  \end{tabular}                                                                                              \\ \midrule

\begin{tabular}[c]{@{}l@{}} Adam  \end{tabular}   & \begin{tabular}[c]{@{}l@{}} $\dot\sigma = \dfrac{ L (\mu-\sigma)\sigma^p }{|\mu-\sigma|\sigma^q+\epsilon}$ \\ \\ $\kappa = \mu^{q+1} \dfrac{q^q}{(q+1)^{q+1}}$ \\ 
$h(\sigma) = (\mu-\sigma)\sigma^q$ \\
$\sigma_\mathrm{A}^{\mathrm{d}} = \epsilon^{1/(q+1)}$ \end{tabular} & 
\begin{tabular}[c]{@{}l@{}} If $h(\sigma)/\epsilon \to 0$ 
\\ $\alpha_{\mathrm{b}} = p$\\ $c_{\mathrm{b}} = L\mu/\epsilon$ \\ \textbf{Time:} \color{orange}{GD-like} \\ If $\epsilon < \kappa$ \& $\epsilon/h(\sigma) \to 0$ \\$\alpha_{\mathrm{b}} = q$\\ $c_{\mathrm{b}} = L$ \\ \textbf{Time:} \color{blue}{SignGD-like} \end{tabular} & 
\begin{tabular}[c]{@{}l@{}} If $\sigma(t) \!\ll\! \sigma_\mathrm{A}^{\mathrm{d}}$, \\ $\alpha_{\mathrm{d}} = p+1$\\ $c_{\mathrm{d}} = L/\epsilon$ \\ \textbf{Time:} \color{orange}{GD-like} \\ If $\sigma(t) \gg \sigma_\mathrm{A}^{\mathrm{d}}$ \\$\alpha_{\mathrm{d}} = q$\\ $c_{\mathrm{d}} = L$ \\ \textbf{Time:} \color{blue}{SignGD-like} \end{tabular} & 
\begin{tabular}[c]{@{}l@{}} If $r\mu^q \ll \epsilon$, \\ $\gamma = 1$\\ $\beta = L\mu^p/\epsilon$ \\ \color{orange}{GD-like} \\ Else \\$\gamma = 0$\\ $\beta = L\mu^q$ \\ \color{blue}{SignGD-like} \end{tabular} \\
\midrule

\end{tabular}
}

\resizebox{\linewidth}{!}{%
\begin{tabular}{@{}lllll@{}}


\multirow{4}{*}[-2em]{Shampoo} & 
\multirow{4}{*}[-0.1em]{
\begin{tabular}[c]{@{}l@{}} $\dot\sigma=\dfrac{L\sigma^p(\mu-\sigma)}{d}$ \\[2ex]  $d(t)\!=\!\sqrt{\lambda_{\mathrm{Sh}} + h(t)}$ \\[1ex] $h(t) \!=\! \int_0^t g(u)^2 du$ \\[1ex] $ h(0) \!=\! 0, d(0) \!=\! \sqrt{\lambda_{\mathrm{Sh}}}$ \\[2ex] $\nu(\sigma)\!=\! \frac{(\sigma-s)(2\mu-\sigma-s)}{4L}$ \\ 
$d_0 \!=\! d(0)$ \\ $d(\sigma)=d_0+\nu(\sigma)$
\end{tabular}  
} & 
\multicolumn{1}{l}{\emph{\underline{Initial phase:}} $\nu(\sigma)\!\ll\! d_0$ } & \emph{\underline{Final decay:}} &   \\
& & 
\begin{tabular}[c]{@{}l@{}} $\alpha_{\mathrm{b}} = p$ \\ $c_{\mathrm{b}} = \dfrac{L\mu}{d_0}$\\ \textbf{Time:} \color{orange}{GD-like} \end{tabular} & 
\begin{tabular}[c]{@{}l@{}} $\alpha_{\mathrm{d}} = p+1$ \\ $c_{\mathrm{d}} = \dfrac{4L^2}{4L d_0 + s^2}$ \\ \textbf{Time:} \color{orange}{GD-like} \end{tabular}      & 
\begin{tabular}[c]{@{}l@{}}  $\gamma = 1$ \\ $\beta= \dfrac{L\mu^{p}}{d_0+\frac{(\mu-s)^2}{4L}}$\\ \color{orange}{GD-like} \end{tabular} \\ 
& & \multicolumn{3}{l}{\emph{\underline{Accumulator phase:}} $s\!\ll \!\sigma\! \ll\! \mu$ \& $2L d_0\! \ll\! \mu \sigma$}  \\
& & 
\begin{tabular}[c]{@{}l@{}} $\alpha_{\mathrm{b}} = p-1 \, (<q)$ \\ $c_{\mathrm{b}} = 2L^2$\\ \textbf{Time:} Finite across \\ this interval
\end{tabular} & 
& 
\begin{tabular}[c]{@{}l@{}}  If $s \!\ll\! \mu$ \& $d_0 \! \ll \! \frac{(\mu-s)^2}{4L}$ \\ $\beta \approx 4L^2\mu^{p-2}$\\ {\color{orange}{GD-like}}, $\beta \propto \mu^{p-2}$: \\ order can reverse \end{tabular} \\
\midrule


\multirow{4}{*}[-2.5em]{K-FAC} & 
\multirow{4}{*}[-2.5em]{$\dot\sigma = \dfrac{L\sigma^p (\mu-\sigma)}{\sigma^p + \lambda_{\mathrm{KF}}}$ } & 
\multicolumn{3}{l}{\emph{\underline{Undamped:}} $\quad$ When $\lambda_{\mathrm{KF}} = 0$ or when $\lambda_{\mathrm{KF}} > 0$ \& if $\sigma^p \gg \lambda_{\mathrm{KF}}$} \\
& & 
\begin{tabular}[c]{@{}l@{}} $\alpha_{\mathrm{b}} = 0$ \\ $c_{\mathrm{b}} = L\mu$\\ \textbf{Time:} Finite \end{tabular} & 
\begin{tabular}[c]{@{}l@{}} $\alpha_{\mathrm{d}} = 1$ \\ $c_{\mathrm{d}} = L$ \\ \textbf{Time:} Asymptotic \end{tabular}      & 
\begin{tabular}[c]{@{}l@{}}  $\gamma = 1$ \\ $\beta= L$\\ \color{orange}{GD-like} \end{tabular} \\ 
& & \multicolumn{3}{l}{\emph{\underline{Damped:}} $\quad$ When $\lambda_{\mathrm{KF}} > 0$ \& if $\sigma^p \ll \lambda_{\mathrm{KF}}$}  \\
& & 
\begin{tabular}[c]{@{}l@{}} $\alpha_{\mathrm{b}} = p$ \\ $c_{\mathrm{b}} = L\mu/\lambda_{\mathrm{KF}}$\\ \textbf{Time:} \color{orange}{GD-like} \end{tabular} & 
\begin{tabular}[c]{@{}l@{}} $\alpha_{\mathrm{d}} = p+1$ \\ $c_{\mathrm{d}} = L/\lambda_{\mathrm{KF}}$ \\ \textbf{Time:} \color{orange}{GD-like} \end{tabular}      & 
\begin{tabular}[c]{@{}l@{}}  $\gamma = 1$ \\[-1ex] $\beta= \dfrac{L \mu^p}{\mu^p + \lambda_{\mathrm{KF}}}$\\ \color{orange}{GD-like} \end{tabular} \\ 
\bottomrule


\end{tabular}
}
\vspace{-0.3cm}
\label{thm:theory_phase_laws}    
\end{theorem}

\paragraph{Interpretation of \Cref{thm:theory_phase_laws} and insights.}
Let's start with \textbf{Euclidean GD}. 
The birth time $T_{\mathrm{GD}}^{\mathrm{birth}}(s,b)$ for $0<s<b<\mu$, and decay time $T_{\mathrm{GD}}^{\mathrm{decay}}(s,\delta)$ for $\mu=0$ are
\begin{equation*}
\vspace{-0.05cm}
  \begin{minipage}[t]{0.55\textwidth}
    \centering
    $T_{\mathrm{GD}}^{\mathrm{birth}}(s,b)=
    \begin{cases}
        \log \frac{\mu-s}{\mu-b},&L=1,\\[1ex]
        \frac{1}{2\mu}\log\frac{b(\mu-s)}{s(\mu-b)},&L=2,\\[1ex]
        \frac{s^{1-p} - b^{1-p}}{L\mu(p-1)}+O(1),&L>2.
    \end{cases}
    $
  \end{minipage}
  \hfill
  \begin{minipage}[t]{0.4\textwidth}
    \centering
    $T_{\mathrm{GD}}^{\mathrm{decay}}(s,\delta) =
\begin{cases}
\log \frac{s}{\delta}, & L=1,\\[1ex]
\frac{\delta^{-p}-s^{-p}}{Lp}, & L>1.
\end{cases}$
  \end{minipage}
\end{equation*}
Consequently, as $s \!\to\! 0$, birth time $T_{\mathrm{GD}}^{\mathrm{birth}}(s,b)$ diverges for $L\!\ge\! 2$, and as $\delta\! \to\! 0$, decay time $T_{\mathrm{GD}}^{\mathrm{decay}}(s,\delta)$ diverges for all $L$.
$T_{\mathrm{GD}}^{\mathrm{birth}}(s,b)$ clearly shows that $L\!=\!1$ has no initialization-dependent birth barrier, whereas $L\!=\!2$ introduces a logarithmic delay, while $L\!>\!2$ 
a polynomial delay.
From the decay time, decay is logarithmic for $L\!=\!1$, and algebraic for $L\!>\!1$, with the polynomial tail becoming slower as depth increases.
No unsupported mode reaches zero in finite time.
The exponential saturation implies that modes with smaller target singular values converge slowly.

\begin{tcolorbox}[colback=yellow!5!white,grow to left by=0\linewidth, width=1\linewidth, boxsep=0.5mm, arc=3mm, left=4mm, right=4mm, top=2mm, bottom=2mm]
\vspace{-0.05cm}
\emph{GD delays the growth of small positive modes while it removes zero target modes only asymptotically. 
Depth promotes low effective rank primarily by suppressing the birth of new modes, not by rapidly pruning active ones.
}
\vspace{-0.05cm}
\end{tcolorbox}

For \textbf{SignGD}, the birth time $T_{\mathrm{SignGD}}^{\mathrm{birth}}(s,b)$ for $\mu>0$,
and decay time $T_{\mathrm{SignGD}}^{\mathrm{decay}}(s,\delta)$ for $\mu=0$ are
\vspace{-0.05cm}
$$
T_{\mathrm{SignGD}}^{\mathrm{birth}}(s,b)=b^{1/L}-s^{1/L} 
\qquad  \qquad
T_{\mathrm{SignGD}}^{\mathrm{decay}}(s,\delta)=s^{1/L}-\delta^{1/L}.$$
Since the time to reach any target $\mu$ is $|s^{1/L}-\mu^{1/L}|$, birth, saturation and decay occur in \emph{finite} time. 
For $\mu=0$, the mode becomes exactly zero at time $s^{1/L}$.
The birth time remains bounded as $s \to 0$.

\textbf{Adam} updates are considered under a quasi-steady approximation to isolate the effect of coordinate normalization, with $m=g$ and $v=g^2$  giving the instantaneous state approximation. If $|\sigma-\mu|\sigma^q \gg \epsilon$ then $\dot\sigma = -L\sigma^q\operatorname{sign}(\sigma-\mu)$, thus behaving like SignGD. 
If $|\sigma-\mu|\sigma^q \ll \epsilon$ then $\dot\sigma = -L\sigma^p(\sigma-\mu)/\epsilon$, thus behaving like GD upto $1/\epsilon$ factor. 
This approximation explains a possible transition between normalized and GD-type motion, but it is not used as an exact theorem for Adam.

\begin{tcolorbox}[colback=yellow!5!white,grow to left by=0\linewidth, width=1\linewidth, boxsep=0.5mm, arc=3mm, left=4mm, right=4mm, top=2mm, bottom=2mm]
\vspace{-0.05cm}
\emph{SignGD exhibits birth, saturation and decay of all modes in finite time for every depth. 
Instantaneous Adam inherits this behavior only when the gradient magnitude dominates $\epsilon$. For sufficiently small birth modes, unsupported modes and final saturation, it behaves like GD.
}
\vspace{-0.05cm}
\end{tcolorbox}

\textbf{Muon} in the aligned model gives
$
\dot\sigma\!=\!-\!L\sigma^q\sign(\sigma\!-\!\mu),
$
reducing it to a scalar sign rule as in SignGD updates.
The disappearance of slow sequential learning under this idealization has been studied in \citet{dragutinovic2026use,rhee2026muon}. However, our phase laws highlight that the same normalization that accelerates weak mode emergence also removes zero target modes in finite time, so faster learning of small components can coexist with faster removal of unsupported ones.

For \textbf{cumulative Shampoo}, based on the common scalar gradient accumulator $h(t)$, and $d(t)$ defined in \Cref{thm:theory_phase_laws}, along a nonstationary aligned mode starting at $\sigma(0)=s>0$ with $d(0) > 0$, we have 
$$
d(\sigma)=d(0)+\frac{\mu(\sigma-s)-(\sigma^2-s^2)/2}{2L} = d_0 + \nu(\sigma).
$$ 
Shampoo's dynamics result in two distinct regimes: 
\underline{\textbf{Initial phase:}} When the accumulator has barely changed, $d_0\! \gg\! \nu(\sigma)$, its dynamics satisfy $\dot\sigma \!\approx\! L\sigma^p(\mu-\sigma)/{d(0)}$,  recovering GD dynamics with $1/d_0$ factor.
\underline{\textbf{Accumulator phase:}}
An intermediate regime arises on an interval $\sigma_a\leq\sigma\leq b$ satisfying $s\ll\sigma_a$, $b\ll\mu$, and $2Ld_0\ll\mu\sigma_a$.
Throughout this interval, accumulated history dominates and 
$
\dot{\sigma}\approx 2L^2\sigma^{p-1},
T_{\mathrm{Sh}}^{\mathrm{birth}}(\sigma_a,b)
\approx
(b^{2/L}-\sigma_a^{2/L})/(4L).
$
In this regime, the growth law is independent of $\mu$, although total birth times retain target dependent ordering.
Ideal Muon instead has target independent birth times to a common absolute threshold from the outset. 
Saturation near a positive target follows a separate local law, with
$
\gamma=1,
\beta=\frac{L\mu^p}{d_0+(\mu-s)^2/(4L)}.
$
If $s\ll\mu$ and $4Ld_0\ll(\mu-s)^2$, then $\beta\approx4L^2\mu^{-2/L}$, so larger targets can have smaller local saturation rates.
For zero target modes, the decay dynamics in this phase is $\dot\sigma\!=\!- \frac{L}{d(\sigma)}\sigma^{p\!+\!1}$, which is dominated by the $p+1$ exponent as shown in \Cref{app:phase_laws},
similar to GD.
%
Thus unsupported modes approach zero only asymptotically.
At fixed initial amplitude and decay tolerance, a larger inherited denominator $d_0$ increases the decay passage time.
Ideal Muon instead removes unsupported aligned modes in finite time.

\begin{tcolorbox}[colback=yellow!5!white,grow to left by=0\linewidth, width=1\linewidth, boxsep=0.5mm, arc=3mm, left=4mm, right=4mm, top=2mm, bottom=2mm]
\emph{Cumulative Shampoo has a two phase emergence behavior: GD like target sensitive birth first, followed by history adapted target independent catch-up phase. It has increasing historical inertia, making larger modes harder to remove, whereas 
Muon has finite time decay. 
}
\end{tcolorbox}

\textbf{K-FAC} updates each layer with the per-layer Fisher curvature block using a damping parameter $\lambda_{\mathrm{KF}}$. From the dynamics, we get two cases: \underline{\textbf{Undamped case ($\lambda_{\mathrm{KF}}=0$):}} the dynamics near zero initialization reduces to $\dot\sigma = L (\mu-\sigma)$, thus $T_{\mathrm{KFAC}}^{\mathrm{birth}}(s,b)=\frac1L\log\frac{\mu-s}{\mu-b}$, 
resulting in finite barrier escape time with target-dependent mode ordering. 
Whereas Muon removes both the factorization barrier and target-amplitude ordering during aligned birth, distinguishing it from K-FAC. 
\underline{\textbf{Damped case ($\lambda_{\mathrm{KF}}>0$):}} When the damped factor dominates, $\lambda_{\mathrm{KF}} \gg \sigma^p$, $\dot\sigma = \frac{L\sigma^p (\mu-\sigma)}{\lambda_{\mathrm{KF}}}$, thus recovering GD-like behavior with $1/\lambda_{\mathrm{KF}}$ factor. Therefore, $\lambda_{\mathrm{KF}}$ defines a spectral mobility threshold.
For a zero target, 
undamped K-FAC removes unwanted modes exponentially, but not in finite time, whereas damped K-FAC does it significantly slower, at only a polynomial rate. 

\begin{tcolorbox}[colback=yellow!5!white,grow to left by=0\linewidth, width=1\linewidth, boxsep=0.5mm, arc=3mm, left=4mm, right=4mm, top=2mm, bottom=2mm]
\emph{Undamped K-FAC updates cancel the depth-dependent mobility barrier, while positive additive damping restores the GD birth asymptotic and a polynomial zero target tail. 
}
\end{tcolorbox}

\vspace{-0.2cm}
\section{Experiments}
\label{sec:exp}
\vspace{-0.1cm}
We numerically verify the phase laws in \Cref{thm:theory_phase_laws} and examine how optimizer-state parameters determine the active spectral regime in deep matrix factorization. 
Unless varied explicitly, we use the same balanced, aligned depth $L=4$ factorization, effective initialization, and target modes across optimizers. Note that $p=1.5$ and $q=0.75$.
We estimate the local birth and decay exponents from the trajectory slopes
$
\alpha_{\mathrm b}(\sigma)
\!=\!
\frac{d\log \dot\sigma}{d\log\sigma},
\,
\alpha_{\mathrm d}(\sigma)
\!=\!
\frac{d\log(-\dot\sigma)}{d\log\sigma}.
$
For saturation, we fit the local residual law
$
\dot r\!=\!-\beta r^\gamma,
\,
r\!=\!|\sigma\!-\!\mu|,
$
where $\beta$ is an exponential rate when $\gamma\!=\!1$ and a constant residual speed when $\gamma\!=\!0$.
Implementation details, hyperparameters and additional experiments are in \Cref{app:exp_settings,app:exp}.

\vspace{-0.2cm}
\paragraph{Verification of the phase laws in \Cref{thm:theory_phase_laws}.}
\Cref{fig:phase_laws} compares quantities measured from the trajectories with the theoretical predictions.
For birth, GD retains exponent $p$, ideal Muon follows $q$, instantaneous Adam interpolates between the Euclidean and normalized regimes, cumulative Shampoo moves from $p$ toward its history-adapted exponent $p-1$, and undamped K-FAC removes the singular-magnitude dependence with $\alpha_\mathrm{b}=0$.
The decay curves distinguish finite extinction under ideal Muon, exponential decay under undamped K-FAC, and the polynomial tails of GD and cumulative Shampoo.
The saturation experiment further reveals optimizer-dependent mode ordering: GD and Adam converge faster for larger targets; undamped K-FAC is approximately target independent while Muon shows dependence; history-dominated Shampoo reverses the GD ordering.
Thus the optimizers change relative spectral progress rather than merely rescaling a common trajectory.

\begin{figure}[t]
    \centering
    \includegraphics[width=\linewidth]{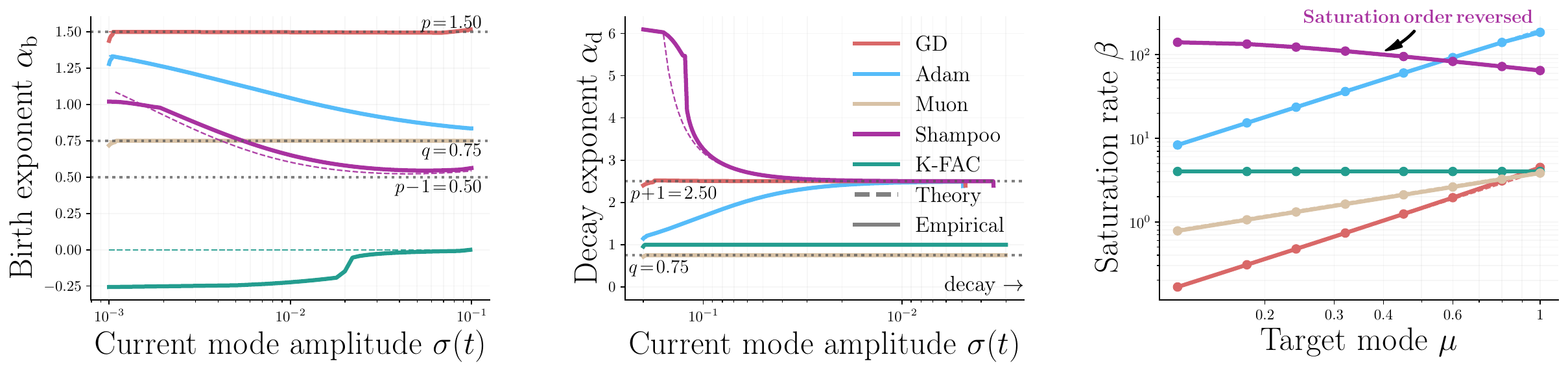}
    \vspace{-0.2cm}
    \caption{
\textbf{Verification of the phase laws in \Cref{thm:theory_phase_laws} at depth $L=4$.}
    \textbf{(a)}
The birth exponent $\alpha_\mathrm{b}$ separates Euclidean, normalized, history adapted, and curvature corrected dynamics. 
\textbf{(b)}
The decay exponent $\alpha_\mathrm{d}$ distinguishes finite extinction of Muon ($\alpha_\mathrm{d}\!=\!q\!<\!1$) from the other optimizers. 
\textbf{(c)}
The saturation coefficient $\beta$ reveals target dependent mode ordering.
}
    \label{fig:phase_laws}
    \vspace{-0.2cm}
    
\end{figure}

\begin{figure}[t]
    \centering
    \begin{subfigure}[t]{0.32\textwidth}
        \centering
        \includegraphics[width=\linewidth]{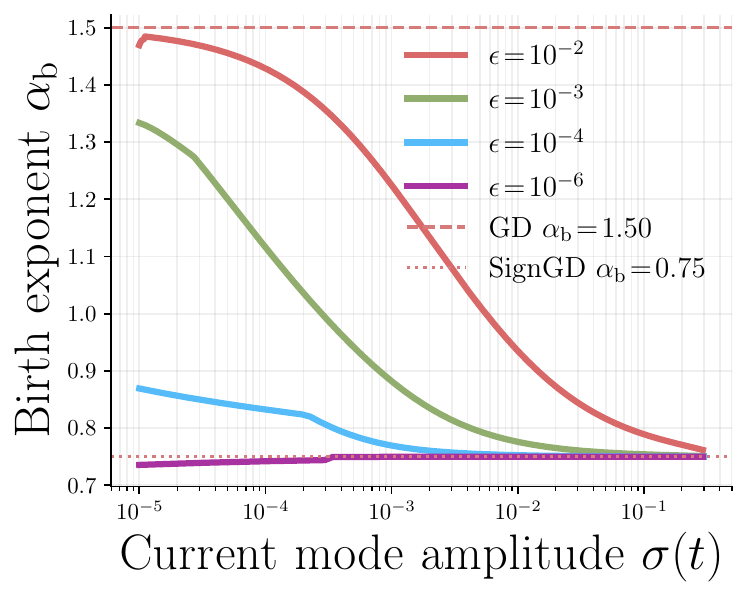}
\vspace{-0.2cm}
    \vspace{-0.2cm}
        
        \caption{Adam $\epsilon$ controls the regime}
        \label{fig:adam}
    \end{subfigure}%
    \hfill
    \begin{subfigure}[t]{0.32\textwidth}
        \centering
        \includegraphics[width=\linewidth]{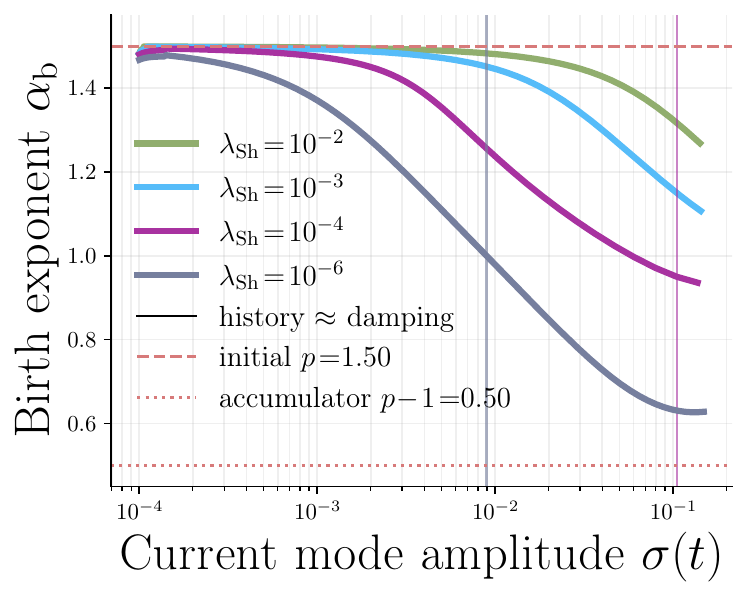}
\vspace{-0.2cm}
        \caption{Two phases of Shampoo}
        \label{fig:shampoo}
    \end{subfigure}
    \hfill
    \begin{subfigure}[t]{0.32\textwidth}
        \centering
        \includegraphics[width=\linewidth]{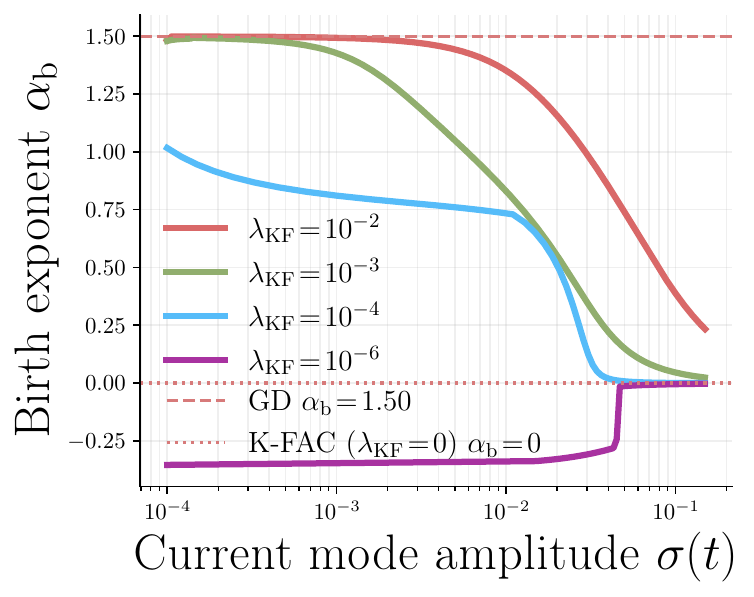}
\vspace{-0.2cm}
        \caption{K-FAC $\lambda_{\mathrm{KF}}\!>\!0$ restores barrier }
        \label{fig:kfac}
    \end{subfigure}
\vspace{-0.2cm}
    \caption{
\textbf{Optimizer state parameters control the local birth regime.}
    Adam's stability parameter $\epsilon$ shifts the GD-to-SignGD crossover; Shampoo damping delays the transition from $p$ to $p-1$; and K-FAC damping restores the GD exponent below its curvature threshold.
}
    \label{fig:phase_trans}
    \vspace{-0.1cm}
\end{figure}

\paragraph{Optimizer state controls spectral crossovers.}
\Cref{fig:phase_trans} varies the state parameters determined in \Cref{thm:theory_phase_laws} that separate the local regimes.
For instantaneous Adam, decreasing $\epsilon$ moves the transition from the GD exponent $p$ to the SignGD exponent $q$ toward smaller modes. 
Large $\epsilon$ retains the Euclidean exponent over a wider range of mode amplitudes, whereas small $\epsilon$ yields the normalized exponent $q$ almost immediately.
For cumulative Shampoo, decreasing $\lambda_{\mathrm{Sh}}$ causes accumulated history to dominate earlier, advancing the transition from $p$ to $p-1$.
For K-FAC, positive damping restores the GD exponent when $\sigma^p\ll\lambda_{\mathrm{KF}},$
whereas curvature correction removes the barrier when $\sigma^p\gg\lambda_{\mathrm{KF}}.$ 
As $\lambda_{\rm KF}\rightarrow0$, this transition moves toward smaller amplitudes, recovering the undamped K-FAC regime.
These experiments confirm that stabilization, accumulated state, and damping determine which phase law governs a mode at a given spectral scale.

\section{Conclusion}
\label{sec:conclusion}

Our analysis shows that the spectral bias of depth depends on optimizer geometry and the stage of training.
A surprising consequence is that rapid learning of weak components can coexist with rapid removal of unsupported ones, as illustrated by the Muon analysis.
Conversely, accumulated history can accelerate the emergence of a mode while prolonging its persistence.
These effects make a single ordering of optimizers by stronger or weaker low-rank bias inadequate.

A further implication is that parameters associated with numerical stability also shape learning bias.
In the analyzed models, Adam's stabilizer, Shampoo's accumulator initialization, and K-FAC's damping determine where normalization, history, or curvature correction governs spectral progress.
They can therefore change how strongly depth suppresses small modes, even within the same optimizer.
Although the phase laws rely on balance and alignment assumptions, they identify concrete mechanisms to understand optimizer behavior. 
This perspective motivates evaluating optimizers through which components they accelerate, suppress, or preserve alongside the training loss, since comparable fit can conceal differences in their dynamics. 
The resulting framework provides a basis for understanding and testing how faster optimizers change the complexity of intermediate models.



\subsection*{AI use statement}

Generative AI tools were used to assist with language polishing, related work discovery, and selected coding and debugging tasks.
They were not used to draft complete sections or to generate mathematical proofs.
All proofs were developed by the authors, although AI tools were occasionally used to check arguments and suggest clearer formulations. 
Example use cases include code for plotting, logging, and refactoring the functions which were all verified to do what was intended. 
All references, code, and final text were independently verified and thoroughly reviewed by the authors, who take full responsibility for the content.




\subsection*{Ethics statement}

This work is theoretical and empirical research on optimization dynamics in deep matrix factorization.
It does not involve human participants, personal or sensitive data, or deployment in high-stakes decision-making settings.
The experiments primarily use controlled synthetic matrix factorization and regression problems designed to derive deeper insights about the optimizer behavior. Thus, we believe the potential benefits of this research outweigh its risks. Additionally, we
do not see any immediate risk stemming from our work.



\subsection*{Reproducibility statement}

We state all assumptions, optimizer conventions, continuous-time update rules, and initialization conditions used in the theoretical analysis.
Complete derivations and proofs are provided in the appendix, including the effective-matrix dynamics, aligned-mode reduction, and birth, decay, and saturation laws.
For the experiments, we specify the data, depth, balanced initialization, optimizer hyperparameters, damping and stabilization parameters, stopping criteria, and procedures used to estimate the local phase exponents in \Cref{app:exp_settings,app:exp}.
We provide the source code and experiment configurations needed to reproduce the reported figures and tables in the supplementary material as a zip file.




\subsubsection*{Acknowledgments}
This work was carried out in partnership with the support of the Commonwealth Bank of Australia.

\bibliography{ref}
\bibliographystyle{iclr2027_conference}
\clearpage

\appendix
\tableofcontents
\newpage
\section{Extended Related Work}
\label{app:related_work}

\paragraph{Deep linear dynamics and balancedness.}
\citet{saxe2013exact} derive exact learning dynamics for deep linear networks under structured initial conditions, including a decoupled mode representation and equal strengths across layers.
Their analysis supplies the direct precedent for our aligned phase model.
\citet{arora2018optimization} show that linear overparameterization changes the effective optimization geometry and can act as a preconditioner.
\citet{du2018algorithmic} study balancing properties of gradient flow in homogeneous models, while \citet{arora2019implicit} use balanced initialization to derive effective matrix and singular value dynamics for deep matrix factorization.
These GD results motivate our setup, but conservation of the GD balance invariant generally does not hold for other optimizers.

\paragraph{Low-rank bias and the limits of norm descriptions.}
\citet{gunasekar2017implicit} investigate implicit regularization in two-layer matrix factorization and conjecture a nuclear-norm characterization in a restricted commuting setting.
\citet{arora2019implicit} extend the analysis to depth and derive singular-value-dependent mobility under balanced gradient flow, showing that depth strengthens the tendency toward low-rank solutions.
\citet{razin2020implicit} construct matrix-completion problems in which norms diverge as the loss vanishes, demonstrating the limits of a universal norm-minimization account.
\citet{li2021towards} instead develop a greedy low-rank learning interpretation under explicit assumptions.
A complementary trajectory perspective is provided by Neural Tangent Kernel and lazy-training analyses, which characterize sufficiently wide networks through a linearization around initialization and an approximately fixed kernel governing function-space evolution \citep{jacot2018neural,lee2019wide,sabanayagam2022analysis}.
In contrast, our analysis retains the nonlinear multiplicative parameterization and its evolving, optimizer-dependent spectral mobility, allowing us to compare the birth, saturation, and decay of singular modes before convergence.

\paragraph{Optimization geometry and implicit bias.}
\citet{wilson2017marginal} demonstrate that adaptive optimization can select solutions with substantially different generalization from GD even when training fit is good. 
\citet{gunasekar2018characterizing} characterize implicit bias for several optimization geometries in underdetermined linear regression and separable classification.
These results establish the broader principle that optimization geometry affects solution selection.
Deep factorization adds a parameter dependent product differential, so the geometry acting on the effective matrix changes with both the factors and optimizer state captured by our general trajectory identity in \Cref{prop:effective_dyn}.


\paragraph{Implicit bias of Adam and related adaptive methods.}
\citet{cattaneo2023implicit} use backward error analysis to characterize corrections to Adam and RMSProp dynamics, relating implicit regularization to hyperparameters and training stage.
The classification literature builds on the maximum margin characterization of GD for separable data \citep{soudry2018implicit}.
\citet{wang2021implicit} study limiting parameter directions in homogeneous classification under their adaptive flow assumptions.
\citet{zhang2024implicit} establish a maximum $\ell_\infty$ margin result for full batch Adam in separable linear classification.
\citet{xie2024implicit} connect AdamW, which includes decoupled weight decay, to $\ell_\infty$-norm constrained optimization under full-batch assumptions.
\citet{vasudeva2026rich} analyze a difference in feature learning between Adam and GD in a Gaussian classification model with two layer ReLU networks.
These results concern distinct losses, optimizer conventions, or notions of bias.
Our squared loss analysis studies finite singular amplitudes and phase times without weight decay, so classification margin results do not directly determine its birth or saturation laws.




\paragraph{Analysis of spectral descent, Shampoo, Muon and K-FAC.}
\citet{fan2026implicit} characterize spectral descent and Muon through margin maximization in multiclass linear classification.
\citet{chen2025muon} connect Muon with decoupled weight decay to spectral-norm-constrained optimization.
%
\citet{eschenhagen2026clarifying} interpret Shampoo as spectral descent adapted to stochasticity and the parameter trajectory.
History dependence is therefore an established feature of this geometry.
Our contribution is the exact first integral of the specified cumulative accumulator in the aligned factorization model and its consequences for phase times.
That identity does not extend unchanged to exponential averages or different inverse powers.
\citet{kerekes2021depth} explicitly study natural gradient implicit bias in deep linear classification and matrix factorization.
Under their Gaussian model and Jacobian rank condition, natural gradient flow for matrix completion matches direct optimization of the effective matrix, leaving unobserved entries unchanged.
This is a direct precedent for connecting geometry to motion outside the observed coordinates.
Our block curvature analysis quantifies the effect of additive damping on mode birth and decay, and distinguishes it from both full curvature and factored damping.
These works provide important geometry and momentum precedents, while our phase laws concern squared loss, explicit stabilization or damping factors, and finite targets in factorized models.

\paragraph{Analysis of Muon in matrix factorization.}
\citet{kang2026uniform} analyze spectral gradient flow in a simplified LoRA factorization setting, proving approximately equal spectral growth rates and convergence results under stated conditions.
\citet{dragutinovic2026use} derive dynamics for simplified Muon in deep linear networks and linear attention and investigate how avoiding slow sequential learning changes simplicity bias.
\citet{rhee2026muon} analyze bilinear factorization with full observation, including uniform factor growth after alignment, a conserved quantity distinct from GD's, and finite step oscillations near solutions.
They also study alignment in simplified settings.
These results already establish much of the ideal polar growth baseline.
Our treatment places it alongside memory dependent adaptive flows and damped curvature flows, with separate statements for birth, saturation, and zero target decay.
\citet{parviz2026reassessing} report controlled optimizer comparisons showing that Muon's relative performance in factorization depends on the problem and hyperparameter tuning.
This motivates separating exact polar flow predictions from practical implementation effects and comparing recovery at a common training fit.

\section{Optimizers}
\label{app:optimizer}

As we detailed in \Cref{sec:prelims}, we study deterministic flows driven by the full gradient, without weight decay or explicit regularization. For each layer, the continuous time dynamics of the matrix $W_\ell$ is 
$$
\dot W_\ell(t)=-Z_\ell(t),
\qquad
W_{\ell,k+1}=W_{\ell,k}-\eta Z_{\ell,k},
$$
where $t$ is continuous time, a dot denotes its derivative, $k$ is a discrete iteration index, and $\eta>0$ is the step size. The second expression is the corresponding factor update once the optimizer states have been specified.
The flow uses a unit learning rate, with any constant common learning rate absorbed into time and the stated memory times. All moment matrices have the same dimensions as $W_\ell$.
The equations for $Z_\ell$ for all the optimizers 
analyzed in this paper are discussed in detail below. 

\paragraph{Euclidean GD.} From our setting, recollect that $H_\ell$ is the layer gradient $H_\ell = \nabla_{W_\ell} \mathcal{L}(W)$.
Gradient descent uses the layer gradient directly, giving
$$
Z_\ell=H_\ell,
\qquad
\dot W_\ell=-B_\ell^\top G A_\ell^\top.
$$
This gradient flow is the baseline for the established analysis of deep matrix factorization \citep{arora2019implicit,li2021towards}.

\paragraph{SignGD.}
SignGD applies $\sign$ function entrywise to the layer gradient \citep{bernstein2018signsgd}. Therefore, $$\dot W_\ell = - Z_\ell = - \sign(H_\ell).$$

\paragraph{Adam.}
Adam maintains entrywise exponential averages of the first moment $m_\ell$ and second moment $v_\ell$ \citep{kingma2014adam}.
In discrete time, the decay coefficients $\beta_1,\beta_2\in[0,1)$ specify the fraction of the previous moment retained:
$$
m_{\ell,k+1}=\beta_1m_{\ell,k}+(1-\beta_1)H_{\ell,k},
\qquad
v_{\ell,k+1}=\beta_2v_{\ell,k}+(1-\beta_2)H_{\ell,k}^{\odot2}.
$$
Both states start at zero. The continuous dynamics of Adam with its normalization and moment states are established in \citet{barakat2021convergence, gould2024continuous}, which are described below.
$$
\tau_m\dot m_\ell=H_\ell-m_\ell,
\qquad
\tau_v\dot v_\ell=H_\ell^{\odot2}-v_\ell,
\qquad
\dot W_\ell= -Z_\ell = - m_\ell\oslash(\sqrt{ v_\ell}+\epsilon).
$$
The constant $\epsilon>0$ in the denominator is the stabilizer. 
At $t=0$, for the uncorrected moment flow, take $m_\ell(0)=0$ and $v_\ell(0)=0$.
The displayed differential equations describe these uncorrected moments.
The phase laws below instead concern the separately specified instantaneous approximation
$m_\ell=H_\ell$ and $v_\ell=H_\ell^{\odot2}$.
The finite memory limit scales the discrete decay coefficients with the step size, and is distinct from a limit holding those coefficients fixed.

\paragraph{Muon.}
Muon orthogonalizes a momentum matrix before applying the update \citep{jordan2024muon}.
We use exact polar normalization: if $X=U_X\Sigma_XV_X^\top$ is a compact singular value decomposition retaining only positive singular values, define $\polar(X)=U_XV_X^\top$ and $\polar(0)=0$.
We consider an ideal setting without momentum, which reduces to the following update. 
$$
\dot W_\ell= -Z_\ell = -\polar(H_\ell),
$$
The variant with momentum replaces $\polar(H_\ell)$ with $\polar(m_\ell)$.
Approximate polar computation, Nesterov momentum, and shape dependent update scaling are additional choices in practical Muon implementations.

\paragraph{Shampoo.}
For matrix parameters, the original Shampoo geometry applies inverse fourth roots of left and right gradient statistics \citep{gupta2018shampoo}.
Our cumulative flow is
$$
\dot C_\ell=H_\ell H_\ell^\top,
\qquad
\dot D_\ell=H_\ell^\top H_\ell,
\qquad
\dot W_\ell=-Z_\ell=-C_\ell^{-1/4}H_\ell D_\ell^{-1/4},
$$
with $C_\ell(0)=\lambda_\mathrm{Sh} I_{d_\ell}$, $D_\ell(0)=\lambda_\mathrm{Sh} I_{d_{\ell-1}}$, and $\lambda_\mathrm{Sh}>0$.
Here $C_\ell\in\R^{d_\ell\times d_\ell}$ and $D_\ell\in\R^{d_{\ell-1}\times d_{\ell-1}}$ accumulate gradient history, and $\lambda_\mathrm{Sh}$ specifies their initial regularization. From the state equations, we get 
$$C_\ell(t)=\lambda_{\mathrm{Sh}}I_{d_\ell}+
\int_0^t H_\ell(u)H_\ell(u)^\top\,du, \quad \text{and} \quad
D_\ell(t)=\lambda_{\mathrm{Sh}}I_{d_{\ell-1}}+
\int_0^t H_\ell(u)^\top H_\ell(u)\,du.
$$

\paragraph{K-FAC.}
K-FAC approximates layer curvature by a Kronecker product of activation and backward signal statistics \citep{martens2015optimizing}.
First, to get the Kronecker product form of $H_\ell$, let's remind the vectorization identity
$
\operatorname{vec}(AXB) =
(B^\top\otimes A)\operatorname{vec}(X),
$
using which we obtain 
$$\operatorname{vec}(H_\ell) = 
(A_\ell \otimes B_\ell^\top) \operatorname{vec}(G) = \mathcal{J}_\ell^\top \operatorname{vec}(G).$$
Note that $\mathcal{J}_\ell$ is the Jacobian of the end to end matrix with respect to layer $\ell$.
For squared loss, the Hessian with respect to the end to end prediction is the identity.
Therefore, the generalized Gauss Newton block associated with layer $\ell$ is
\begin{align}
    \mathcal F_\ell^{\mathrm{GN}} &=
\mathcal J_\ell^\top\mathcal J_\ell \nonumber \\
&= (A_\ell \otimes B_\ell^\top)(A_\ell^\top \otimes B_\ell)\nonumber \\
&= (A_\ell A_\ell^\top) \otimes (B_\ell^\top B_\ell). \label{eq:gn_kfac}
\end{align}
The block curvature is therefore already in Kronecker factored form in the full observation deep linear model.

Note that the same block is obtained from the expected Fisher under an isotropic Gaussian observation model. Briefly, let's consider $y$ given $x$ and parameters $\Theta$ is isotropic Gaussian, that is, 
$
y\mid x,\Theta
\sim
\mathcal N
\left(
Wx,
\tau^2 I
\right).
$
Let
$
h_{\ell-1}(x) = A_\ell x
$
denote the input activation to the layer $\ell$.
The Jacobian of the output with respect to $W_\ell$ is
$
J_\ell(x) =
h_{\ell-1}(x)^\top\otimes B_\ell.
$
The expected Fisher block is
\begin{align*}
\mathcal F_\ell^{\mathrm F}
&=
\frac1{\tau^2}
\mathbb E_x
\left[
J_\ell(x)^\top J_\ell(x)
\right] \nonumber
\\
&=
\frac1{\tau^2}
\mathbb E_x
\left[
h_{\ell-1}(x)h_{\ell-1}(x)^\top
\right]
\otimes
(B_\ell^\top B_\ell) \nonumber
\\
&= \frac1{\tau^2} (A_\ell \mathbb E[xx^\top] A_\ell^\top) 
\otimes
(B_\ell^\top B_\ell) \quad 
\text{; whitened full observations } \mathbb E[xx^\top]=I \nonumber\\
&=
\frac1{\tau^2} (A_\ell A_\ell^\top) 
\otimes
(B_\ell^\top B_\ell). 
\end{align*}
After absorbing the constant $\tau^{-2}$ into the learning rate or damping,
$
\mathcal F_\ell = 
\mathcal F_\ell^{\mathrm F}
=
\mathcal F_\ell^{\mathrm{GN}}.
$
Thus the generalized Gauss Newton block and expected Fisher block coincide under the observation model used here.
We use additive damping on the complete per layer curvature block:
\[
\operatorname{vec}(Z_\ell)
=
\left(
\mathcal F_\ell
+
\lambda_{\mathrm{KF}}I
\right)^{-1}
\operatorname{vec}(H_\ell).
\]
The continuous update is
$
\dot W_\ell
=
- Z_\ell.
$
When $\lambda_{\mathrm{KF}}=0$ and the factors are invertible, the Kronecker identity gives
\[
Z_\ell
=
(B_\ell^\top B_\ell)^{-1}
H_\ell
(A_\ell A_\ell^\top)^{-1}.
\]
For positive damping, it is preferable to retain the block form
$
\left(
\mathcal F_\ell
+
\lambda_{\mathrm{KF}}I
\right)^{-1},
$
because additive damping of the full block is not generally equal to damping the two Kronecker factors separately. 
The undamped case is considered only where the inverse is defined, or with an explicitly specified pseudoinverse.
For general matrix sensing, the measurement loss does not automatically yield these regression curvature factors, so a Kronecker approximation must be specified separately.

\section{Effective Matrix $W$ Dynamics}
\label{app:eff_W_dyn}
Consider the setting described in \Cref{sec:prelims}, where the deep linear network is
$$W= W_L W_{L-1} \cdots W_1.$$
Here $W$ is the effective matrix and
each $W_\ell \in \mathbb{R}^{d_\ell \times d_{\ell-1}}$.
Loss for matrix factorization: 
$$\mathcal L(W) = \frac{1}{2} \| W - M \|_F^2$$
where $M$ is the matrix whose entries are known and to be factorized.
%
%
Gradient of the loss $\mathcal L(W)$ with respect to $W$ is 
$$G = \nabla_W \mathcal L(W) = W - M.$$

\paragraph{Layer gradients $H_\ell$.}
We compute the layer gradients $H_\ell$, which are the gradients of the loss with respect to layer $\ell$. That is $H_\ell=\nabla_{W_\ell}\mathcal{L}(W)$.
From the definition of prefix and suffix products, we have
$$A_\ell = W_{\ell -1} \cdots W_1, \qquad A_1 = I_{d_0}$$
$$B_\ell = W_L \cdots W_{\ell+1}, \qquad B_L = I_{d_L}$$
This implies,  
$$W = B_\ell W_\ell A_\ell$$
Therefore, applying the chain rule, we get 
$$H_\ell = \nabla_{W_\ell} \mathcal L(W) = B_\ell^\top G A_\ell^\top$$

\paragraph{Proof of \Cref{prop:effective_dyn}: Effective matrix dynamics.}
\begin{proof}
Since the effective matrix is the product of all $W_\ell$, we derive its continuous time dynamics using the above.
\begin{align}
    W &= W_L W_{L-1} \cdots W_1 \nonumber \\ 
    \dot{W} &= \sum_{\ell=1}^L W_L \cdots \dot{W}_\ell W_1 \nonumber \\
    &= \sum_{\ell=1}^L B_\ell \dot{W}_\ell A_\ell \label{eq:dt_w}
\end{align}
$\dot{W}_\ell$ in \Cref{eq:dt_w} changes depending on the optimizer. For each optimizer, we give the updates $Z_\ell$ in \Cref{app:optimizer} and summarize it in \Cref{tab:optimizer_conventions}. Thus, eventually $\dot W_\ell = -Z_\ell$.
\end{proof}

\section{Aligned Singular Mode Dynamics}
\label{app:aligned-mode-dynamics}

In this section, we prove \Cref{prop:aligned_dynamics} by restricting these dynamics to the balanced and aligned mode manifold introduced in the main text.
We first explain when balance is preserved, then show how balance and alignment produce a common layer amplitude, and finally derive the effective singular-mode dynamics.
Throughout this section, let
$
d\leq \min_{\ell=0,\ldots,L} d_\ell
$
denote the active mode dimension.
For each layer, let there be an orthogonal matrix
$
Q_\ell\in\mathbb R^{d_\ell\times d},
\,
Q_\ell^\top Q_\ell=I_d.
$ 
All matrix identities below are understood on these active subspaces.
Components in the orthogonal complements are initialized at zero and remain absent under the compatible aligned updates considered here.

\subsection{Balance under Euclidean Gradient Flow}

We first recall why balancedness is an exact invariant of Euclidean gradient flow.
For adjacent layers, define the imbalance matrix
\[
\mathcal D_\ell
=
W_\ell W_\ell^\top
-
W_{\ell+1}^\top W_{\ell+1},
\qquad
\ell=1,\ldots,L-1.
\]

\begin{lemma}[Balancedness invariant for Euclidean gradient flow]
\label{lem:gd-balance}
Let the loss be any differentiable function of the end-to-end product
$
W=W_L\cdots W_1.
$
Under Euclidean gradient flow
$
\dot W_\ell=-H_\ell,
$
the imbalance matrices satisfy
$$
\frac{d}{dt}\mathcal D_\ell=0.
$$
Consequently,
$
\mathcal D_\ell(t)=\mathcal D_\ell(0)
$
for every time at which the flow exists.
\end{lemma}

\begin{proof}

Fix $\ell\in\{1,\ldots,L-1\}$ and define
$
R_\ell
=
B_{\ell+1}^\top G A_\ell^\top
$
with $G=\nabla_W\mathcal L(W)$.
Since
\[
B_\ell
=
B_{\ell+1}W_{\ell+1}, \qquad
\text{and} \qquad
A_{\ell+1}
=
W_\ell A_\ell,
\]
the adjacent layer gradients can be written as
\[
H_\ell
=
W_{\ell+1}^\top R_\ell, \qquad
\text{and} \qquad
H_{\ell+1}
=
R_\ell W_\ell^\top.
\]
Under Euclidean gradient flow,
$
\dot W_\ell=-H_\ell,
\,
\dot W_{\ell+1}=-H_{\ell+1}.
$
Therefore,
\[
\begin{aligned}
\frac{d}{dt}
\left(
W_\ell W_\ell^\top
\right)
&=
-H_\ell W_\ell^\top
-
W_\ell H_\ell^\top
\\
&=
-
W_{\ell+1}^\top R_\ell W_\ell^\top
-
W_\ell R_\ell^\top W_{\ell+1}.
\end{aligned}
\]
Similarly,
\[
\begin{aligned}
\frac{d}{dt}
\left(
W_{\ell+1}^\top W_{\ell+1}
\right)
&=
-H_{\ell+1}^\top W_{\ell+1}
-
W_{\ell+1}^\top H_{\ell+1}
\\
&=
-
W_\ell R_\ell^\top W_{\ell+1}
-
W_{\ell+1}^\top R_\ell W_\ell^\top.
\end{aligned}
\]
The two derivatives are identical.
Hence
\[
\frac{d}{dt}
\left(
W_\ell W_\ell^\top
-
W_{\ell+1}^\top W_{\ell+1}
\right)
=0.
\]
\end{proof}

Lemma~\ref{lem:gd-balance} holds without alignment and without specializing the loss to full observation.
It is an exact property of continuous time Euclidean gradient flow. This conservation is also showed in \citep{saxe2013exact,gunasekar2017implicit,arora2019implicit}.
For finite-step gradient descent, the first-order terms still cancel, but an imbalance of order equal to the square of the step size can be introduced at each update.

For non-Euclidean optimizers, the matrix identity
$
\frac{d}{dt}\mathcal D_\ell=0
$
does not hold in general.
Balance is instead preserved in our analysis because \Cref{assump:balancedness,assump:aligned_modes,assump:compatible_updates} define an invariant aligned manifold on which all layers have equal mode amplitudes, equal optimizer states, and equal mode updates.
We establish this property below.

\subsection{Consequences of Balance and Alignment}
\label{app:balance_consequence}
Under \Cref{assump:aligned_modes}, the target and the layer factors at initialization have the forms
\[
M
=
Q_L
\operatorname{diag}(\mu_1,\ldots,\mu_d)
Q_0^\top, \qquad \mu_i \geq 0, 
\text{and}
\]
\[
W_\ell(0)
=
Q_\ell
\operatorname{diag}
\left(
a_{\ell 1}(0),\ldots,a_{\ell d}(0)
\right)
Q_{\ell-1}^\top, \qquad a_{\ell i}(0)>0.
\]

The corresponding adjacent Gram matrices are
\[
W_\ell(0)W_\ell(0)^\top
=
Q_\ell
\operatorname{diag}
\left(
a_{\ell 1}(0)^2,\ldots,a_{\ell d}(0)^2
\right)
Q_\ell^\top, \qquad \text{and}
\]
\[
W_{\ell+1}(0)^\top W_{\ell+1}(0)
=
Q_\ell
\operatorname{diag}
\left(
a_{\ell+1,1}(0)^2,\ldots,a_{\ell+1,d}(0)^2
\right)
Q_\ell^\top.
\]

\Cref{assump:balancedness} states that these matrices are equal.
Multiplying this equality by $Q_\ell^\top$ on the left and by $Q_\ell$ on the right gives
\[
a_{\ell i}(0)^2
=
a_{\ell+1,i}(0)^2
\]
for every mode $i$ and every adjacent pair of layers.
The positivity assumption removes the sign ambiguity, yielding
$
a_{\ell i}(0)
=
a_{\ell+1,i}(0).
$
Applying this relation across the network gives a common initial mode amplitude
\[
a_{\ell i}(0)=a_i(0)
\qquad
\text{for every }\ell.
\]

Thus \Cref{assump:balancedness,assump:aligned_modes} together imply more than equal layer norms.
They imply equality of the corresponding aligned mode amplitudes across all layers.

\subsection{The Aligned Manifold Before Imposing Balance}

It is useful first to compute the layer gradients for aligned factors without assuming that their amplitudes are equal.
Suppose that at some time $t$,
\[
W_\ell
=
Q_\ell
\operatorname{diag}
\left(
a_{\ell 1},\ldots,a_{\ell d}
\right)
Q_{\ell-1}^\top.
\]
Then the end-to-end matrix $W$ is
\[
W
=
Q_L
\operatorname{diag}
\left(
\sigma_1,\ldots,\sigma_d
\right)
Q_0^\top \qquad \text{where} \qquad 
\sigma_i
=
\prod_{k=1}^L a_{ki}.
\]
As a result, the effective gradient is
\[
G
=
W-M
=
Q_L
\operatorname{diag}
\left(
r_1,\ldots,r_d
\right)
Q_0^\top, \qquad
\text{where} \qquad
r_i=\sigma_i-\mu_i.
\]
For layer $\ell$, define the prefix and suffix mode products
\[
\pi_{\ell i}^{-}
=
\prod_{k=1}^{\ell-1}a_{ki}, \qquad
\text{and} \qquad
\pi_{\ell i}^{+}
=
\prod_{k=\ell+1}^{L}a_{ki}.
\]
The empty products are equal to one.
On the active subspaces,
\[
A_\ell
=
Q_{\ell-1}
\operatorname{diag}
\left(
\pi_{\ell 1}^{-},\ldots,\pi_{\ell d}^{-}
\right)
Q_0^\top, \qquad
\text{and} \qquad
B_\ell
=
Q_L
\operatorname{diag}
\left(
\pi_{\ell 1}^{+},\ldots,\pi_{\ell d}^{+}
\right)
Q_\ell^\top.
\]
Using the layer gradient formula $H_\ell=B_\ell^\top G A_\ell^\top,$
we obtain
\[
\begin{aligned}
H_\ell
&=
Q_\ell
\operatorname{diag}
\left(
\pi_{\ell 1}^{+},\ldots,\pi_{\ell d}^{+}
\right)
Q_L^\top
Q_L
\operatorname{diag}
\left(
r_1,\ldots,r_d
\right)
Q_0^\top
Q_0
\operatorname{diag}
\left(
\pi_{\ell 1}^{-},\ldots,\pi_{\ell d}^{-}
\right)
Q_{\ell-1}^\top
\\
&=
Q_\ell
\operatorname{diag}
\left(
g_{\ell 1},\ldots,g_{\ell d}
\right)
Q_{\ell-1}^\top,
\end{aligned}
\]
where
\[
g_{\ell i}
=
r_i
\pi_{\ell i}^{-}
\pi_{\ell i}^{+}
=
r_i
\prod_{k\neq\ell}a_{ki}.
\]

Thus alignment of the factors with the target implies alignment of every layer gradient with the corresponding pair of layer bases.
This statement does not yet require balance.

\subsection{Preservation of Balance for Compatible Optimizer Flows}

We now explain the role of \Cref{assump:compatible_updates}.
Suppose that, at some time, the amplitudes of mode $i$ agree across layers: $
a_{\ell i}=a_i
\,
\text{for every }\ell.
$
Then
\[
\sigma_i=a_i^L, \qquad
\text{and} \qquad
g_{\ell i}
=
r_i a_i^{L-1},
\]
which is independent of $\ell$.
Hence all layers receive the same scalar gradient for mode $i$.
\Cref{assump:compatible_updates} requires that the optimizer use identical hyperparameters and identical initial mode states in every layer.
It also requires that equal mode gradients and equal mode states produce equal mode updates and equal state derivatives.
Consequently, whenever $a_{\ell i}=a_i $
and the corresponding optimizer states agree across layers, the update has the form
\[
Z_\ell
=
Q_\ell
\operatorname{diag}
\left(
z_1,\ldots,z_d
\right)
Q_{\ell-1}^\top,
\]
with the same coefficient $z_i$ in every layer.
Since
$ \dot W_\ell=-Z_\ell, $ we have
$
\dot a_{\ell i}=-z_i
$
for every layer.
Therefore
\[
\frac{d}{dt}
\left(
a_{\ell i}-a_{\ell+1,i}
\right)
=0.
\]
Equal amplitudes remain equal, and the aligned factorization remains balanced:
$
W_\ell W_\ell^\top
=
W_{\ell+1}^\top W_{\ell+1}.
$
The same argument applies to the optimizer states.
Equal states receive equal state derivatives because the gradients and hyperparameters are equal.
Thus equality of both amplitudes and states is preserved by the joint optimizer flow.
Equivalently, the set
\[
\mathcal M_{\mathrm{bal,align}}
=
\left\{
W_\ell
=
Q_\ell
\operatorname{diag}(a_1,\ldots,a_d)
Q_{\ell-1}^\top
\text{ for every }\ell
\right\},
\]
together with equal mode states across layers, is an invariant manifold of every optimizer satisfying \Cref{assump:compatible_updates}.
This conclusion holds on intervals where the optimizer vector field is well defined and the relevant amplitudes remain positive.
For nonsmooth updates at zero, the continuation is specified separately in the corresponding optimizer result.

\paragraph{Compatibility of the optimizers considered in this work.}

The compatibility condition is exact for the idealized optimizer flows studied in the main text.

For Euclidean GD, $z_i=g_i,$ so equal layer gradients immediately give equal updates.

For coordinatewise SignGD, $z_i=\operatorname{sign}(g_i)$.
Because elementwise sign is not invariant under general orthogonal changes of basis, the aligned bases must coincide with the actual parameter coordinates.
We therefore take $Q_\ell=I$ for SignGD, up to signed coordinate permutations.
The same restriction applies to coordinatewise Adam and to the instantaneous Adam approximation.
With $Q_\ell=I,$ diagonal gradients preserve diagonal first- and second-moment states.
Identical moment initialization and common hyperparameters then give identical states and updates across layers.

For ideal Muon without momentum, the layer gradient in the aligned model $H_\ell= Q_\ell \operatorname{diag}(g_1,\ldots,g_d)
Q_{\ell-1}^\top$ gives $\operatorname{polar}(H_\ell) = Q_\ell \bigl(\operatorname{sign}(g_1),\ldots,\operatorname{sign}(g_d)\bigr) Q_{\ell-1}^\top$ where $\operatorname{sign}(0)=0$.
Therefore, it preserves equality of the scalar states $g_i$ across layers.

For cumulative Shampoo, let the left and right accumulators be initialized isotropically.
An aligned gradient $H_\ell= Q_\ell \operatorname{diag}(g_1,\ldots,g_d)
Q_{\ell-1}^\top$
gives
\[
H_\ell H_\ell^\top
=
Q_\ell
\operatorname{diag}(g_1^2,\ldots,g_d^2)
Q_\ell^\top, \qquad
\text{and} \qquad
H_\ell^\top H_\ell
=
Q_{\ell-1}
\operatorname{diag}(g_1^2,\ldots,g_d^2)
Q_{\ell-1}^\top.
\]
The accumulator eigenvalues therefore remain aligned and equal across layers.
The preconditioned update consequently has the compatible form
$
Z_\ell
=
Q_\ell
\operatorname{diag}(z_1,\ldots,z_d)
Q_{\ell-1}^\top.
$

For the block-curvature model of K-FAC, the curvature eigenvalue of mode $i$ in every layer is $\sigma_i^{2-2/L}$ under balance.
Common additive damping therefore produces the same preconditioned scalar update in every layer.

These statements concern the idealized flows specified in the main text.
They do not imply that arbitrary practical implementations preserve balance away from the aligned manifold.
In particular, coordinatewise methods are generally basis dependent, and practical K-FAC damping applied separately to the two Kronecker factors need not preserve the same layer symmetry.

\subsection{Why the Alignment Assumption is Appropriate}

\Cref{assump:aligned_modes} does not claim that arbitrary training trajectories become aligned with the target.
It restricts the analysis to a canonical invariant subspace in which exact mode dynamics can be derived.
This restriction is standard in analyses of deep linear networks and deep matrix factorization
\citep{saxe2013exact,gunasekar2017implicit,arora2019implicit,li2021towards}.
It has three useful properties.
First, full-observation regression fixes a target basis through the singular directions of $M$.
If the factors and optimizer states are initialized in compatible chain bases, the layer gradients remain in those bases.
Second, fixing the same initial spectral configuration for every optimizer ensures that differences in mode evolution arise from the optimizer update and its state, rather than from different initial singular directions.
Third, the aligned model retains the nonlinear multiplicative dependence of the effective modes on the factors.
It therefore preserves the depth-dependent mobility responsible for the phase laws while removing singular-vector rotation from the comparison.

For spectral and curvature-based optimizers, the aligned subspace is preserved by their matrix-equivariant updates.
For coordinatewise methods, the same construction is exact when the aligned basis is the parameter coordinate basis.
The general matrix dynamics outside this invariant subspace are given in the previous appendix section.

If the target contains repeated singular values, the singular basis is not unique within the corresponding degenerate subspace.
Any fixed orthonormal basis in that subspace may be chosen as the reference basis.
The theorem tracks modes in these fixed reference bases and does not require their amplitudes to remain sorted by magnitude.

\subsection{Proof of \Cref{prop:aligned_dynamics}}

\begin{proof}

By \Cref{assump:balancedness,assump:aligned_modes}, the layer amplitudes agree at initialization: $a_{\ell i}(0)=a_i(0).$
By \Cref{assump:compatible_updates} and the invariant-manifold argument above, this equality persists throughout training: $a_{\ell i}(t)=a_i(t)$ for every layer $\ell$ and every time before a mode reaches zero.
Hence
\[
W_\ell(t)
=
Q_\ell
\operatorname{diag}
\left(
a_1(t),\ldots,a_d(t)
\right)
Q_{\ell-1}^\top.
\]
Multiplying the factors gives
\[
\begin{aligned}
W
&=
W_L\cdots W_1
\\
&=
Q_L
\operatorname{diag}(a_1,\ldots,a_d)
Q_{L-1}^\top
Q_{L-1}
\operatorname{diag}(a_1,\ldots,a_d)
Q_{L-2}^\top
\cdots
Q_1
\operatorname{diag}(a_1,\ldots,a_d)
Q_0^\top
\\
&=
Q_L
\operatorname{diag}
\left(
a_1^L,\ldots,a_d^L
\right)
Q_0^\top.
\end{aligned}
\]
We can rewrite it as 
$$
W
=
Q_L
\operatorname{diag}
\left(
\sigma_1,\ldots,\sigma_d
\right)
Q_0^\top, \qquad \text{where} \qquad
\sigma_i=a_i^L.
$$
For the full-observation loss,
\[
G
=
W-M
=
Q_L
\operatorname{diag}
\left(
r_1,\ldots,r_d
\right)
Q_0^\top, \qquad
\text{where} \qquad
r_i=\sigma_i-\mu_i.
\]
Under the common layer amplitudes,
\[
A_\ell
=
Q_{\ell-1}
\operatorname{diag}
\left(
a_1^{\ell-1},\ldots,a_d^{\ell-1}
\right)
Q_0^\top, \qquad
\text{and} \qquad
B_\ell
=
Q_L
\operatorname{diag}
\left(
a_1^{L-\ell},\ldots,a_d^{L-\ell}
\right)
Q_\ell^\top.
\]
Using $H_\ell = B_\ell^\top G A_\ell^\top,$ we obtain
\[
\begin{aligned}
H_\ell
&=
Q_\ell
\operatorname{diag}
\left(
a_1^{L-\ell},\ldots,a_d^{L-\ell}
\right)
\operatorname{diag}
\left(
r_1,\ldots,r_d
\right)
\operatorname{diag}
\left(
a_1^{\ell-1},\ldots,a_d^{\ell-1}
\right)
Q_{\ell-1}^\top
\\
&=
Q_\ell
\operatorname{diag}
\left(
r_1a_1^{L-1},\ldots,r_da_d^{L-1}
\right)
Q_{\ell-1}^\top.
\end{aligned}
\]
Let $q=1-\frac1L$, then $a_i^{L-1} = \left(a_i^L\right)^{(L-1)/L} = \sigma_i^q,$ and the layer gradient has the form
\begin{align}
H_\ell
=
Q_\ell
\operatorname{diag}
\left(
g_1,\ldots,g_d
\right)
Q_{\ell-1}^\top, \qquad
\text{where} \qquad
g_i
=
r_i\sigma_i^q.
\label{eq:H_l_aligned}
\end{align}
The coefficient $g_i$ is independent of the layer index $\ell$.

By \Cref{assump:compatible_updates}, the optimizer update is
$
Z_\ell
=
Q_\ell
\operatorname{diag}
\left(
z_1,\ldots,z_d
\right)
Q_{\ell-1}^\top,
$
with the same $z_i$ in every layer.
Since $\dot W_\ell=-Z_\ell,$ comparison of the aligned coefficients gives
\[
\dot a_i=-z_i.
\]
Differentiate $\sigma_i=a_i^L.$
This yields
\[
\begin{aligned}
\dot\sigma_i
&=
La_i^{L-1}\dot a_i
\\
&=
-La_i^{L-1}z_i
\\
&=
-L\sigma_i^qz_i.
\end{aligned}
\]
Therefore $\dot\sigma_i=-L\sigma_i^qz_i.$
The same result follows directly from the effective-matrix dynamics derived in \Cref{app:eff_W_dyn}:
\[
\begin{aligned}
\dot W
&=
-\sum_{\ell=1}^L
B_\ell Z_\ell A_\ell
\\
&=
-\sum_{\ell=1}^L
Q_L
\operatorname{diag}
\left(
a_1^{L-\ell},\ldots,a_d^{L-\ell}
\right)
\operatorname{diag}
\left(
z_1,\ldots,z_d
\right)
\operatorname{diag}
\left(
a_1^{\ell-1},\ldots,a_d^{\ell-1}
\right)
Q_0^\top
\\
&=
-L
Q_L
\operatorname{diag}
\left(
a_1^{L-1}z_1,\ldots,a_d^{L-1}z_d
\right)
Q_0^\top
\\
&=
Q_L
\operatorname{diag}
\left(
\dot\sigma_1,\ldots,\dot\sigma_d
\right)
Q_0^\top.
\end{aligned}
\]

Finally, define $ p=2q=2-\frac2L.$
This exponent appears whenever the optimizer update contains one additional factor of the layer gradient $g_i=r_i\sigma_i^q.$
For example, Euclidean GD has $z_i=g_i,$
and therefore
$
\dot\sigma_i
=
-L\sigma_i^qg_i
=
-L\sigma_i^p(\sigma_i-\mu_i).
$
\end{proof}

\section{Optimizer Specific Phase Laws}
\label{app:phase_laws}

In this section, we prove the phase laws for all the optimizers discussed in \Cref{sec:results}.
\Cref{prop:aligned_dynamics} reduces every compatible optimizer to a scalar evolution for each aligned mode.
Suppressing the mode index, the effective singular amplitude satisfies
\[
\dot{\sigma}
=
-L\sigma^q z(\sigma,S(\sigma)),
\qquad
q=1-\frac1L,
\]
where $z$ is the optimizer update coefficient and $S$ denotes any associated optimizer state.
The state $S(\sigma)$ may include momentum, adaptive moments, accumulated gradient covariances, or curvature estimates.
Before deriving the individual optimizer results, we establish a common passage time calculus for birth, decay, and saturation.
Each optimizer specific phase law then follows by identifying the local form of $z$, or equivalently the local power law of $\dot{\sigma}$.
This separates the universal passage time argument from the optimizer specific evolution of its internal state.

We use asymptotic notations $o(.)$ and $O(.)$ when taking limits in the asymptotic analysis. 

\subsection{Birth, Decay, and Saturation times}

As discussed in \Cref{sec:results}, for a positive target $\mu>0$, let a mode begin at $\sigma(0)=s>0.$
For a fixed threshold $s<b<\mu,$ the birth time is the first passage time
\[
T_{\mathrm{opt}}^{\mathrm{birth}}(s,b)
=
\inf
\left\{
t\geq0:
\sigma(t)=b
\right\}.
\]
Birth always refers to escape from a positive seed $s\to0^+$.
It does not refer to spontaneous activation of an exactly zero factor mode.
For an unsupported mode with $\mu=0,$ and a tolerance $0<\delta<s,$ the decay time is
\[
T_{\mathrm{opt}}^{\mathrm{decay}}(s,\delta)
=
\inf
\left\{
t\geq0:
\sigma(t)=\delta
\right\}.
\]
For a nonzero target, define the residual $r(t)=|\sigma(t)-\mu|.$
Given a tolerance $\tau>0$ and $r(t_0)=r_0$, the saturation time is
\[
T_{\mathrm{opt}}^{\mathrm{sat}}(r_0, \tau)=
\inf\left\{
u\geq0:
r(t_0+v)\leq\tau
\text{ for every }v\geq u
\right\}.
\]

The requirement that the mode remain inside the tolerance is relevant for optimizers with momentum or relay dynamics, which may cross the target and subsequently leave the tolerance region.
For a monotone first order flow, this definition reduces to the first time at which $r(t)=\tau.$


We first state the passage time formula directly in terms of the optimizer update coefficient.

\begin{proposition}[Passage times from the common mode dynamics]
\label{prop:general_passage_time}
Suppose that a positive target mode increases monotonically from $s$ to $b$ and satisfies $z(\sigma,S(\sigma))<0$ for $\sigma\in[s,b].$
Then
\[
T_{\mathrm{opt}}^{\mathrm{birth}}(s,b)
=
-\frac1L
\int_s^b
\frac{
x^{-q}
}{
z(x,S(x))
}
\,dx.
\]
Suppose that an unsupported mode decreases monotonically from $s$ to $\delta$ and satisfies $z(\sigma,S(\sigma))>0$ for $\sigma\in[\delta,s].$
Then
\[
T_{\mathrm{opt}}^{\mathrm{decay}}(s,\delta)
=
\frac1L
\int_\delta^s
\frac{
x^{-q}
}{
z(x,S(x))
}
\,dx.
\]
If the residual decreases monotonically and satisfies $\dot r=-\Psi(r),$ and $\Psi(r)>0,$ then the time required to reduce the residual from $r_0$ to $\tau$ is
\[
T_{\mathrm{opt}}^{\mathrm{sat}}(r_0,\tau)
=
\int_\tau^{r_0}
\frac{dr}{\Psi(r)}.
\]
\end{proposition}

\begin{proof}
During birth, $\dot\sigma = -L\sigma^qz(\sigma,S) >0.$
Since the trajectory is monotone, the optimizer state can be viewed locally as a function of $\sigma$.
Separating variables gives
\[
dt
=
-\frac1L
\frac{
\sigma^{-q}
}{
z(\sigma,S(\sigma))
}
\,d\sigma.
\]
Integrating from $\sigma=s$ to $\sigma=b$ proves the birth formula.
During decay, $\dot\sigma = -L\sigma^qz(\sigma,S) <0.$
Thus
\[
dt
=
-\frac1L
\frac{
\sigma^{-q}
}{
z(\sigma,S(\sigma))
}
\,d\sigma.
\]
Integrating from $\sigma=s$ down to $\sigma=\delta$ and reversing the integral limits proves the decay formula.
The saturation formula follows similarly from
\[
dt
=
-\frac{dr}{\Psi(r)}
\]
and integration from $r_0$ to $\tau$.
\end{proof}
For a stateful optimizer, Proposition~\ref{prop:general_passage_time} is exact but need not be closed.
The notation $S(x)$ denotes the optimizer state along the particular monotone trajectory.
A closed phase law requires either solving the coupled state equations or eliminating the state through an invariant.
Cumulative Shampoo admits such an elimination through its accumulator identity.
Full Adam and momentum Muon generally retain coupled state dynamics.

\subsection{Birth and Decay Exponents}

We now classify passage times using the local power of the complete singular mode dynamics.
The exponent is defined from $\dot\sigma$, rather than from the factor update $z$ alone.

\begin{definition}[Birth and decay exponents]
\label{def:phase_exponents}

A positive target mode has birth exponent $\alpha_{\mathrm b}$ and birth coefficient $c_{\mathrm b}>0$ if
\[
\dot\sigma
=
c_{\mathrm b}\sigma^{\alpha_{\mathrm b}}
\left(
1+o(1)
\right)
\qquad
\text{as }\sigma\to0^+,
\]
where $o(1)$ denotes a quantity tending to zero in the stated limit. 
An unsupported mode has decay exponent $\alpha_{\mathrm d}$ and decay coefficient $c_{\mathrm d}>0$ if
\[
\dot\sigma
=
-c_{\mathrm d}\sigma^{\alpha_{\mathrm d}}
\left(
1+o(1)
\right)
\qquad
\text{as }\sigma\to0^+.
\]
\end{definition}
The coefficients $c_{\mathrm b}$ and $c_{\mathrm d}$ may depend on the target, depth, optimizer parameters, and limiting optimizer state.
The exponent records only the leading dependence on the current singular amplitude.
A stateful optimizer may have different exponents in different phases.
The following elementary calculation determines whether the corresponding passage time remains finite or diverges.

\begin{proposition}[Birth and decay times]
\label{prop:power_passage}

Suppose that an increasing mode satisfies $\dot\sigma = c\sigma^\alpha,$ with $c>0.$
Then the time required to move from $s$ to $b>s$ is
\[
T^{\mathrm{birth}}(s,b)
=
\begin{cases}
\dfrac1c\log\dfrac{b}{s},
& \alpha=1,
\\[3mm]
\dfrac{
b^{1-\alpha}-s^{1-\alpha}
}{
c(1-\alpha)
},
& \alpha\neq1.
\end{cases}
\]
Suppose that a decreasing mode satisfies $\dot\sigma = -c\sigma^\alpha,$ with $c>0.$
Then the time required to move from $s$ to $\delta<s$ is
\[
T^{\mathrm{decay}}(s,\delta)
=
\begin{cases}
\dfrac1c\log\dfrac{s}{\delta},
& \alpha=1,
\\[3mm]
\dfrac{
s^{1-\alpha}-\delta^{1-\alpha}
}{
c(1-\alpha)
},
& \alpha\neq1.
\end{cases}
\]
\end{proposition}

\begin{proof}

For birth, $dt = \frac1c\sigma^{-\alpha}d\sigma. $
Therefore
\[
T^{\mathrm{birth}}(s,b)
=
\frac1c
\int_s^b
x^{-\alpha}\,dx.
\]
The integral is logarithmic when $\alpha=1$ and gives the stated power law when $\alpha\neq1$.
For decay, $dt = -\frac1c\sigma^{-\alpha}d\sigma.$
Integrating from $s$ down to $\delta$ gives
\[
T^{\mathrm{decay}}(s,\delta)
=
\frac1c
\int_\delta^s
x^{-\alpha}\,dx,
\]
which yields the stated formula.
\end{proof}
The same classifications hold when the dynamics are only asymptotically proportional to a power.


\begin{corollary}[Birth time classification]
\label{cor:birth_classification}

Suppose $\dot\sigma = c_{\mathrm b}\sigma^{\alpha_{\mathrm b}}
\left( 1+o(1) \right)$ as $\sigma\to0^+.$
For fixed $b>0$, the birth time satisfies
\[
T^{\mathrm{birth}}(s,b)
=
\begin{cases}
T^{\mathrm{birth}}(0,b)
-
\dfrac{
s^{1-\alpha_{\mathrm b}}
}{
c_{\mathrm b}(1-\alpha_{\mathrm b})
}
\left(
1+o(1)
\right),
& \alpha_{\mathrm b}<1,
\\[3mm]
\dfrac1{c_{\mathrm b}}
\log\dfrac1s (1 + o(1)),
& \alpha_{\mathrm b}=1,
\\[3mm]
\dfrac{
s^{1-\alpha_{\mathrm b}}
}{
c_{\mathrm b}(\alpha_{\mathrm b}-1)
}
\left(
1+o(1)
\right),
& \alpha_{\mathrm b}>1.
\end{cases}
\]
If $\alpha_{\mathrm b}<1$ gives finite escape from a vanishing positive seed.

If $\alpha_{\mathrm b}=1$ gives a logarithmic birth barrier.

If $\alpha_{\mathrm b}>1$ gives a polynomial birth barrier.
\end{corollary}

Finite escape as $s\to0^+$ does not imply spontaneous activation from $\sigma(0)=0$.
If the vector field vanishes at zero and the optimizer takes zero to be absorbing, an exactly zero mode remains zero even when $\alpha_{\mathrm b}<1$.
The birth exponent concerns the limit from positive initialization.

\begin{corollary}[Decay time classification]
\label{cor:decay_classification}

Suppose $\dot\sigma=-c_{\mathrm d}\sigma^{\alpha_{\mathrm d}} 
\left(1+o(1)\right)$ as $\sigma\to0^+.$
For fixed $s>0$, the decay time satisfies
\[
T^{\mathrm{decay}}(s,\delta)
=
\begin{cases}
T^{\mathrm{decay}}(s,0)
-
\dfrac{
\delta^{1-\alpha_{\mathrm d}}
}{
c_{\mathrm d}(1-\alpha_{\mathrm d})
}
\left(
1+o(1)
\right),
& \alpha_{\mathrm d}<1,
\\[3mm]
\dfrac1{c_{\mathrm d}}
\log\dfrac1\delta (1+o(1)),
& \alpha_{\mathrm d}=1,
\\[3mm]
\dfrac{
\delta^{1-\alpha_{\mathrm d}}
}{
c_{\mathrm d}(\alpha_{\mathrm d}-1)
}
\left(
1+o(1)
\right),
& \alpha_{\mathrm d}>1.
\end{cases}
\]
If $\alpha_{\mathrm d}<1$, the mode reaches zero in finite time under the absorbing continuation.

If $\alpha_{\mathrm d}=1$, the mode decays exponentially:
$
\sigma(t)
= \exp{\lp -c_{\mathrm d} t+o(t)\rp}
$

If $\alpha_{\mathrm d}>1$, the mode decays algebraically:
$
\sigma(t)
\sim
\left(
c_{\mathrm d}
\left(
\alpha_{\mathrm d}-1
\right)t
\right)^{-1/(\alpha_{\mathrm d}-1)}.
$

In the final two cases, zero is reached only asymptotically.
\end{corollary}

The terms finite, logarithmic, and polynomial refer to the dependence of the passage time on $s$ or $\delta$.
The terms exponential and algebraic refer to the trajectory as a function of training time.
For example, $\alpha_{\mathrm d}=1$ gives logarithmic tolerance time and exponential trajectory decay.

\subsection{Saturation Exponent}

Birth and decay describe behavior near $\sigma=0$.
Saturation instead describes the local approach to a nonzero target $\mu$.

\begin{definition}[Saturation exponent]
\label{def:saturation_exponent}

Let $r(t) =  |\sigma(t)-\mu|.$
The mode has saturation exponent $\gamma$ and coefficient $\beta>0$ if
\[
\dot r
=
-\beta r^\gamma
\left(
1+o(1)
\right)
\qquad
\text{as }r\to0^+.
\]
\end{definition}

\begin{proposition}[Saturation times]
\label{prop:saturation_phase_law}

Suppose that the residual satisfies
$
\dot r
=
-\beta r^\gamma,
\qquad
\beta>0.
$
The time required to reduce the residual from $r_0$ to $\tau\in(0,r_0)$ is
\[
T^{\mathrm{sat}}(r_0,\tau)
=
\begin{cases}
\dfrac1\beta
\log\dfrac{r_0}{\tau},
& \gamma=1,
\\[3mm]
\dfrac{
r_0^{1-\gamma}-\tau^{1-\gamma}
}{
\beta(1-\gamma)
},
& \gamma\neq1.
\end{cases}
\]

If $\gamma<1$, the target is reached in finite time.

If $\gamma=0$, the residual decreases linearly:
$
r(t)
=
\left(
r_0-\beta t
\right)_+,
$
and $\beta$ is the constant residual speed.

If $\gamma=1$, the residual decreases exponentially:
$
r(t)
=
r_0e^{-\beta t},
$
and $\beta$ is the exponential saturation rate.

If $\gamma>1$, the residual has an algebraic tail:
$
r(t)
\sim
\left(
\beta(\gamma-1)t
\right)^{-1/(\gamma-1)}.
$
\end{proposition}

\begin{proof}

Separating variables gives
$
dt
=
-\frac1\beta r^{-\gamma}dr.
$
Integrating from $r_0$ down to $\tau$ yields
\[
T^{\mathrm{sat}}(r_0,\tau)
=
\frac1\beta
\int_\tau^{r_0}
r^{-\gamma}\,dr.
\]
Evaluating the integral gives the passage time formula.
Solving the same separated equation for $r(t)$ yields the stated trajectory laws.
\end{proof}

For $\gamma=0$, the coefficient $\beta$ is a speed rather than an exponential rate.
For $\gamma=1$, it is the usual local exponential rate.
For general $\gamma$, we refer to $\beta$ as the saturation coefficient.
The scalar saturation exponent is most useful for monotone first order flows.
If the optimizer state introduces overshoot or oscillation, the residual may not satisfy a closed first order equation.
In that case, saturation must be analyzed through the full state dynamics, and the enter and remain definition distinguishes true saturation from a temporary crossing of the tolerance band.

The optimizer specific proofs that follow therefore have two steps.
First, we derive or eliminate the relevant optimizer state to identify $z$ along the phase.
Second, we apply \Cref{prop:power_passage,prop:saturation_phase_law} to obtain the corresponding birth, decay, and saturation times.

\subsection{Euclidean Gradient Descent Phase Laws}

The update step of {Euclidean GD} is $Z_\ell = H_\ell$, and so in the aligned model $z_i = g_i = r_i \sigma^q$ from \Cref{eq:H_l_aligned}. Therefore, each mode evolves as $$\dot\sigma=L\sigma^p(\mu-\sigma).$$ Note that this is the established GD baseline \citep{arora2019implicit}.

\begin{proposition}[Euclidean GD phase laws]
\label{prop:gd}

The birth time $T_{\mathrm{GD}}^{\mathrm{birth}}(s,b)$ for $0<s<b<\mu$ with $b$ and $\mu$ fixed is
$$
T_{\mathrm{GD}}^{\mathrm{birth}}(s,b)
=
\begin{cases}
\log \lp \dfrac{\mu-s}{\mu-b} \rp, 
& L=1,
\\[3mm]
\dfrac{1}{2\mu}
\log\dfrac{b(\mu-s)}{s(\mu-b)},
& L=2,
\\[3mm]
\dfrac{s^{-(L-2)/L} - b^{-(L-2)/L}}{\mu(L-2)} + O(1),
& L>2.
\end{cases}
$$
For $\mu=0$, decay time $T_{\mathrm{GD}}^{\mathrm{decay}}(s,\delta)$ with $s$ fixed is
$$
T_{\mathrm{GD}}^{\mathrm{decay}}(s,\delta) =
\begin{cases}
\log \dfrac{s}{\delta}, & L=1,\\[1mm]
\dfrac{\delta^{-p}-s^{-p}}{Lp}, & L>1.
\end{cases}
$$
Near a nonzero target $\mu$,
\[
T_{\mathrm{GD}}^{\mathrm{sat}}(r_0,\delta)
=
\frac1{L\mu^p}
\log\frac{r_0}{\delta} + O(1).
\]

For nonzero target modes $\mu>0$, the birth exponent $\alpha_{\mathrm{b}}=p$ and coefficient $c_\mathrm{b} = L\mu$ with exponential local saturation $\gamma=1$ at rate $\beta=L\mu^p$ near target $\mu$. 
As $s \to 0$, birth time $T_{\mathrm{GD}}^{\mathrm{birth}}(s,b)$ diverges for $L\ge 2$.
For zero target modes $\mu=0$, the decay exponent $\alpha_{\mathrm{d}}=p+1$ and coefficient $c_\mathrm{d} = L$. 
As $\delta \to 0$, decay time $T_{\mathrm{GD}}^{\mathrm{decay}}(s,\delta)$ diverges for all $L$.
\end{proposition}

\begin{proof} \,
\paragraph{Birth and decay times.}

From the mode dynamics $\dot\sigma=L\sigma^p(\mu-\sigma)$, the exact birth and decay times are
\[
T_{\mathrm{GD}}^{\mathrm{birth}}(s,b)
=
\frac1L
\int_s^b
\frac{
x^{-p}
}{
\mu-x
}
\,dx, \qquad
T_{\mathrm{GD}}^{\mathrm{decay}}(s,\delta) = - \frac{1}{L}
\int_s^\delta x^{-(p+1)} \, dx.
\]
For birth time, especially for $L>2$, we can decompose it to, 
\begin{align*}
T_{\mathrm{GD}}^{\mathrm{birth}}(s,b)
&=
\frac1L
\int_s^b
\frac{
x^{-p}
}{
\mu-x
}
\,dx \\
&=
\frac1L
\int_s^b
\frac{
\mu x^{-p} - x^{1-p} + x^{1-p}
}{
\mu(\mu-x)
}
\,dx \\
&=
\frac1{L\mu}
\int_s^b
x^{-p}
\,dx +
\frac1{L\mu}
\int_s^b
\frac{ x^{1-p}
}{
\mu-x
}
\,dx \\
&= \frac1{L\mu}
\int_s^b
x^{-p}
\,dx + O(1)
\end{align*}
Integrating the birth and decay times using \Cref{prop:power_passage} give the times in \Cref{prop:gd}. From the dynamics and \Cref{cor:birth_classification,cor:decay_classification}, we get the birth and decay exponents: $\alpha_{\mathrm{b}}=p$ with $c_{\mathrm{b}}=L\mu$ and $\alpha_{\mathrm{d}}=p+1$ with $c_{\mathrm{d}}=L$.

From birth time $T_{\mathrm{GD}}^{\mathrm{birth}}$  and \Cref{cor:birth_classification}, as $s \to 0^+$, $L=1$ has no initialization-dependent birth barrier and escapes in finite time as $p=2-2/L=0$.  
Thus the mode can escape from an arbitrarily small initialization in a uniformly bounded amount of time. In the corresponding birth approximation, $\dot{\sigma} \approx \mu_i$, so the mode grows linearly: $\sigma(t) \approx \epsilon + \mu_i t$.
On the other hand, $L=2$ introduces a logarithmic delay as $p=1$. The corresponding mode trajectory is exponential: $\dot{\sigma}\approx 2\mu\sigma $ and hence
$ \sigma(t)\approx \epsilon e^{2\mu_i t}$, therefore the singular mode itself grows exponentially in time.
Similarly, $L>2$ produces a polynomial delay as $p>1$.

As $\delta \to 0^+$, $T_{\mathrm{GD}}^{\mathrm{decay}}$ diverges from \Cref{cor:decay_classification}, since $\alpha_{\mathrm{d}}=p+1 \geq 1$. Equality holds for $L=1$, while $\alpha_d>1$ for $L>1$. Thus GD does not kill zero-target singular values in finite time.\\  
For \(L=1\), $\alpha_{\mathrm{d}}=1$, therefore the corresponding $\sigma(t)$ trajectory is exponential and the decay time is logarithmic in \(1/\delta\).
The state decays exponentially in time. It never reaches zero at a finite time. \\
For \(L>1\), it is a polynomial divergence in \(1/\tau\).
The exact trajectory is
$ \sigma(t) = \left( \sigma_0^{-p}+Lp\,t \right)^{-1/p}. $
Therefore, at late times,
$\sigma(t) \sim (Lp\,t)^{-1/p}.$
So decay is algebraic rather than exponential.

\paragraph{Saturation near target.}
Suppose $\sigma(t) \approx \mu >0$. Let 
$r(t) = \mu - \sigma(t)$, therefore $ \dot{r} = - \dot{\sigma}$.
Then, from the dynamics $$ \dot{r} = - L \sigma^p r.$$

For approach to a positive target from below and above, respectively,
$$
T_{\mathrm{GD},-}^{\mathrm{sat}}(r_0,\delta)
=
\frac1L\int_\delta^{r_0}\frac{dr}{(\mu-r)^p r},
\qquad \text{ and } \qquad
T_{\mathrm{GD},+}^{\mathrm{sat}}(r_0,\delta)
=
\frac1L\int_\delta^{r_0}\frac{dr}{(\mu+r)^p r}.
$$
The first formula requires $0<\delta<r_0<\mu$.
For fixed $r_0$,
$$
T_{\mathrm{GD},\pm}^{\mathrm{sat}}(r_0,\delta)
=
\frac1{L\mu^p}\log\frac{r_0}{\delta}+O(1)
\qquad\text{as }\delta\to0.
$$

Thus, the convergence is exponential with rate $L \mu^p$ from \Cref{prop:saturation_phase_law}. The larger $\mu$ gives a larger local saturation rate only for $L>1$. At $L=1$, the rate is one for every target.
The exponential saturation implies that modes with smaller target singular values also converge slowly. 

\end{proof}

\paragraph{Role of depth.}
Depth strongly suppresses the birth of new modes, strengthening the initialization barrier. For \(L>2\), the divergence of birth time is polynomial. Thus an infinitesimally initialized singular mode takes polynomially longer to become macroscopically active. 
This is the primary mechanism behind sequential low-rank learning.

Depth does not rapidly eliminate existing modes. 
Decay time diverges for all depth. Indeed:
\(L=1\) gives an exponential tail;
\(L>1\) gives the slower polynomial tail.
Thus Euclidean GD does not produce exact finite-time rank reduction. Its low-rank bias comes mainly from preventing weak modes from being born, rather than from rapidly pruning modes that are already active.

\subsection{SignGD Phase Laws}
{SignGD} updates each factor using $Z_\ell=\operatorname{sign}(H_\ell)$. 
Therefore, in the balanced aligned model, this gives $z_i = \sign(g_i)=\sign(r_i \sigma^q) = \sign(r_i)$ since $\sigma$ is positive. Therefore, for each mode
$$
\dot\sigma = -L\sigma^q\operatorname{sign}(\sigma-\mu).
$$

Thus, each factor amplitude moves toward its target at unit speed, independently of the gradient magnitude.
\begin{proposition}[SignGD phase laws] The birth time $T_{\mathrm{SignGD}}^{\mathrm{birth}}(s,b)$ for $0<s<b<\mu$, and decay time $T_{\mathrm{SignGD}}^{\mathrm{decay}}(s,\delta)$ for $\mu=0$ are
$$
T_{\mathrm{SignGD}}^{\mathrm{birth}}(s,b)=b^{1/L}-s^{1/L}, 
\qquad  \qquad
T_{\mathrm{SignGD}}^{\mathrm{decay}}(s,\delta)=s^{1/L}-\delta^{1/L}.
$$
If a mode approaches a nonzero target from above and below and its residual decreases from $r_0$ to $\delta$, then
$$
T_{\mathrm{SignGD},-}^{\mathrm{sat}}(r_0,\delta)
=
(\mu-\delta)^{1/L}-(\mu-r_0)^{1/L}
\quad \text{and} \quad
T_{\mathrm{SignGD},+}^{\mathrm{sat}}(r_0,\delta)
=
(\mu+r_0)^{1/L}-(\mu+\delta)^{1/L}.
$$
In particular,
$
T_{\mathrm{SignGD},\pm}^{\mathrm{sat}}(r_0,0)
<
\infty.
$
The same finite-time conclusion holds when the target is approached from above.
The birth and decay exponents are $\alpha_\mathrm{b}=\alpha_\mathrm{d} = q$ and coefficients $c_\mathrm{b}=c_\mathrm{d} = L$. The saturation is linear $\gamma=0$, with rate $\beta=L\mu^q$.
\label{prop:signgd}
\end{proposition}

\begin{proof} \,
\paragraph{Birth, saturation and decay times.}
From the dynamics, $\dot\sigma = -L\sigma^q\operatorname{sign}(\sigma-\mu)$, when $0<s<b<\mu$, $\sign(\sigma-\mu)=-1$.  Therefore, $\dot\sigma = L\sigma^q$. When $\mu=0$, $\sign(\sigma-\mu)=1$ since we consider infinitesimal positive initialization. Therefore, $\dot\sigma = -L\sigma^q$. Thus
$$T_{\mathrm{SignGD}}^{\mathrm{birth}}(s,b)= \frac{1}{L}\int_s^b x^{-q} \, dx \qquad
T_{\mathrm{SignGD}}^{\mathrm{decay}}(s,\delta)= -\frac{1}{L}\int_s^\delta x^{-q} \, dx.
$$
Integrating these give the times in \Cref{prop:signgd}.
From the dynamics and \Cref{cor:birth_classification,cor:decay_classification}, we get the birth and decay exponents: $\alpha_{\mathrm{b}}=q$ with $c_{\mathrm{b}}=L$ and $\alpha_{\mathrm{d}}=q$ with $c_{\mathrm{d}}=L$. Note that the coefficients are target independent.

Since the exponent $q<1$ and the time to reach any target $\mu$ is $|s^{1/L}-\mu^{1/L}|$, birth, saturation and decay occur in \emph{finite} time. 
For $\mu=0$, the mode becomes exactly zero at time $s^{1/L}$.
The birth time remains bounded as $s \to 0$. 
Therefore, SignGD weakens the birth barrier for nonzero modes, but it can be more aggressively sparsifying for modes whose target is exactly zero. It still slows small singular values through $\sigma^{1-1/L}$ but much less than GD.

The exact saturation times from below and above are $\dfrac1L \displaystyle\int_\delta^{r_0} \dfrac{dr}{(\mu \pm r)^q}$.
$$
T_{\mathrm{SignGD},-}^{\mathrm{sat}}(r_0,\delta)
=
(\mu-\delta)^{1/L}-(\mu-r_0)^{1/L}
\quad \text{and} \quad
T_{\mathrm{SignGD},+}^{\mathrm{sat}}(r_0,\delta)
=
(\mu+r_0)^{1/L}-(\mu+\delta)^{1/L}.
$$
Locally,
$
\dot r=-L\mu^q+O(r).
$
Thus $\gamma=0$ and $\beta=L\mu^q$.
For $L>1$, the residual is only asymptotically linear near the hitting time from \Cref{prop:saturation_phase_law}.

\end{proof}

\subsection{Instantaneous Adam Phase Laws}
{Adam} updates are considered under a quasi-steady approximation and a diagonal aligned model to isolate the effect of coordinate normalization, with 
$m=g$ and $v=g^2$  giving the instantaneous state approximation where
$$
\dot\sigma=-\frac{L\sigma^p(\sigma-\mu)}{|\sigma-\mu|\sigma^q+\epsilon}
$$
for each mode.
Note that, if $|\sigma-\mu|\sigma^q \gg \epsilon$ then $\dot\sigma = -L\sigma^q\operatorname{sign}(\sigma-\mu)$, thus behaving like SignGD. 
If $|\sigma-\mu|\sigma^q \ll \epsilon$ then $\dot\sigma = -L\sigma^p(\sigma-\mu)/\epsilon$, thus behaving like GD upto $1/\epsilon$ factor. 

\begin{proposition}[Instantaneous Adam phase laws]
\label{prop:iadam-phases}
Let $0<s<b<\mu$.
The birth time is
\[
T_{\mathrm{IAdam}}^{\mathrm{birth}}(s,b)
=
b^{1/L}-s^{1/L}
+
\epsilon \, T_{\mathrm{GD}}^{\mathrm{birth}}(s,b).
\]

Let $h_\mu(\sigma)=(\mu-\sigma)\sigma^q$, when $\epsilon < \mu^{q+1} \dfrac{q^q}{(q+1)^{q+1}}$ and $h_\mu(\sigma)/\epsilon \to 0$, Adam behaves GD-like, and $\epsilon/h_\mu(\sigma) \to 0$ it behaves like SignGD-like.

For an unsupported mode with $\mu=0$, the decay time from $s$ to $\delta\in(0,s)$ is
\[
T_{\mathrm{IAdam}}^{\mathrm{decay}}(s,\delta)
=
s^{1/L}-\delta^{1/L} +
\epsilon \, T_{\mathrm{GD}}^{\mathrm{decay}}(s,\delta).
\]
Let $\sigma_\mathrm{A}^{\mathrm{d}} \asymp \epsilon^{1/(q+1)}$, 
the final decay satisfies $\sigma(t) < \sigma_\mathrm{A}^{\mathrm{d}}$ and behaves like GD. 

Suppose a nonzero mode approaches its target from below and its residual decreases from $r_0$ to $\delta$, where
$
0<\delta<r_0<\mu.
$
Then
\[
T_{\mathrm{IAdam}}^{\mathrm{sat}}(r_0,\delta)
=
(\mu-\delta)^{1/L}
-
(\mu-r_0)^{1/L}
+
\frac{\epsilon}{L}
\int_{\mu-r_0}^{\mu-\delta}
\frac{x^{-p}}{\mu-x}
\,dx.
\]
An analogous expression holds when the mode approaches the target from above.
Near saturation, it behaves like GD with rate $L\mu^p/\epsilon$.
\end{proposition}

\begin{proof} \,

\paragraph{Birth of a positive mode.}
From the dynamics, $\dot\sigma=-\dfrac{L\sigma^p(\sigma-\mu)}{|\sigma-\mu|\sigma^q+\epsilon}$, the birth time for $0<s<b<\mu$ is
\begin{align*}
T_{\mathrm{IAdam}}^{\mathrm{birth}}(s,b) &= \frac{1}{L} \int_s^b x^{-q} \, dx + \frac{\epsilon}{L} \int_s^b \frac{x^{-p}}{\mu-x} \, dx \\
&= T_{\mathrm{SignGD}}^{\mathrm{birth}}(s,b) + \epsilon T_{\mathrm{GD}}^{\mathrm{birth}}(s,b)
\end{align*}
Therefore, birth time is the sum of a finite SignGD contribution and an $\epsilon$-dependent Euclidean contribution.
For fixed $\epsilon>0$, $b$, and $\mu$, as $s\to0^+$,
\[
T_{\mathrm{IAdam}}^{\mathrm{birth}}(s,b)
=
\begin{cases}
\Theta(1),
& L=1,
\\[1mm]
\dfrac{\epsilon}{2\mu}
\log\dfrac1s
+
O(1),
& L=2,
\\[3mm]
\dfrac{\epsilon}{\mu(L-2)}
s^{-(L-2)/L}
+
O(1),
& L>2.
\end{cases}
\]
Thus a fixed positive $\epsilon$ restores the Euclidean GD birth classification in the extreme small-initialization regime:
birth remains finite for $L=1$, is logarithmically delayed for $L=2$, and is polynomially delayed for $L>2$.

For $L>1$, for $0<\sigma<\mu$ define
$$
h_\mu(\sigma)=(\mu-\sigma)\sigma^q,
\qquad
h_{\max}
=
\mu^{q+1}\frac{q^q}{(q+1)^{q+1}}.
$$
If $0<\epsilon<h_{\max}$, the equation $h_\mu(\sigma)=\epsilon$ has two roots
$$
0<\sigma_-<\frac{q\mu}{q+1}<\sigma_+<\mu.
$$
Then $h_\mu(\sigma)>\epsilon$ exactly on $(\sigma_-,\sigma_+)$.
An accurate approximation with SignGD requires the stronger condition $h_\mu(\sigma)\gg\epsilon$.

If $\epsilon\ge h_{\max}$, no interval satisfies $h_\mu(\sigma)>\epsilon$.
If $\epsilon\gg h_{\max}$, the entire birth trajectory is uniformly approximated by GD divided by $\epsilon$.

As $\epsilon/\mu^{q+1}\to0$, \,\,
$
\sigma_-=(\epsilon/\mu)^{1/q}(1+o(1)),
\,\,\,
\mu-\sigma_+=({\epsilon}/{\mu^q})(1+o(1)).
$

For $L=1$, $h_\mu(\sigma)=\mu-\sigma$, so the two root description does not apply.

The early-growth crossover is approximately
\[
\sigma_{\mathrm{A}}^\mathrm{b}
\asymp
\left(
\frac{\epsilon}{\mu}
\right)^{1/q},
\]
while the crossover near saturation satisfies
$
\mu-\sigma^{\mathrm{s}}_\mathrm{A}
\asymp
\dfrac{\epsilon}{\mu^q}.
$ \\
Modes far below $\sigma^{\mathrm{b}}_\mathrm{A}$ initially follow the Euclidean birth law.
Modes between the two crossovers move approximately according to SignGD.
Modes sufficiently close to the target return to an $\epsilon$-dominated regime.

\paragraph{Decay of an unsupported mode.}

For $\mu=0$, the mode dynamics equation becomes
$
\dot\sigma
=
-L
\dfrac{
\sigma^{p+1}
}{
\sigma^{q+1}+\epsilon
}.
$
When
$
\sigma^{q+1}\gg\epsilon,
$
the mode follows the SignGD-like law
$
\dot\sigma
\approx
-L\sigma^q.
$
When
$
\sigma^{q+1}\ll\epsilon,
$
the final decay satisfies
$
\dot\sigma
\approx
-\dfrac{L}{\epsilon}\sigma^{p+1}.
$
The crossover scale is therefore
\[
\sigma^{\mathrm{d}}_\mathrm{A}
\asymp
\epsilon^{1/(q+1)}.
\]

For $\mu=0$, similar to birth time, 
$$
T_{\mathrm{IAdam}}^{\mathrm{decay}}(s,\delta) = T_{\mathrm{SignGD}}^{\mathrm{decay}}(s,\delta) + \epsilon T_{\mathrm{GD}}^{\mathrm{decay}}(s,\delta).
$$
Therefore, 
\[
T_{\mathrm{IAdam}}^{\mathrm{decay}}(s,\delta)
=
\begin{cases}
s-\delta
+
\epsilon\log\dfrac{s}{\delta},
& L=1,
\\[3mm]
s^{1/L}-\delta^{1/L}
+
\dfrac{\epsilon}{Lp}
\left(
\delta^{-p}-s^{-p}
\right),
& L>1.
\end{cases}
\]
For fixed $\epsilon>0$ and fixed initial value $s$, as $\delta\to0^+$,
\[
T_{\mathrm{IAdam}}^{\mathrm{decay}}(s,\delta)
=
\begin{cases}
\epsilon\log\dfrac1\delta+O(1),
& L=1,
\\[3mm]
\dfrac{\epsilon}{Lp}\delta^{-p}+O(1),
& L>1.
\end{cases}
\]
Thus instantaneous Adam does not remove an unsupported mode in finite time when $\epsilon>0$.
For $L=1$, the final trajectory decays exponentially:
$
\sigma(t)
\asymp
e^{-t/\epsilon}.
$
For $L>1$, it has the algebraic tail
$
\sigma(t)
\asymp
\left(
\frac{Lp}{\epsilon}t
\right)^{-1/p}.
$
The SignGD finite-time decay is therefore replaced by a logarithmic tolerance time at depth one and a polynomial tolerance time at greater depth.

\paragraph{Saturation at a nonzero target.}
Following the above analysis, 
for approach to a positive target from either side,
$$
T_{\mathrm{IAdam},\pm}^{\mathrm{sat}}(r_0,\delta)
=
T_{\mathrm{Sign},\pm}^{\mathrm{sat}}(r_0,\delta)
+
\epsilon T_{\mathrm{GD},\pm}^{\mathrm{sat}}(r_0,\delta).
$$
These identities are exact.
For fixed $r_0$,
$$
T_{\mathrm{IAdam},\pm}^{\mathrm{sat}}(r_0,\delta)
=
\frac{\epsilon}{L\mu^p}\log\frac{r_0}{\delta}+O(1).
$$

Thus the outer approach is approximately linear, but the final approach is exponential with rate
$
\beta
=
\frac{L\mu^p}{\epsilon}.
$
Unlike ideal SignGD, instantaneous Adam reaches a nonzero target only asymptotically when $\epsilon>0$.

\end{proof}

\paragraph{Depth and the order of limits.}

The two exponents $q$ and $p$ play different roles.
The exponent $q$ controls the SignGD-like motion in the factor-amplitude coordinate and the locations of the crossover scales.
The exponent $p$ controls the $\epsilon$-dominated birth barrier, the final decay tail, and the local exponential saturation rate.

For fixed $\epsilon>0$, increasing depth changes the birth asymptotics from logarithmic at $L=2$ to polynomial at $L>2$, with exponent
$
p-1
=
\frac{L-2}{L}.
$
The unsupported-mode decay tail has exponent
$
p
=
2-\frac2L,
$
which also increases with depth.
Thus deeper factorizations make the extreme small-gradient regimes increasingly resistant both to the emergence of infinitesimal positive modes and to the complete removal of already active modes.

For $L \ge 2$, the limits $\epsilon\to0$ and $s\to0$ do not commute.
For every fixed $s>0$,
\[
\lim_{\epsilon\to0}
T_{\mathrm{IAdam}}^{\mathrm{birth}}(s,b)
=
T_{\mathrm{SignGD}}^{\mathrm{birth}}(s,b).
\]
In contrast, for every fixed $\epsilon>0$, the limit $s\to0$ reveals the Euclidean GD birth barrier.
Therefore an Adam-like update may appear SignGD-like at practical mode scales while remaining mathematically GD-like in the infinitesimal limit.
For $L=1$, both iterated limits equal $b$.

\textbf{Takeaway.}
\emph{Ideal SignGD makes $\sigma^{1/L}$ a constant-speed coordinate: positive modes are born, unsupported modes disappear, and nonzero targets are reached in finite time for every depth.
The instantaneous Adam approximation inherits this behavior only over the range in which the gradient magnitude dominates $\epsilon$.
For fixed $\epsilon>0$, sufficiently small birth modes recover GD's finite--logarithmic--polynomial depth transition, unsupported modes develop logarithmic or polynomial decay tails, and final saturation becomes exponential.
Thus coordinate normalization controls the phase crossovers. 
}

\subsection{Muon Phase Laws}
{Muon} updates each factor by orthogonalizing matrix updates, that is $Z_\ell = U_\ell V_\ell^\top$.  
In the diagonal aligned singular mode model, $G = \diag(g_i)$. Therefore, $z_i = \sign(g_i)$. Thus, for each mode
\begin{align*}
    \dot\sigma&=-L\sigma^q\sign(g) \\
    &= -L\sigma^q\sign((\sigma-\mu)\sigma^q) \\
    &= -L\sigma^q\sign(\sigma-\mu)
\end{align*}
reducing it to a scalar sign rule as in SignGD updates. Therefore, \Cref{prop:signgd} holds for Muon as well.

Note that this is a spectral update in matrix space even though the aligned model reduces it to a scalar sign rule as in SignGD updates, and
this does not mean Muon has the same global bias as coordinate SignGD. Coordinate SignGD is coordinate-basis dependent. Muon is matrix-basis / spectral-geometry dependent. But in an aligned diagonal model, both reduce to sign dynamics on singular modes. Therefore, Muon weakens low rank bias in the diagonal mode. It controls spectral geometry, not rank directly.
A spectral-norm ball constrains the largest singular value, but it does not penalize the number of nonzero singular values. So Muon should not be expected to behave like nuclear-norm minimization.

\subsection{Cumulative Shampoo Phase Laws}
{Shampoo} updates each factor using $Z_\ell=C_\ell^{-1/4}H_\ell D_\ell^{-1/4}$ where $C_\ell$ and $D_\ell$ are the accumulated left and right gradient-covariance matrices defined in \Cref{tab:optimizer_conventions}.  In the balanced aligned model, the two covariance factors reduce to the same scalar accumulator.
In the aligned diagonal model, $H=\diag(g_i)$, therefore $HH^\top = H^\top H = \diag(g_i^2)$.
For the common scalar accumulator, let 
$$h(t) = \int_0^t g(u)^2 \, du, \qquad \text{and}  \qquad d(t)=\sqrt{\lambda_{\mathrm{Sh}} + h(t)}, \,\, d(0)=\sqrt{\lambda_{\mathrm{Sh}}}. $$
For a fresh isotropic Shampoo state,
$
d_0=d(0)=\sqrt{\lambda_{\mathrm{Sh}}},
$
although the analysis below also permits $d_0$ to include gradient history accumulated before the phase under consideration.
The scalar factor update $C_\ell^{-1/4}H_\ell D_\ell^{-1/4} = \diag(g_i/d)$. Thus 
\[
\dot a
=
-\frac{g}{d}.
\]
Since
$
\dot\sigma
=
L\sigma^q\dot a,
$
each singular mode satisfies
\[
\dot\sigma
=
-L
\frac{
(\sigma-\mu)\sigma^p
}{
d
}
=
L
\frac{
(\mu-\sigma)\sigma^p
}{
d
}.
\]
The accumulator obeys
\[
\dot h
=
g^2
=
(\sigma-\mu)^2\sigma^p.
\]
Therefore
\[
\dot d
=
\frac{\dot h}{2d}
=
\frac{
(\sigma-\mu)^2\sigma^p
}{
2d
}.
\]

\paragraph{An exact first integral.}

Along any nonstationary monotone segment, divide $\dot d$ by $\dot\sigma$:
\[
\frac{dd}{d\sigma}
=
\frac{
(\sigma-\mu)^2\sigma^p/(2d)
}{
-L(\sigma-\mu)\sigma^p/d
}
=
-\frac{\sigma-\mu}{2L}.
\]
Integrating from an initial effective singular value $\sigma(0)=s$ gives
\begin{align}
d(\sigma)
=
d_0
+
\frac{
\mu(\sigma-s)
-
\frac12(\sigma^2-s^2)
}{
2L}.
\label{eq:sh_d_sigma}
\end{align}
This relation exactly eliminates the accumulator from the aligned scalar dynamics.
In particular, cumulative Shampoo is not simply an instantaneous approximation to SignGD or Muon: its current mobility is determined by the distance already travelled by the mode.

\begin{proposition}[Cumulative Shampoo phase laws]
\label{prop:shampoo-phases}

Let a nonstationary aligned mode begin at $\sigma(0)=s>0$ with denominator $d_0>0$. Consider $d(\sigma)$ in \Cref{eq:sh_d_sigma}.
For $0<s<b<\mu$, the birth time is exactly
\[
T_{\mathrm{Sh}}^{\mathrm{birth}}(s,b)
=
\frac1L
\int_s^b
\frac{
d(x)
}{
x^p(\mu-x)
}
\,dx.
\]
For an unsupported mode with $\mu=0$, define $c_s = d_0+\dfrac{s^2}{4L}.$
The decay time from $s$ to $\delta\in(0,s)$ is
\[
T_{\mathrm{Sh}}^{\mathrm{decay}}(s,\delta)
=
\begin{cases}
c_s\log\dfrac{s}{\delta}
-
\dfrac{s^2-\delta^2}{8} & L=1,
\\[3mm]
\dfrac{c_s}{Lp}
\left(
\delta^{-p}-s^{-p}
\right)
-
\dfrac{
s^{2-p}-\delta^{2-p}
}{
4L^2(2-p)
} & L>1.
\end{cases}
\]

For fixed $d_0>0$, the birth asymptotics from a vanishing positive seed have exponent $\alpha_\mathrm{b}=p$ and coefficient $c_\mathrm{b}=L\mu/d_0$.
On an intermediate interval $\sigma_a\leq\sigma\leq b$ satisfying $s\ll\sigma_a$, $b\ll\mu$, and $2Ld_0\ll\mu\sigma_a$, the leading birth law instead has exponent $\alpha_\mathrm{b}=p-1$ and coefficient $c_\mathrm{b}=2L^2$.
The uniform approximation and its error bound are stated below.
For an unsupported mode with fixed $s>0$ and $d_0>0$, the final decay as $\sigma\to0$ has exponent $\alpha_\mathrm{d}=p+1$ and coefficient $c_\mathrm{d}=L/(d_0+s^2/(4L))$.


For a nonzero target approached monotonically from above and below, let the residual move from $r_0$ to $\delta$. 
Then local saturation is exponential at rate $\beta = \dfrac{L\mu^p}{d_\star}$ with
\[
T_{\mathrm{Sh},\pm}^{\mathrm{sat}}(r_0,\delta)
=
\frac{
d_\star
}{
L\mu^p
}
\log\frac{r_0}{\delta} + O(1), \qquad
\text{where} \qquad
d_\star
=
d(\mu)
=
d_0+\frac{(\mu-s)^2}{4L}.
\]

\end{proposition}

\begin{proof} \,
\paragraph{Birth of a positive mode.}

The exact first integral $d(\sigma)$ reveals two distinct birth regimes.
For fixed positive damping, every mode initially experiences a GD-like phase.
After sufficient gradient history has accumulated, the dynamics can enter a history-adapted phase with a different singular-value exponent.

\paragraph{Initial damping dominated phase.}
Let's define 
$$
\nu_s(x)=\frac{(x-s)(2\mu-x-s)}{4L}.
$$
The exact birth time decomposition is
$$
T_{\mathrm{Sh}}^{\mathrm{birth}}(s,b)
=
d_0T_{\mathrm{GD}}^{\mathrm{birth}}(s,b)+R_s(b),
\qquad
\text{where} \qquad
R_s(b)
=
\frac1L\int_s^b
\frac{\nu_s(x)}{x^p(\mu-x)}\,dx.
$$
Since $0\le\nu_s(x)\le\mu x/(2L)$ and $p<2$,
$
0\le R_s(b)
\le
\dfrac{\mu}{2L^2}
\displaystyle\int_0^b\dfrac{x^{1-p}}{\mu-x}\,dx
<\infty.
$
Thus, for fixed $d_0>0$ and $b<\mu$,
$$
T_{\mathrm{Sh}}^{\mathrm{birth}}(s,b)
=
\begin{cases}
T_{\mathrm{Sh}}^{\mathrm{birth}}(0,b)+o(1), & L=1,\\[3mm]
\dfrac{d_0}{2\mu}\log(1/s)+O(1), & L=2,\\[3mm]
\dfrac{d_0s^{1-p}}{L\mu(p-1)}+O(1), & L>2.
\end{cases}
$$
The depth one limit is
$$
T_{\mathrm{Sh}}^{\mathrm{birth}}(0,b)
=
\left(d_0+\frac{\mu^2}{4}\right)
\log\frac{\mu}{\mu-b}
-\frac{\mu b}{4}+\frac{b^2}{8}.
$$

Thus the initial Shampoo birth exponent $\alpha_\mathrm{b} = p$, which is exactly the Euclidean GD exponent and its coefficient $c_\mathrm{b}= L\mu/d_0$.
Therefore, cumulative Shampoo with fixed positive damping retains the Euclidean GD birth classification:
the birth time is finite at depth one, logarithmic at depth two, and polynomial for $L>2$.
This conclusion is important because the later adapted phase does not remove the infinitesimal barrier created by the initial denominator.
The adapted phase changes how a mode grows after sufficient history has accumulated, but a mode starting arbitrarily close to zero must first traverse the damping-dominated GD-like regime.

\paragraph{Crossover to the history adapted phase.}

The accumulator becomes comparable to its initial value when
\[
d_0
\asymp
\frac{
\mu(\sigma-s)-\frac12(\sigma^2-s^2)
}{
2L}.
\]

An interior point satisfying $\nu_s(\sigma_c)=d_0$ exists if and only if
$
4Ld_0<(\mu-s)^2.
$
It is
$
\sigma_{\mathrm{Sh}}^\mathrm{b}
=
\mu-\sqrt{(\mu-s)^2-4Ld_0}.
$
If $4Ld_0/(\mu-s)^2\to0$, then
$$
\sigma_{\mathrm{Sh}}^\mathrm{b}-s
=
\frac{2Ld_0}{\mu-s}
\left(1+O\left(\frac{Ld_0}{(\mu-s)^2}\right)\right).
$$
The approximation $\sigma_{\mathrm{Sh}}^\mathrm{b}\approx s+2Ld_0/\mu$ additionally requires $s\ll\mu$.

A substantial history adapted birth interval exists only when this crossover occurs well before the target.
Equivalently, the total history accumulated by the target must be large relative to the initial denominator, that is, 
$
d_0
\ll
\dfrac{(\mu-s)^2}{4L}.
$
If this condition fails, the mode remains predominantly in the GD-like regime throughout its approach to $\mu$.

At fixed $d_0$ and $\mu$, the crossover scale increases with depth.
Thus deeper factorizations require a larger effective singular value before accumulated Shampoo history dominates the initial damping.

\paragraph{History-adapted accumulator phase.}
Consider an interval $\sigma_a\le\sigma\le b<\mu$ with $s<\sigma_a$.
Define
$$
\mathcal E
=
\frac{s}{\sigma_a}
+
\frac{b}{\mu}
+
\frac{2Ld_0}{\mu\sigma_a}.
$$
For fixed $L$ and $\mathcal E\to0$, uniformly on this interval,
$
d(\sigma)
=
\dfrac{\mu\sigma}{2L}(1+O(\mathcal E)),
$ that is accumulated history is dominating.
Then,
$$
\dot\sigma
=
2L^2\sigma^{p-1}(1+O(\mathcal E)).
$$

Hence the adapted Shampoo birth exponent in this phase $\alpha_\mathrm{b} = p-1 = 1 -\dfrac{2}{L}$. 

If the adapted phase begins at $\sigma_a$ and ends at $b\ll\mu$, its duration is
\[
T_{\mathrm{Sh}}^{\mathrm{birth}}(\sigma_a,b)
=
\frac{1+O(\mathcal E)}{2L^2}
\int_{\sigma_a}^{b}
x^{-(p-1)}
\,dx
=
\frac{
b^{2/L}-\sigma_a^{2/L}
}{
4L} \lp 1+O(\mathcal E) \rp
\]
from \Cref{prop:power_passage,cor:birth_classification}. This intermediate approximation cannot be extended to $\sigma_a=0$ at fixed $d_0>0$.
The target amplitude $\mu$ has disappeared from this law.
Thus, after enough gradient history has accumulated, Shampoo no longer orders growing modes by their target amplitudes.

The initial and adapted Shampoo exponents lie on opposite sides of Muon's exponent $q = 1-1/L$.
For $L>1$, $p>q$, so initial Shampoo suppresses small modes more strongly than Muon.
In contrast, $p-1 = q - 1/L < q$,
so the adapted phase has weaker singular magnitude dependence than Muon.
Cumulative Shampoo therefore changes from a GD-like, target sensitive birth law to a history-adapted, target independent catch-up law.
This catch-up law is reached only after the accumulator crosses the damping controlled threshold.

\paragraph{Decay of an unsupported mode.}

For $\mu=0$, the first integral reduces to
\[
d(\sigma)
=
d_0+\frac{s^2-\sigma^2}{4L}
=
c_s-\frac{\sigma^2}{4L}, \qquad
\text{where} \qquad
c_s
=
d_0+\frac{s^2}{4L}.
\]
The decay dynamics are
$
\dot\sigma
=
-L
\dfrac{
\sigma^{p+1}
}{
d(\sigma)
}.
$
Therefore, 
\begin{align*}
T_\mathrm{Sh}^{\mathrm{decay}}(s,\delta) &= - \dfrac{c_s}{L} \int_s^\delta x^{-(p+1)} \, dx + \dfrac{1}{4L^2} \int_s^\delta x^{-(p-1)} \, dx \\
&= c_s T_\mathrm{GD}^{\mathrm{decay}}(s,\delta) + \dfrac{1}{4L^2} \int_s^\delta x^{-(p-1)} \, dx
\end{align*}
Note that $p-1=1-2/L$ is less than $1$, thus the second term results in finite limit as $\delta \to 0$ from \Cref{prop:power_passage}. So, the decay is determined by the first term, and so the decay exponent is $\alpha_\mathrm{d}=p+1$, which is the same as Euclidean GD, and $c_\mathrm{d}=L/c_s$. The full solution of the decay time is derived from \Cref{cor:decay_classification,prop:power_passage}:
\[
T_{\mathrm{Sh}}^{\mathrm{decay}}(s,\delta)
=
\begin{cases}
c_s\log\dfrac{s}{\delta}
-
\dfrac{s^2-\delta^2}{8} & L=1,
\\[3mm]
\dfrac{c_s}{Lp}
\left(
\delta^{-p}-s^{-p}
\right)
-
\dfrac{
s^{2-p}-\delta^{2-p}
}{
4L^2(2-p)
} & L>1.
\end{cases}
\]
The mode consequently decays exponentially for $L=1$:
$
\sigma(t)
\asymp
e^{-t/c_s},
$
and algebraic for $L>1$:
$
\sigma(t)
\sim
\left(
\frac{Lp}{c_s}t
\right)^{-1/p}.
$

Thus cumulative Shampoo never removes an unsupported mode in finite time.
Its decay has the same logarithmic-versus-polynomial classification as Euclidean GD, but its coefficient retains the mode's accumulated gradient history.
This produces a historical inertia effect.
At fixed $s$, $\delta$, and $L$,
$$
\frac{\partial}{\partial d_0}
T_{\mathrm{Sh}}^{\mathrm{decay}}(s,\delta)
=
T_{\mathrm{GD}}^{\mathrm{decay}}(s,\delta)>0.
$$
Thus a larger inherited denominator increases every decay passage time when the initial amplitude is held fixed.
Therefore, a mode that has previously received larger gradients begins the decay phase with a larger denominator $d_0$, and therefore a larger limiting constant $c_s$.
Such a mode takes proportionally longer to remove.
Muon has no corresponding cumulative denominator 
and can drive an unsupported aligned mode to zero in finite time.
Cumulative Shampoo instead makes previously active modes increasingly persistent.

Depth strengthens the late decay tail.
For $L>1$, $p=2-\dfrac2L$ increases toward two as $L$ increases.
Accordingly, the tolerance time diverges as $\delta^{-p},$
and the trajectory decays as $t^{-1/p},$
whose exponent decreases toward $1/2$ with depth.

\paragraph{Saturation at a nonzero target.}

At the target $\sigma=\mu$, 
$$
d_\star = 
d(\mu)
=
d_0+\dfrac{(\mu-s)^2}{4L}.
$$
Near the target,
for a residual interval traversed by the trajectory,
$$
T_{\mathrm{Sh},-}^{\mathrm{sat}}(r_0,\delta)
=
\frac1L\int_\delta^{r_0}
\frac{d_\star-r^2/(4L)}{(\mu-r)^p r}\,dr,
\qquad
\text{and}
\qquad
T_{\mathrm{Sh},+}^{\mathrm{sat}}(r_0,\delta)
=
\frac1L\int_\delta^{r_0}
\frac{d_\star-r^2/(4L)}{(\mu+r)^p r}\,dr.
$$
Consequently,
$$
T_{\mathrm{Sh},\pm}^{\mathrm{sat}}(r_0,\delta)
=
\frac{d_\star}{L\mu^p}\log\frac{r_0}{\delta}+O(1).
$$

From \Cref{prop:saturation_phase_law}, Cumulative Shampoo consequently saturates monotonically and exponentially $\gamma=1$ with local rate $\beta = \dfrac{
L\mu^p}{d_\star}$.
Thus Shampoo approaches a nonzero target only asymptotically.
Unlike Muon or SignGD in their ideal memoryless forms, it does not reach the target in finite time.

The dependence of the saturation rate on the target amplitude changes with accumulated history.
If the initial denominator dominates,
$
d_0
\gg
\dfrac{(\mu-s)^2}{4L},
$
then
$
\beta
\approx
\dfrac{
L\mu^p
}{
d_0
}.
$
This has the same target ordering as GD: larger target singular values saturate faster.

If accumulated history dominates,
$
d_0
\ll
\dfrac{(\mu-s)^2}{4L},
$
then
$
\beta
\approx
\dfrac{
4L^2\mu^p
}{
(\mu-s)^2
}.
$
When $s\ll\mu$, this simplifies to
\[
\beta
\approx
4L^2\mu^{p-2}
=
4L^2\mu^{-2/L}.
\]
For modes compared at common $s$ and $d_0$ with $\mu>s$,
$$
\frac{d}{d\mu}\log\beta
=
\frac{p}{\mu}
-
\frac{2(\mu-s)}{4Ld_0+(\mu-s)^2}.
$$
Thus the rate now decreases with the target singular value $\mu$ precisely when
$
4Lp\,d_0
<
(\mu-s)\bigl((2-p)\mu+ps\bigr)
$.
Thus sufficiently adapted Shampoo can reverse the GD saturation ordering:
larger modes, which generated larger cumulative preconditioners, can converge more slowly near their targets.

As depth increases, the exponent $-\frac2L$ approaches zero.
Hence the history dominated saturation rate becomes less sensitive to the target amplitude, even though its overall depth-dependent prefactor remains.
\end{proof}

\paragraph{Depth dependent phase structure.}

Cumulative Shampoo has no single birth exponent valid throughout training.
At fixed positive damping, an infinitesimal mode first follows
$
\dot\sigma
\propto
\mu\sigma^p,
$
with the GD exponent
$
p=2-2/L.
$
After the cumulative state dominates, it follows
$
\dot\sigma
\propto
\sigma^{p-1},
$
with adapted exponent
$
p-1=1-2/L.
$

For $L=1$, the adapted exponent is $-1$.
For $L=2$, it is zero and the adapted singular value grows approximately at constant speed.
For $L>2$, it lies strictly between zero and one, so the adapted phase has a finite hitting time from a vanishing positive mode.
Nevertheless, for any fixed $d_0>0$, the complete birth time as $s\to0^+$ remains governed by the preceding GD-like phase.

The decay exponent remains $p+1$ in the mode equation and produces a polynomial tolerance time of order $\delta^{-p} $ for $L>1$.
The saturation rate is
$
\beta=\dfrac{L\mu^p}{d_\star}
$, which is approximately $4L^2 \mu^{-2/L}$ at saturation. Thus, it can make larger singular modes saturate more slowly, reversing the order during saturation. 
Saturation depends jointly on depth, target amplitude, damping, and the mode's accumulated history.


\noindent
\textbf{Takeaway.}
\emph{Cumulative Shampoo is a two-stage spectral optimizer rather than a fixed interpolation between GD and Muon.
A newly emerging mode first obeys the GD birth law and therefore retains the finite--logarithmic--polynomial depth barrier under fixed damping.
After sufficient gradient history has accumulated, the mode enters a target-independent catch-up phase with exponent $p-1$, weaker singular-magnitude dependence than Muon, and finite adapted-phase hitting time.
The same accumulated state acts as inertia during decay: unsupported modes vanish only asymptotically, and previously active modes become harder to remove.
Near a nonzero target, Shampoo saturates monotonically and exponentially, but strong accumulation can reverse GD's mode ordering by making larger target modes converge more slowly.}

\subsection{K-FAC Phase Laws}
{K-FAC} updates each layer with the per-layer Fisher or Gauss Newton curvature block. As detailed in \Cref{app:optimizer} from \Cref{eq:gn_kfac}, $\mathcal{F}_\ell = (A_\ell A_\ell^\top) \otimes (B_\ell^\top B_\ell)$ where $A_\ell = W_{\ell-1} \cdots W_1$ and $B_\ell = W_L \cdots W_{\ell+1}$.
Therefore, from \Cref{app:balance_consequence} based on the balanced aligned model, the corresponding eigenvalue of the layerwise Gauss Newton block is 
\[
\begin{aligned}
\kappa_\ell =
\left(a^{\ell-1}\right)^2
\left(a^{L-\ell}\right)^2
=
a^{2L-2} = \sigma^p.
\end{aligned}
\]
The curvature eigenvalue is independent of the layer index $\ell$.
Consequently, every layer receives the same scalar update, and balancedness is preserved within this aligned idealization.

The damped layer update is
\[
\dot a_\ell
=
-
\frac{
g_\ell
}{
\kappa_\ell+\lambda_{\mathrm{KF}}
} 
= -\frac{
(\sigma-\mu)\sigma^q
}{
\sigma^p+\lambda_{\mathrm{KF}}
}.
\]
Recall that $\dot\sigma = L\sigma^q\dot a.$ 
Hence, the singular value dynamics in the aligned model for each mode is $$\dot\sigma = \dfrac{L\sigma^p (\mu-\sigma)}{\sigma^p + \lambda_{\mathrm{KF}}}.$$

\paragraph{Relation to exact joint Gauss Newton.}

The preceding derivation uses a block diagonal curvature model that inverts each layer block separately.
It should be distinguished from exact Gauss Newton over all layers jointly.
For one balanced scalar mode, the full product Jacobian with respect to the $L$ layer amplitudes is
\[
J
=
a^{L-1}
\begin{bmatrix}
1 & \cdots & 1
\end{bmatrix}.
\]
Hence the predictor space curvature is
$
JJ^\top
=
L a^{2L-2}
=
L\sigma^p.
$
Exact damped Gauss Newton therefore gives
\[
\dot\sigma_{\mathrm{GN}}
=
\eta
\frac{
L\sigma^p
}{
L\sigma^p+\lambda_{\mathrm{GN}}
}
(\mu-\sigma).
\]
Comparing with the K-FAC dynamics derived above, we observe that
when damping is zero, both cancel the factorization-induced power $\sigma^p$.
They differ only by a common factor $L$. 
With positive damping, the two filters are not identical unless their damping parameters are scaled consistently.

\paragraph{K-FAC phase laws.}
From the dynamics $\dot\sigma = \dfrac{L\sigma^p (\mu-\sigma)}{\sigma^p + \lambda_{\mathrm{KF}}}$, 
the factor
$
\dfrac{
\sigma^p
}{
\sigma^p+\lambda_{\mathrm{KF}}
}
$
determines how much of the depth-dependent factorization mobility is cancelled.

When $\lambda_{\mathrm{KF}}=0$ and $\sigma>0$, the mobility factor cancels:
$
\dot{\sigma}
=
L(\mu-\sigma).
$
Thus undamped K-FAC behaves like direct gradient flow on the effective singular value, up to the common factor $L$.
When $\lambda_{\mathrm{KF}}>0$, the cancellation is incomplete for sufficiently small modes.

\begin{proposition}[K-FAC phase laws]
\label{prop:kfac-phases}

Let $0<s<b<\mu$.
The time required for a positive mode initialized at $\sigma(0)=s$ to reach $b$ is
\[
T_{\mathrm{KFAC}}^{\mathrm{birth}}(s,b)
=
\frac1L
\log
\frac{\mu-s}{\mu-b}
+
\lambda_{\mathrm{KF}} T_{\mathrm{GD}}^{\mathrm{birth}}(s,b)
\]
For an unsupported mode with $\mu=0$, the time required to decay from $s$ to
$\delta\in(0,s)$ is
\[
T_{\mathrm{KFAC}}^{\mathrm{decay}}(s,\delta)
=
\begin{cases}
    (1+\lambda_{\mathrm{KF}})
\log\dfrac{s}{\delta},
&\qquad
L=1, \\[3mm]
\dfrac1L
\log\dfrac{s}{\delta}
+
\dfrac{\lambda_{\mathrm{KF}}}{Lp}
\left(
\delta^{-p}-s^{-p}
\right),
&\qquad
L>1.
\end{cases}
\]
For a nonzero target, the time required to reduce the residual from a fixed $r_0$ to $\delta$ satisfies from either side satisfies,
$$
T_{\mathrm{KF},\pm}^{\mathrm{sat}}(r_0,\delta)
=
\frac1L\log\frac{r_0}{\delta}
+
\lambda_{\mathrm{KF}}
T_{\mathrm{GD},\pm}^{\mathrm{sat}}(r_0,\delta).
$$
Consequently, 
$$
T_{\mathrm{KF},\pm}^{\mathrm{sat}}(r_0,\delta)
=
\frac{\mu^p+\lambda_{\mathrm{KF}}}{L\mu^p}
\log\frac{r_0}{\delta}+O(1).
$$
\end{proposition}

\begin{proof}\,
\paragraph{Birth of a positive mode.}

From the dynamics, we get
\begin{align*}
T_{\mathrm{KFAC}}^{\mathrm{birth}}(s,b) &=  \frac{1}{L} \int_s^b (\mu - x)^{-1} \, dx + \frac{\lambda_{\mathrm{KF}}}{L} \int_s^b \frac{\sigma^{-p}}{(\mu - x)} \, dx \\
&= \frac1L
\log
\frac{\mu-s}{\mu-b}
+
\lambda_{\mathrm{KF}} T_{\mathrm{GD}}^{\mathrm{birth}}(s,b).
\end{align*}
Thus, the K-FAC birth time decomposes into two terms.

\paragraph{Undamped K-FAC.}

When $\lambda_{\mathrm{KF}}=0,$
the mode dynamics reduce to $\dot{\sigma}=L(\mu-\sigma).$
The birth time is
$
T_{\mathrm{KFAC}}^{\mathrm{birth}}(s,b)
=
\dfrac1L
\log
\dfrac{\mu-s}{\mu-b}
$, thus finite time emergence.
The exact trajectory is
$
\sigma(t)
=
\mu-(\mu-s)e^{-Lt}.
$

The near-zero birth dynamics are
$
\dot{\sigma}
\approx
L\mu.
$
Hence the undamped K-FAC birth exponent $\alpha_\mathrm{b}$ is zero.
The initial velocity does not vanish with $\sigma$.
Undamped K-FAC therefore removes the depth-dependent small-mode barrier.

For a common absolute threshold $b$ and seed $s$, ideal Muon has target independent birth time whenever $\mu>b$.
Undamped curvature flow instead has birth time
$$
\frac1L\log\frac{\mu-s}{\mu-b},
$$
which decreases with $\mu$.

The target amplitude nevertheless remains in the velocity.
This distinguishes undamped K-FAC from ideal Muon.
Muon removes both the factorization barrier and target-amplitude ordering during aligned birth, whereas undamped K-FAC removes the barrier but preserves target ordering.

The cancellation at exactly $\sigma=0$ requires care.
At an exactly zero mode, the undamped Fisher block is singular, so the update depends on the inverse or pseudoinverse convention.
The finite birth statement above concerns the limit from a positive seed $s\to0^+$.

\paragraph{Damped K-FAC.}

For $\lambda_{\mathrm{KF}}>0$, define the spectral curvature threshold
$
\sigma_{\mathrm{KF}}^\mathrm{b}
=
\lambda_{\mathrm{KF}}^{1/p},
\,\,
L>1.
$
This is the scale at which
$
\sigma^p
\asymp
\lambda_{\mathrm{KF}}.
$
When $\sigma\ll\sigma_{\mathrm{KF}}^\mathrm{b},$ the preconditioning factor is approximately $\dfrac{\sigma^p}{\lambda_{\mathrm{KF}}}.$
For a small positive mode with $\sigma\ll\mu$,
$
\dot{\sigma}
\approx
\dfrac{L\mu}{\lambda_{\mathrm{KF}}}
\sigma^p.
$
Thus damped K-FAC recovers the Euclidean GD birth exponent $\alpha_\mathrm{b} = p$.

When $\sigma\gg\sigma_{\mathrm{KF}}^\mathrm{b},$ the preconditioning factor approaches one.
The mode then follows
$
\dot{\sigma}
\approx
L(\mu-\sigma),
$
and behaves like direct optimization of the effective matrix.

Damped K-FAC therefore has a two-regime birth law: GD-like below $\sigma_{\mathrm{KF}}^\mathrm{b},$ and curvature-corrected above $\sigma_{\mathrm{KF}}^\mathrm{b}$.
%
%
Thus any fixed positive damping restores the Euclidean finite--logarithmic--polynomial birth classification in the infinitesimal limit.

For $L \ge 2$, the limits $\lambda_{\mathrm{KF}}\to0$ and $s\to0$ do not commute.
For every fixed $s>0$,
\[
\lim_{\lambda_{\mathrm{KF}}\to0}
T_{\mathrm{KFAC}}^{\mathrm{birth}}(s,b)
=
\frac1L
\log
\frac{\mu-s}{\mu-b}.
\]
In contrast, for every fixed $\lambda_{\mathrm{KF}}>0$, the limit $s\to0^+$
is logarithmically divergent at $L=2$ and polynomially divergent for $L>2$.
A numerically small damping can therefore be asymptotically decisive for sufficiently small spectral modes.

For $L=1$,
$
\dot\sigma=\dfrac{\mu-\sigma}{1+\lambda_{\mathrm{KF}}},
$
so damping only changes the speed.

For $L>1$, an exactly zero mode satisfies $\dot\sigma=0$
under every fixed positive damping because the numerator contains $\sigma^p$.
Damped K-FAC therefore also requires a positive seed to initiate growth.

\paragraph{Decay of an unsupported mode.}

For zero target mode $\mu=0$, the K-FAC mode dynamics equation becomes
$
\dot{\sigma}
=
-L
\dfrac{
\sigma^{p+1}
}{
\sigma^p+\lambda_{\mathrm{KF}}
}.
$

When $\lambda_{\mathrm{KF}}=0,$ this reduces to 
$\dot{\sigma}
=
-L\sigma.
$
Hence
\[
T_{\mathrm{KFAC},0}^{\mathrm{decay}}(s,\delta)
=
\frac1L
\log\frac{s}{\delta}, \qquad \text{and} \qquad
\sigma(t)
=
s e^{-Lt}.
\]
Undamped K-FAC removes an unsupported mode exponentially in time, but the mode reaches zero only asymptotically.
Its tolerance time is logarithmic.

When $\lambda_{\mathrm{KF}}>0$ and
$
\sigma\gg\sigma_{\mathrm{KF}}^\mathrm{b},
$
the decay is still approximately exponential:
$
\dot\sigma
\approx
-L\sigma.
$
Once $\sigma\ll\sigma_{\mathrm{KF}}^\mathrm{b},$
the damping term dominates and
$
\dot{\sigma}
\approx
-\frac{L}{\lambda_{\mathrm{KF}}}
\sigma^{p+1}.
$

For $L=1$, $p=0$ and damping only rescales the exponential rate:
$
\dot\sigma
=
-\frac{1}{1+\lambda_{\mathrm{KF}}}\sigma,
$
so
\[
T_{\mathrm{KFAC}}^{\mathrm{decay}}(s,\delta)
=
(1+\lambda_{\mathrm{KF}})
\log\frac{s}{\delta}.
\]

For $L>1$, the final decay is algebraic as shown in \Cref{prop:kfac-phases}.
Equivalently, the singular value evolves as 
$
\sigma(t)
\sim
\left(
\dfrac{Lp}{\lambda_{\mathrm{KF}}}t
\right)^{-1/p}.
$
Thus positive damping converts the undamped exponential decay into a polynomial tail for every depth $L>1$.

The curvature threshold governs both birth and decay. 
Above $\sigma_{\mathrm{KF}}^\mathrm{b},$
K-FAC behaves like direct effective-matrix optimization.
Below it, damping restores the same small-mode mobility that appears under Euclidean GD.

\paragraph{Saturation at a nonzero target.}
The exact saturation times from below and above are,
$$
T_{\mathrm{KF},\pm}^{\mathrm{sat}}(r_0,\delta)
=
\frac1L\log\frac{r_0}{\delta}
+
\lambda_{\mathrm{KF}}
T_{\mathrm{GD},\pm}^{\mathrm{sat}}(r_0,\delta).
$$
Consequently,
$
T_{\mathrm{KF},\pm}^{\mathrm{sat}}(r_0,\delta)
=
\dfrac{\mu^p+\lambda_{\mathrm{KF}}}{L\mu^p}
\log\dfrac{r_0}{\delta}+O(1).
$

Therefore, the residual dynamics is
$
\dot r
\approx
-
L
\dfrac{
\mu^p
}{
\mu^p+\lambda_{\mathrm{KF}}
}
r.
$
This implies K-FAC saturates monotonically and exponentially, with the rate $\beta = L
\dfrac{\mu^p}{\mu^p+\lambda_{\mathrm{KF}}}.$

In the undamped case, $\beta = L$.
The rate is independent of the target singular value and of the mobility exponent $p$.
Under the present time normalization it is proportional to depth; after rescaling time by the common factor $L$, every mode has unit saturation rate.

With positive damping, the saturation rate has two regimes.
If $\mu^p\gg\lambda_{\mathrm{KF}},$
then $\beta \approx L$,
so the mode behaves like undamped K-FAC.
If $\mu^p\ll\lambda_{\mathrm{KF}},$
then $\beta \approx \dfrac{L}{\lambda_{\mathrm{KF}}}\mu^p. $
Small target modes therefore recover the GD-like target sensitivity and saturate more slowly.

K-FAC always approaches a nonzero target asymptotically in this model.
It does not have the finite-time relay behavior of ideal Muon. 
\end{proof}

\paragraph{Role of depth.}

Undamped K-FAC removes the depth-dependent factorization exponent from all three phases.
Positive modes have finite birth time, unsupported modes decay exponentially, and nonzero targets saturate exponentially.
Depth remains only through the common multiplicative rate $L$ under the present time convention.

Positive damping restores depth dependence through $p$.
For birth, the asymptotic delay changes from finite at $L=1$ to logarithmic at
$L=2$ and polynomial for $L>2$, with exponent $p-1$.
For unsupported modes, damping produces a polynomial tolerance time
for $L>1$.
For saturation, small target modes have rate 
$
\dfrac{L}{\lambda_{\mathrm{KF}}}\mu^p,
$
whereas sufficiently large modes approach the undamped rate $L$.

Thus damping does more than stabilize curvature inversion.
It determines which spectral modes experience natural-gradient-like cancellation and which retain the depth-dependent mobility of Euclidean GD.

\medskip

\noindent
\textbf{Takeaway.}
\emph{Undamped K-FAC cancels the factorization-induced small-mode barrier while preserving target-amplitude ordering: positive modes are born in finite time, unsupported modes decay exponentially, and all nonzero targets share the same normalized saturation rate.
Additive damping introduces the spectral threshold
$\sigma_{\mathrm{KF}}^\mathrm{b}=\lambda_{\mathrm{KF}}^{1/p}$.
Above this threshold, K-FAC remains curvature corrected; below it, the Euclidean depth laws return.
Consequently, any fixed positive damping restores logarithmic or polynomial birth delays, produces a polynomial zero-target tail for $L>1$, and slows the saturation of small target modes.
Damping is therefore an inductive-bias parameter, not merely a numerical regularizer.}

\section{Experimental Settings}
\label{app:exp_settings}

We use synthetic deep linear models to isolate the effects of optimizer geometry from nonlinear representation effects.
Unless stated otherwise, the effective predictor is
$
W = W_L W_{L-1}\cdots W_1,
$
and all hidden layers have width equal to the input dimension.
We use deterministic full batch optimization and do not use weight decay or explicit regularization.
All experiments use double precision.
We report continuous optimization time as $t = k\eta$, where $k$ is the optimization step and $\eta$ is the step size or learning rate.

For comparisons between optimizers, the learning rate is selected separately for each method because the update magnitudes induced by Euclidean, normalized, spectral, and curvature based methods are not directly comparable.
For each optimizer, we sweep the learning rate over a logarithmic grid and select the largest value that gives stable monotone progress without numerical divergence or persistent finite step oscillations away from the target, as well as reaches a similar loss within an admissible range for all optimizers.
We then adjust the training horizon so that all optimizers reach approximately the same final training loss. In general, we aim to train for $5000$ steps or lower, with each optimizer attaining similar losses.
This prevents differences in optimization progress from being interpreted as differences in spectral bias.
We plot comparisons of spectra, effective rank, and mode trajectories over time.  

\paragraph{Spectral diagnostics.}
Write $W=U\Sigma V^\top$ for a singular value decomposition, with singular values $\sigma_1\geq\cdots\geq\sigma_r\geq0$, where $r=\min(d_0,d_L)$. We measure stable rank and effective rank using
$$
\operatorname{srank}(W)=\frac{\sum_i \sigma_i^2}{\sigma_1^2}, 
\qquad
\operatorname{erank}(W)=\exp{\lp -\sum_k p_k \log p_k \rp}, \,\, p_k = \frac{\sigma_k}{\sum_i \sigma_i}.
$$
These quantities measure different aspects of spectral concentration and are reported together with the training loss.
Thresholds and conventions at $W=0$ are fixed before comparison.

\paragraph{Matrix Factorization.}

The first setting directly matches the full observation model used in the phase analysis.
The input and output dimensions are both $16$, and the network has depth $L=\{2,4,8\}$ with hidden width $16$.
The target matrix is diagonal in the initialization basis,
$$
M
=
\operatorname{diag}
\left(
1.00,
0.80,
0.60,
0.45,
0.32,
0.24,
0.18,
0.12,
0,
\ldots,
0
\right),
$$
and therefore has rank $8$.
This choice gives several positive target modes with distinct amplitudes together with unsupported modes whose target is zero.

We use balanced aligned initialization.
For the phase law experiments, positive target modes are initialized at effective amplitude $10^{-3}$
while zero target modes are initialized at $0.2$.
The corresponding factor amplitudes are initialized as
$
a_i(0)
=
s_i^{1/L},
$
so that
$
\sigma_i(0)
=
a_i(0)^L
=
s_i
$
for every depth.
This initialization preserves the aligned construction used in the theoretical analysis and allows birth and decay to be observed within the same run.

The phase statistics are evaluated every optimization step.
For birth, we track the local mode dynamics over
$
10^{-3}
\leq
\sigma
\leq
10^{-1}.
$
For decay, we follow a zero target mode from its initial active value toward zero.
For saturation, we estimate the residual dynamics
$
r(t)
=
\left|
\mu-\sigma(t)
\right|
$
for every positive target mode.
Local exponents are estimated by linear regression over a short sliding window in logarithmic coordinates.
The default window contains $21$ recorded points.

For instantaneous Adam, we set $\beta_1=\beta_2=0$ so that the update isolates the normalization mechanism studied in the phase analysis.
For Muon, we use exact compact polar normalization and no momentum.
For cumulative Shampoo, the left and right statistics are accumulated without exponential forgetting, the initial accumulator is $\lambda_{\mathrm{Sh}}I$, and the preconditioner is refreshed at every step.
For K-FAC, we use the additive block Gauss Newton model analyzed in the theory.
No update norm clipping is applied to K-FAC.
The reported learning rates are numerical integration parameters rather than quantities used to establish optimizer rankings. 
We train all the optimizers for $5000$ steps or lower, and the hyperparameters are given in \Cref{tab:matrix_factorization_hparams}.

\begin{table}[h!]
\centering
\caption{
Optimizer settings used for the full observation matrix factorization experiments using depths $L=\{2,4,8\}$.
The learning rates are chosen to keep the discrete trajectories close to the corresponding continuous flows while allowing all methods to reach comparable final training loss. Each optimizer is trained for a maximum of $5000$ steps.
}
\label{tab:matrix_factorization_hparams}
\begin{tabular}{lllll}
\toprule
Optimizer
& Learning rate
& Main state parameter
& State value
& Depth
\\
\midrule
GD
& $5\times10^{-2}$
& none
& none
& $L=\{2,4\}$
\\
GD
& $8\times10^{-2}$
& none
& none
& $L=\{8\}$
\\
 Adam
& $1\times10^{-3}$
& $(\beta_1, \beta_2, \epsilon)$
& $(0,0,0.02)$
& $L=\{2,4\}$
\\
 Adam
& $2\times10^{-3}$
& $(\beta_1, \beta_2, \epsilon)$
& $(0,0,0.02)$
& $L=\{8\}$
\\
Muon
& $5\times10^{-4}$
& momentum
& $0$
& $L=\{2,4,8\}$
\\
 Shampoo
& $2\times10^{-4}$
& $\lambda_{\mathrm{Sh}}$
& $10^{-6}$
& $L=\{2,4,8\}$
\\
K-FAC, undamped
& $5\times10^{-4}$
& $\lambda_{\mathrm{KF}}$
& $0$
& $L=\{2,4,8\}$
\\
K-FAC, damped
& $5\times10^{-4}$
& $\lambda_{\mathrm{KF}}$
& $5\times 10^{-3}$
& $L=\{4,8\}$
\\
K-FAC, damped
& $8\times10^{-4}$
& $\lambda_{\mathrm{KF}}$
& $5\times 10^{-3}$
& $L=\{2\}$
\\
\bottomrule
\end{tabular}
\end{table}

\paragraph{Linear Regression.}

The second setting studies optimizer behavior when the observations do not uniquely determine the effective matrix.
We generate noiseless multivariate linear regression data according to
$
Y
=
W_{\mathrm{true}}X,
$
and train the deep linear predictor by minimizing
$
\mathcal{L}(W)
=
\frac{1}{2n}
\left\|
WX-Y
\right\|_F^2.
$

We use input output data dimensions as $64$ and number of training samples $32$,
so the regression problem is underdetermined.
The teacher has rank $8$.
Its nonzero singular values are the same as $M$ in matrix factorization.
The left and right singular directions of the teacher are sampled once from random orthogonal bases and are kept fixed across all optimizer runs.

The entries of $X\in\mathbb{R}^{64\times32}$ are sampled independently from a standard Gaussian distribution.
We do not whiten the data.
Consequently,
$
\Sigma_X
=
\dfrac{1}{n}
XX^\top
$
has rank at most $32$ and is not proportional to the identity.
The effective gradient is therefore
$
G
=
(W-W_{\mathrm{true}})
\Sigma_X,
$
rather than $W-W_{\mathrm{true}}$.
This setting intentionally leaves the balanced aligned scalar regime and tests whether the optimizer dependent spectral behavior persists when the data covariance couples target modes.

The network has depth $L=4$ and hidden width $64$.
Each layer is initialized with a scaled orthogonal matrix.
The initialization scale is $0.1$, giving a small but full rank initial effective predictor while avoiding the extremely slow escape observed for much smaller deep products.
The same initialization is used for all optimizers.

\begin{table}[t]
\centering
\caption{
Optimizer settings used for underdetermined linear regression with depth $L=4$.
The learning rates were selected over a different grid for different optimizers to obtain stable training and comparable final data loss.
}
\label{tab:linear_regression_hparams}
\begin{tabular}{llll}
\toprule
Optimizer
& Learning rate
& Main state parameter
& State value
\\
\midrule
GD
& $3$
& none
& none
\\
Adam
& $0.05$
& $(\beta_1,\beta_2,\epsilon)$
& $(0,0,0.02)$
\\
Muon
& $0.00025$
& momentum
& $0$
\\
Shampoo
& $0.003$
& $\lambda_{\mathrm{Sh}}$
& $10^{-6}$
\\
K-FAC
& $0.0005$
& $\lambda_{\mathrm{KF}}$
& $10^{-4}$
\\
\bottomrule
\end{tabular}
\end{table}

For this experiment, we use similar optimizers as in the matrix factorization experiment.
The specific learning rates are selected for this finite sample problem rather than transferred from the aligned full observation experiments.

Because the optimizers have substantially different effective step sizes, a common iteration count does not imply a common degree of optimization.
We therefore use the final training loss as the primary tuning criterion.
For each optimizer, the learning rate is selected from a different grid for different optimizers, particularly  
smaller values for strongly preconditioned methods when necessary.
Configurations that diverge or produce persistent numerical instability are discarded.
Among the stable configurations, we select the largest learning rate for which the final training loss is within a factor of two of the lowest final loss attained by the compared methods.
We obtain the hyperparameters in \Cref{tab:linear_regression_hparams} such that the training loss for all optimizers is within a range when trained for $5000$ steps or lower. 

Unless stated otherwise, all reported curves use the same random data realization and the same model initialization across optimizers.

\paragraph{Hardware and Code.} 
Each experiment is run on 1 NVIDIA GeForce RTX 4090 and takes about 5-10 minutes to complete for each optimizer.
A functioning code repository is provided as a zip file in the supplementary material. 

\section{Additional Experimental Results}
\label{app:exp}

\subsection{Matrix Factorization}
\label{sec:exp_matrix_factorization}

We first evaluate the optimizer dependent spectral dynamics in the full observation matrix factorization setting described in \Cref{app:exp_settings}, where the theoretical phase laws apply most directly.

\begin{figure}[h!]
    \centering
    \includegraphics[width=\linewidth]{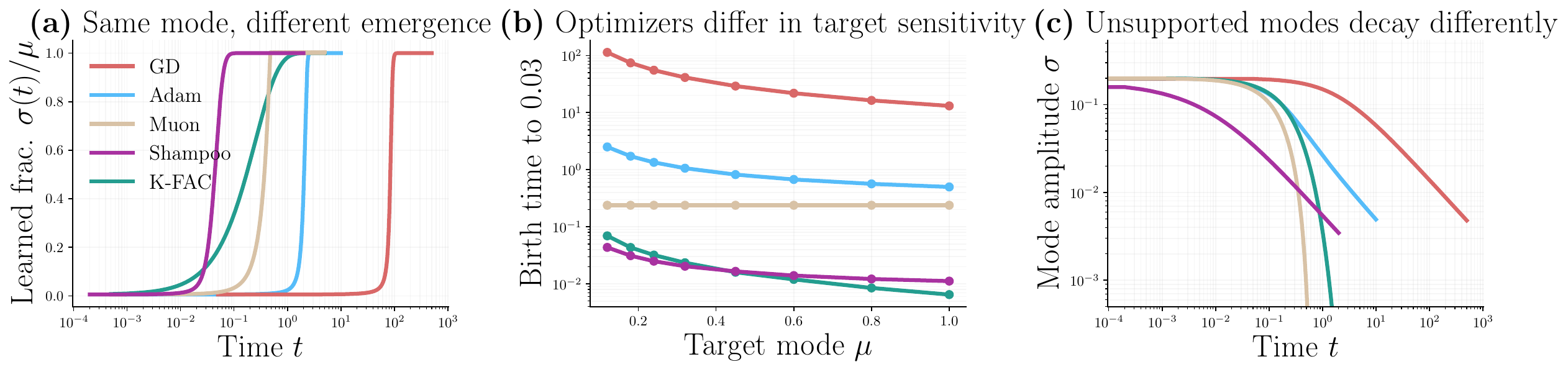}
    \vspace{-0.5cm}
    \caption{\textbf{Optimizer geometry sets the pace of spectral learning evaluation in matrix factorization.} 
    \textbf{(a)} Evolution of a weak target mode with \(\mu=0.18\), shown as the learned fraction \(\sigma(t)/\mu\). GD exhibits the factorization induced small mode barrier, whereas normalization and preconditioning substantially accelerate emergence. 
    \textbf{(b)} Birth time to the common amplitude $b=0.03$ as the target mode $\mu$ varies. While ideal Muon is independent of the target mode ordering during birth, other optimizers are target sensitive. 
    \textbf{(c)} Decay trajectories for an unsupported mode with target $\mu=0$.
    Ideal Muon removes the mode in finite time, undamped K-FAC produces exponential decay, while GD, instantaneous Adam with fixed stabilization, and cumulative Shampoo retain increasingly slow tails near zero. 
    Optimizer geometry therefore determines which spectral components emerge, persist, and disappear.
    }
    \label{fig:intro}
    \vspace{-0.4cm}
\end{figure}

\Cref{fig:intro} provides a broader validation of the theory in addition to the theoretical evaluation done in \Cref{sec:exp}.
\Cref{fig:matrix_factorization_summary} compares the optimization and rank dynamics.
Although the methods eventually reach similar losses and effective ranks, their trajectories differ substantially.
GD learns the target modes sequentially and retains a low effective rank for longer, consistent with the small mode mobility factor $\sigma^p$.
Adam reduces this separation and activates weaker modes much earlier.
Muon produces an even more rapid redistribution of spectral mass, reflecting its normalized spectral update.
Cumulative Shampoo initially behaves more like GD but quickly changes regime as its accumulated preconditioner adapts.
K-FAC exhibits a smoother evolution in which curvature correction reduces the factorization induced separation between mode time scales.
The stable and effective rank curves therefore mainly differ during the transient phase, while their later agreement reflects the unique full observation target rather than identical optimization dynamics.

\begin{figure}[h!]
    \centering
    \includegraphics[width=0.32\linewidth]{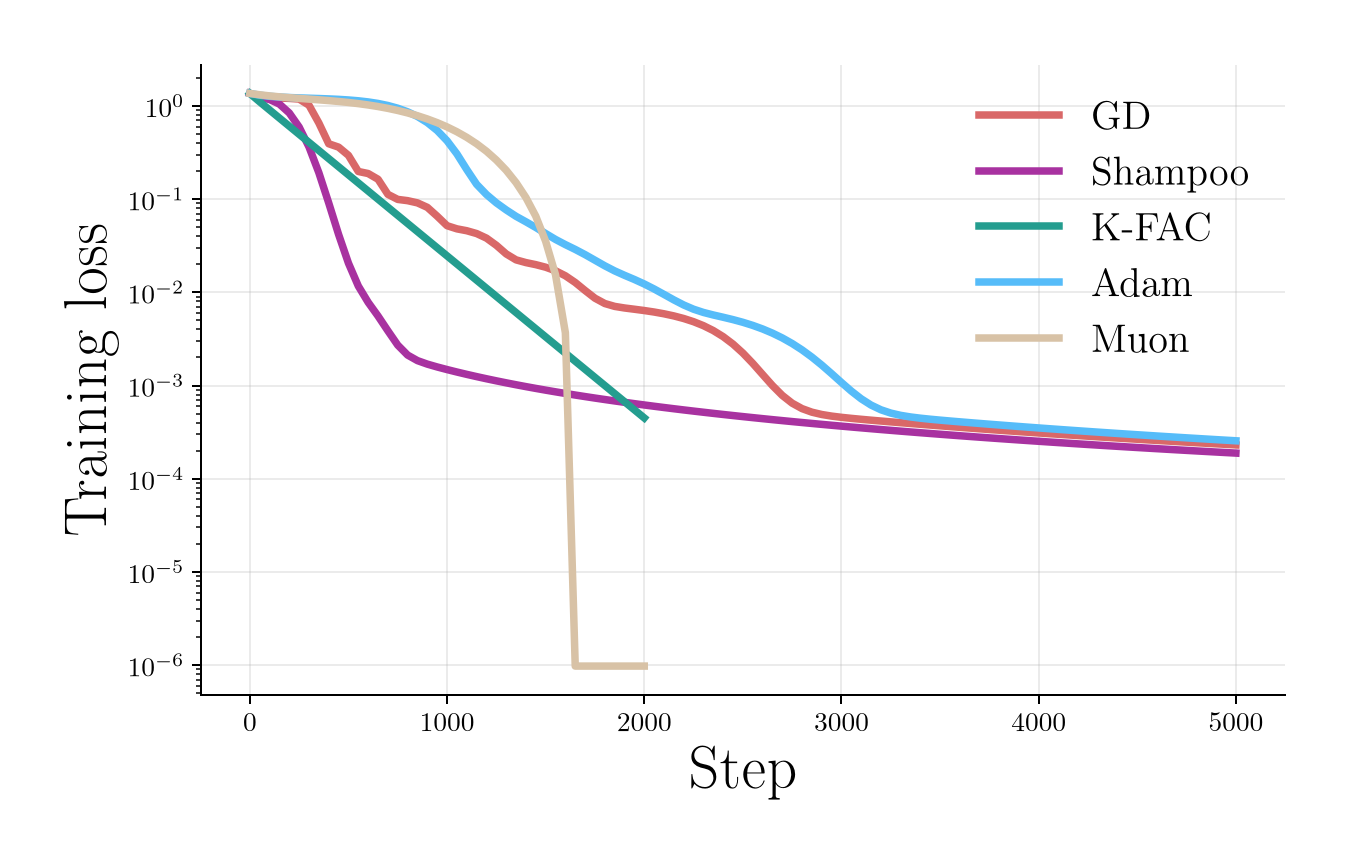}
    \includegraphics[width=0.32\linewidth]{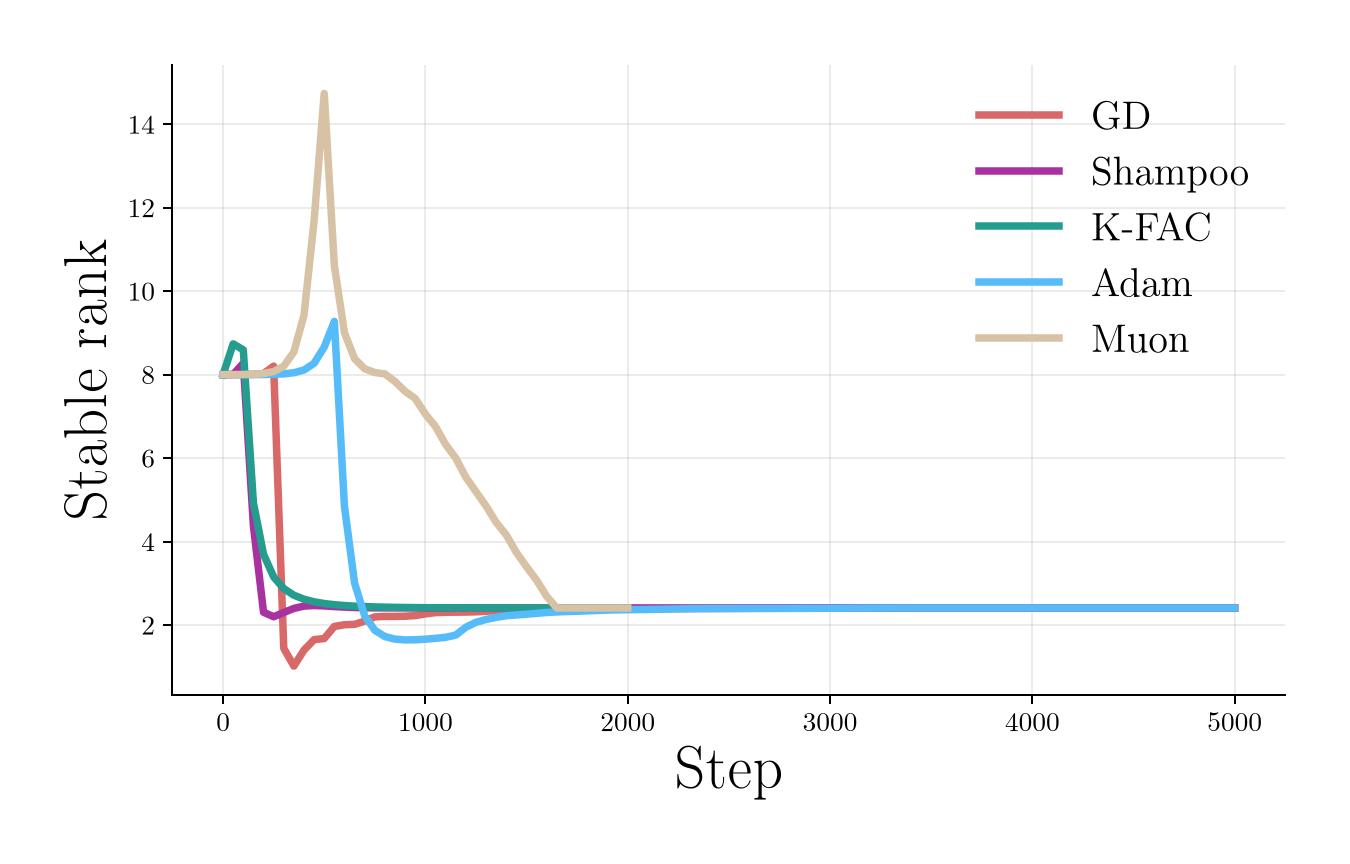}
    \includegraphics[width=0.32\linewidth]{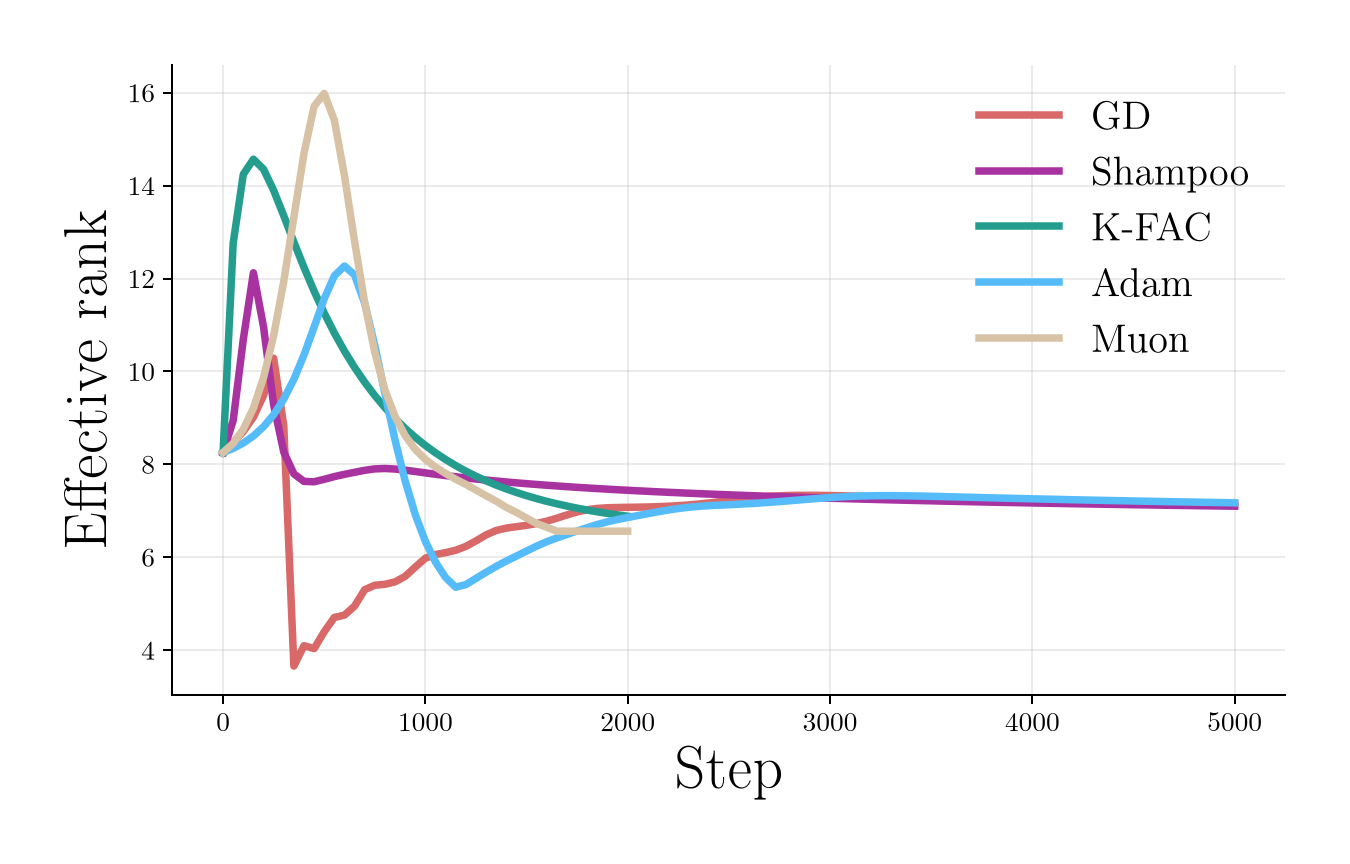}
    \caption{\textbf{Optimization and rank dynamics in matrix factorization.}
All optimizers train the same depth $L=4$ deep linear factorization toward the same rank eight target and reach comparable final training loss.
\textbf{Left:}
training loss over optimization steps.
\textbf{Middle:}
stable rank of $W$.
\textbf{Right:}
effective rank of $W$.
The optimizers exhibit markedly different transient rank dynamics despite converging toward similar final spectra.
GD retains a lower rank representation for longer, whereas Adam, Muon, cumulative Shampoo, and K-FAC alter the rate at which spectral mass is distributed across modes.
}
    \label{fig:matrix_factorization_summary}
\end{figure}

The singular value trajectories in \Cref{fig:matrix_factorization_modes} make these differences more explicit.
Under GD, larger target modes emerge first and weaker modes remain close to initialization for substantially longer times.
This sequential emergence is the dynamical mechanism underlying its transient low rank bias.
Adam substantially compresses these mode dependent delays, while Muon causes the positive modes to follow nearly the same spectral clock until they encounter their different target amplitudes.
Cumulative Shampoo exhibits the predicted two stage behavior: early growth remains separated across modes, after which accumulated gradient history accelerates the weaker modes and reduces target amplitude dependence.
K-FAC produces the weakest separation between mode trajectories, consistent with curvature correction compensating for the factorization induced mobility barrier.
These trajectories illustrate that optimizers reaching similar final solutions can nevertheless pass through very different spectral representations during training.

\begin{figure}[h!]
    \centering
    \begin{subfigure}{0.33\textwidth}
        \centering
        \includegraphics[width=\linewidth]{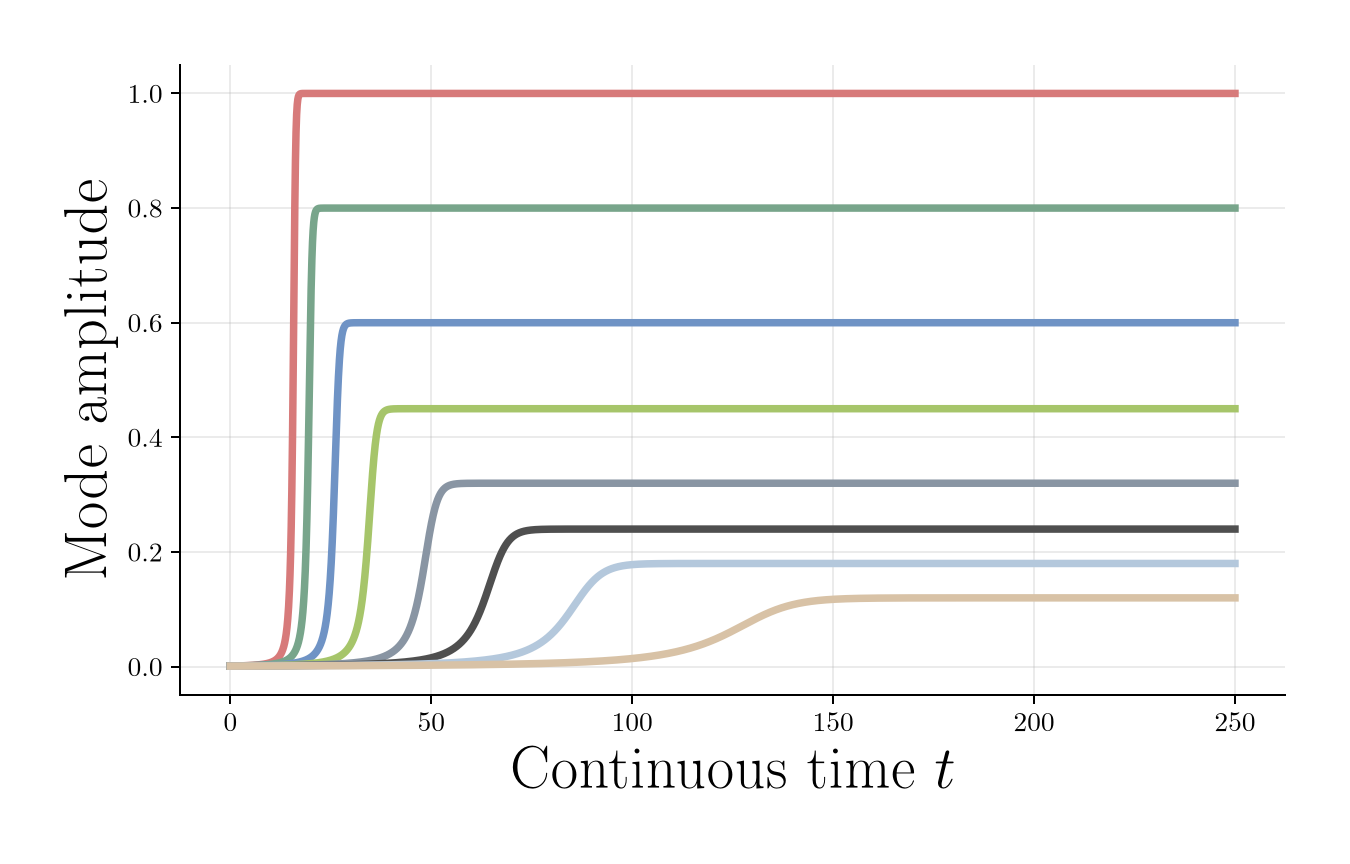}
        \caption{GD}
        \label{fig:gd_mode}
    \end{subfigure}%
    \begin{subfigure}{0.33\textwidth}
        \centering
        \includegraphics[width=\linewidth]{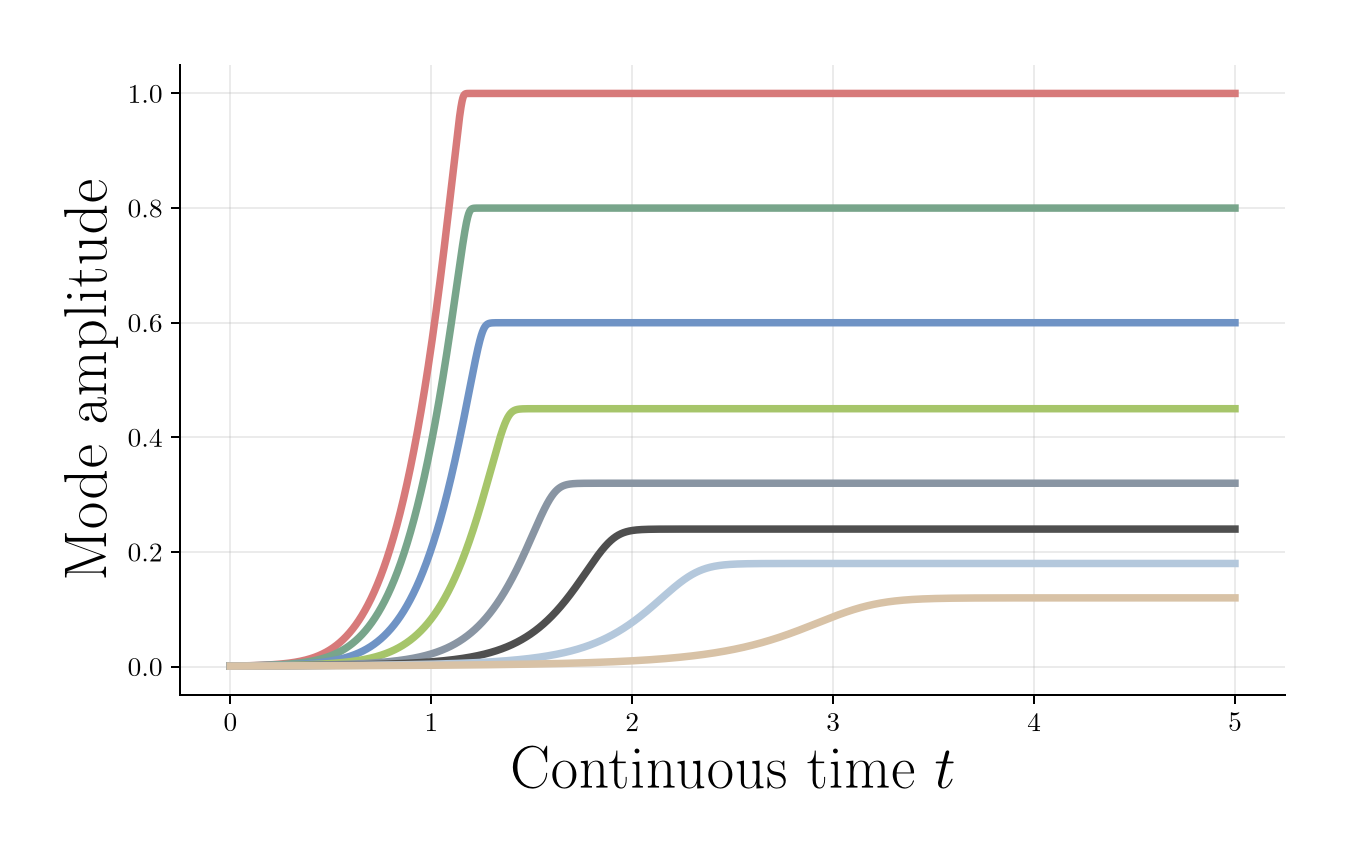}
        \caption{Adam}
        \label{fig:adam_mode}
    \end{subfigure}
    \begin{subfigure}{0.33\textwidth}
        \centering
        \includegraphics[width=\linewidth]{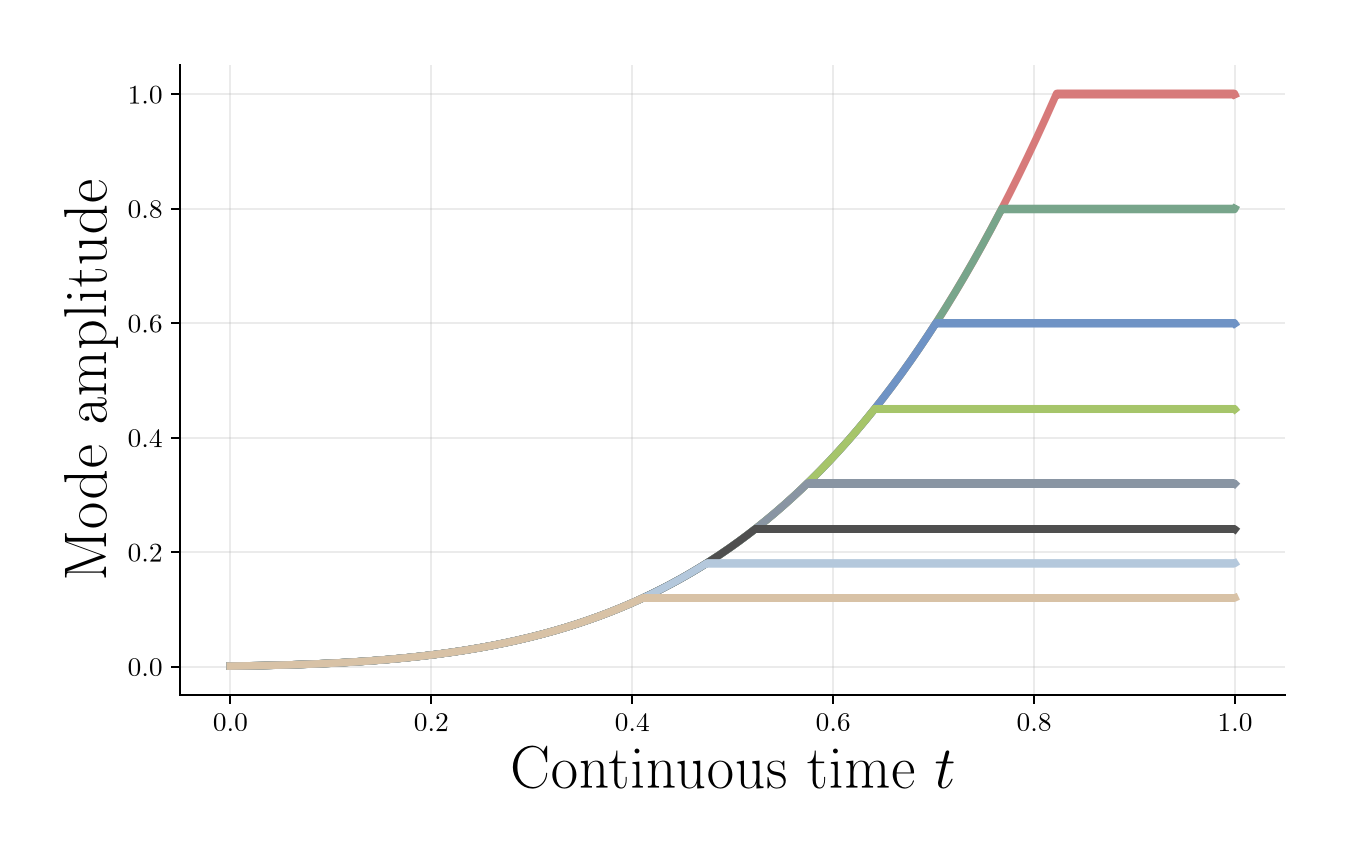}
        \caption{Muon}
        \label{fig:muon_mode}
    \end{subfigure}
    \begin{subfigure}{0.33\textwidth}
        \centering
        \includegraphics[width=\linewidth]{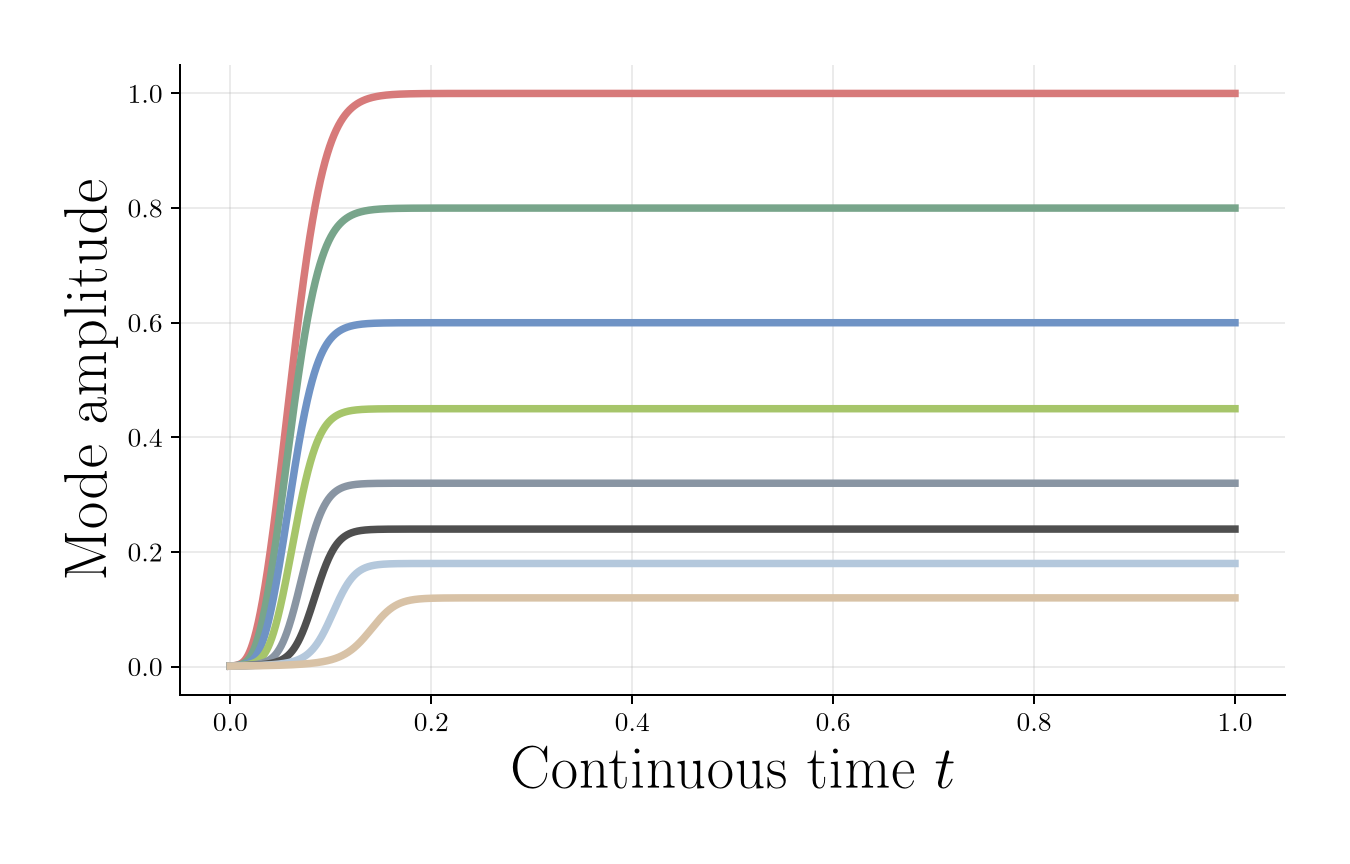}
        \caption{Shampoo}
        \label{fig:shampoo_mode}
    \end{subfigure}
    \begin{subfigure}{0.33\textwidth}
        \centering
        \includegraphics[width=\linewidth]{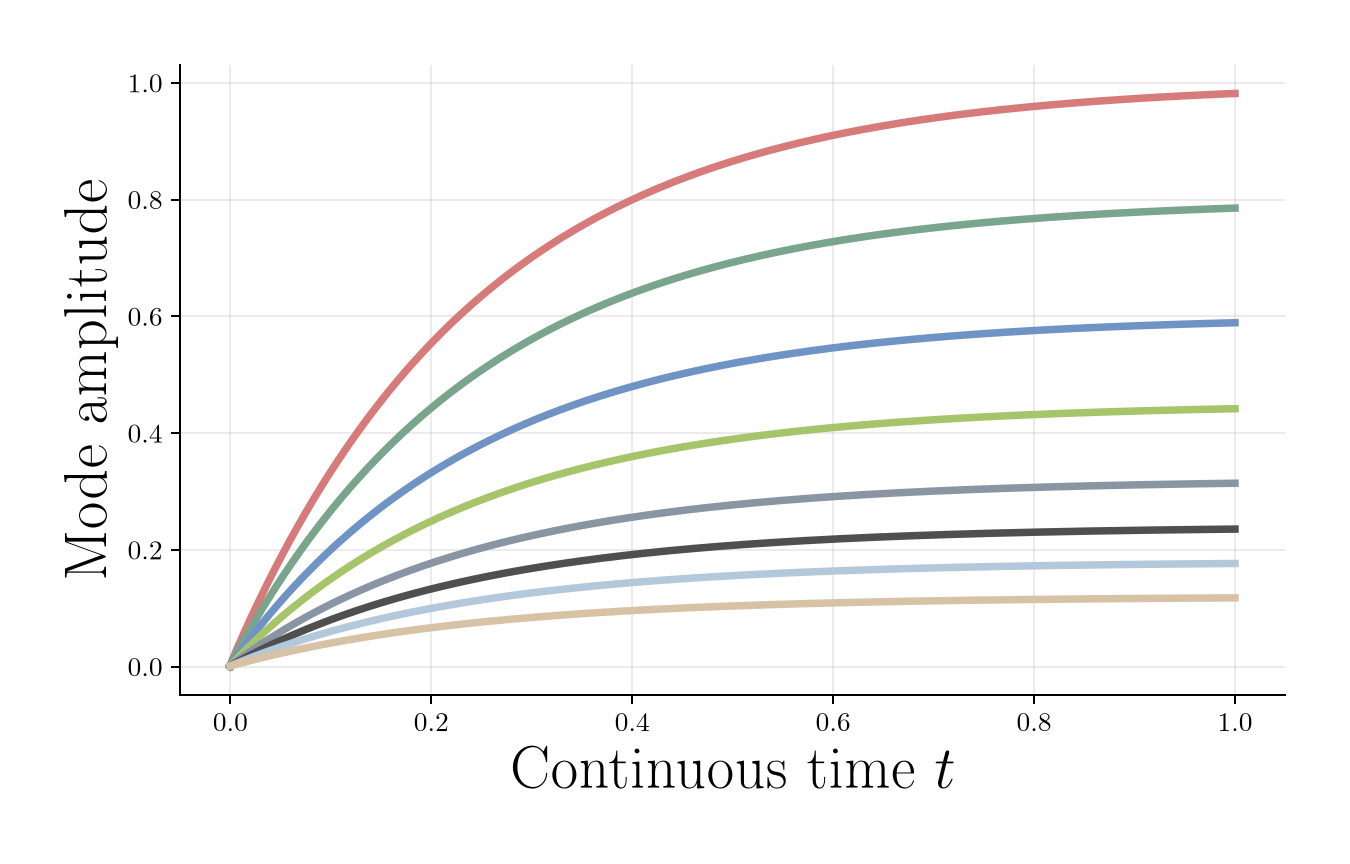}
        \caption{K-FAC}
        \label{fig:kfac_mode}
    \end{subfigure}
    \caption{\textbf{Singular mode trajectories reveal optimizer dependent spectral clocks.} 
GD exhibits strongly sequential learning;
Adam reduces these delays, while Muon makes the modes evolve on nearly the same clock until they approach their respective targets.
Cumulative Shampoo shows 
a faster history adapted growth of weaker modes.
K-FAC further reduces the factorization induced separation between mode time scales through curvature correction.
The common final amplitudes are determined by the unique full observation target, whereas the transient trajectories expose the distinct inductive biases of the optimizers.
}
    \label{fig:matrix_factorization_modes}
\end{figure}

\FloatBarrier

\subsubsection{Role of Depth}
Finally, we vary the factorization depth while keeping the target matrix and effective initialization fixed.
\Cref{fig:depth_phase_transitions} focuses on Adam, cumulative Shampoo, and K-FAC for depths$=\{2,4,8\}$ because their state dependent transitions make this interaction particularly explicit.
For Adam, increasing depth raises both the GD like and normalized exponents and changes the location of the $\epsilon$ controlled crossover.
For Shampoo, depth changes both the initial exponent $p$ and the history adapted exponent $p-1$, while also shifting the point at which accumulated history overtakes the initial damping.
For damped K-FAC, increasing depth strengthens the GD like small mode regime through $p$, whereas the curvature corrected regime continues to approach exponent zero.
The corresponding depth dependence for GD and Muon follows directly from their fixed exponents $p$ and $q$ and is consistent with the phase laws reported in the main text.

\begin{figure}[h!]
    \centering
    \includegraphics[width=0.8\linewidth]{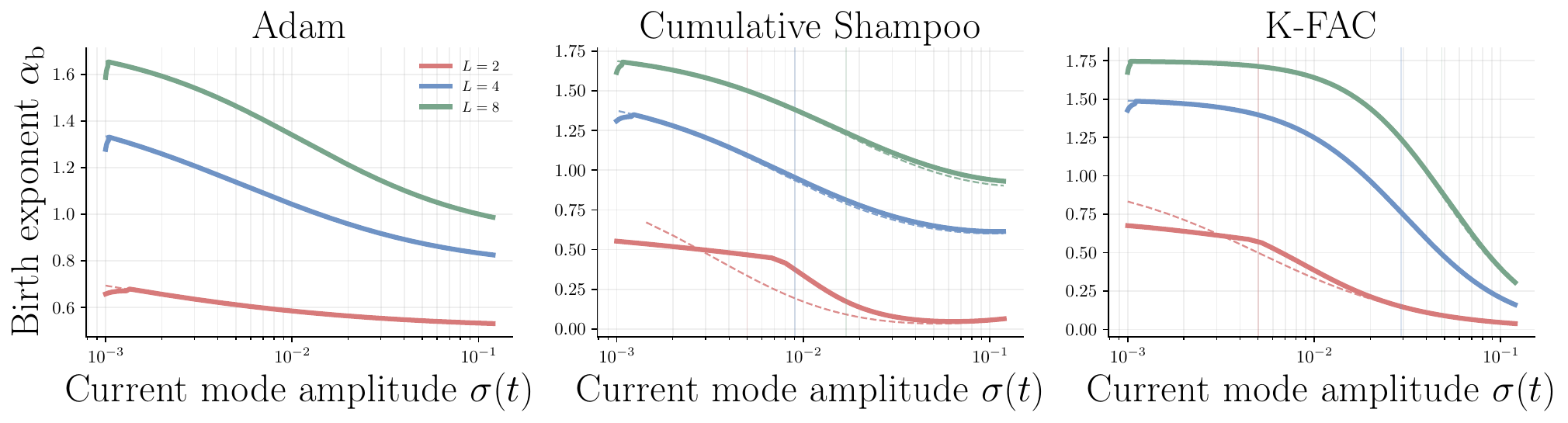}
      \caption{\textbf{Depth changes both the spectral exponents and the optimizer state transitions.}
\textbf{Left:}
for instantaneous Adam, increasing depth changes both the GD like exponent $p=2-2/L$ and the normalized exponent $q=1-1/L$. 
\textbf{Middle:}
for cumulative Shampoo, depth changes the initial GD like exponent $p$, the history adapted exponent $p-1$, and the amplitude at which accumulated history begins to dominate.
\textbf{Right:}
for damped KFAC, deeper factorizations strengthen the small mode GD like barrier while curvature correction still drives the exponent toward zero above the damping controlled mobility threshold.
}
    \label{fig:depth_phase_transitions}
\end{figure}

\Cref{fig:depth_decay_shampoo_saturation} studies the decay exponent of unsupported modes and focuses on Shampoo saturation, where mode order reversal can happen theoretically.
For decay, increasing depth changes the small mode geometry through $p=2-2/L$.
Empirically, GD, Adam, and cumulative Shampoo follow the predicted increase
$
\alpha_d = p+1 = 3-\frac{2}{L},
$
so deeper factorizations produce a steeper polynomial tail near zero.
In contrast, undamped K-FAC remains close to $\alpha_d=1$, reflecting its exponential decay law, while Muon exhibits a much weaker depth dependence, consistent with its normalized geometry.
Thus depth does not affect all optimizers uniformly: it strengthens the Euclidean style barrier for GD like methods, but has much less impact on normalized or curvature corrected decay.

For saturation, we focus on cumulative Shampoo because it is the optimizer for which mode reversal can happen and depth also affects it.
Its asymptotic saturation coefficient depends on both target size and accumulated history,
$
\beta_{\rm Sh}(\mu)=\dfrac{L\mu^p}{D_\star(\mu)}.
$
As a result, the ordering of saturation rates across target modes need not agree with GD.
The right panel shows precisely this effect.
For $L=2$, $\beta$ decreases with $\mu$, so smaller target modes saturate faster than larger ones.
For $L=8$, the trend reverses and larger target modes saturate faster.
The case $L=4$ lies near the transition and becomes weakly nonmonotone.
This depth driven change in ordering is the clearest empirical instance of the Shampoo mode reversal predicted by the theory.

\begin{figure}[h!]
    \centering
    \includegraphics[width=0.38\linewidth]{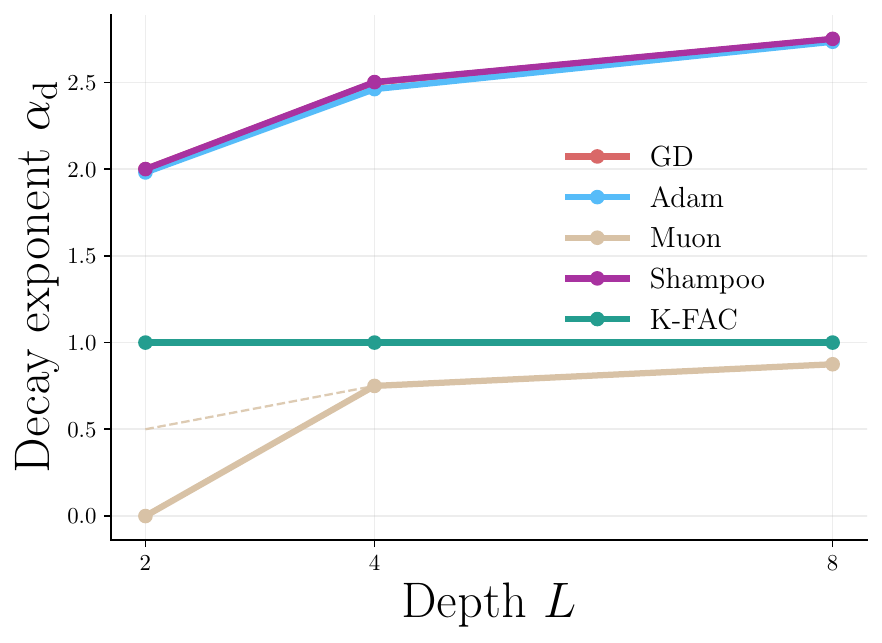}
    \includegraphics[width=0.38\linewidth]{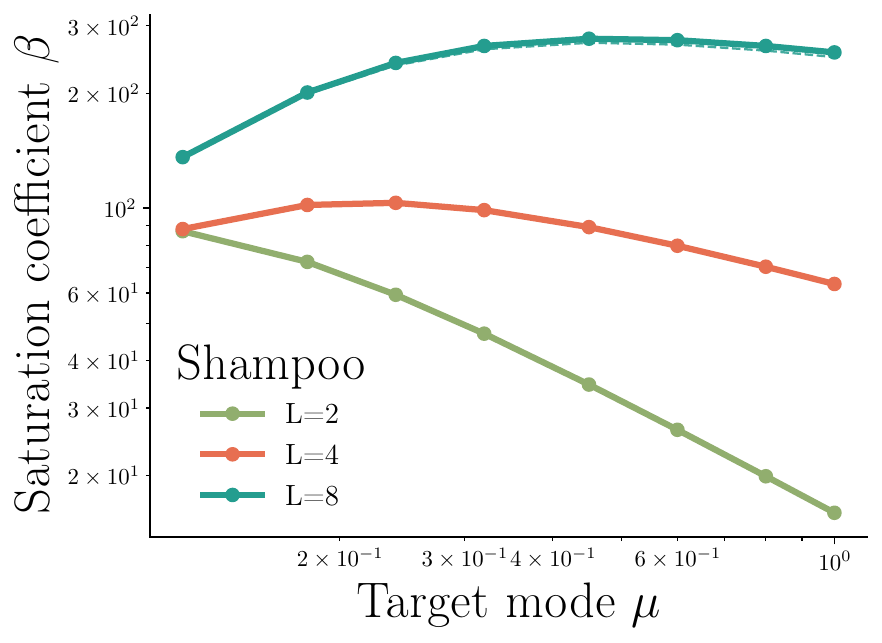}
  \caption{
    \textbf{Effect of depth on decay laws and Shampoo saturation order.}
\textbf{Left:}
empirical decay exponents $\alpha_d$ for unsupported modes as depth varies.
All optimizers closely follow the theoretical predictions in general.
\textbf{Right:}
Shampoo saturation coefficient $\beta$ as a function of the positive target mode $\mu$ for $L\in\{2,4,8\}$.
Shampoo is singled out here because its saturation phase is where mode order reversal can happen theoretically.
At $L=2$, $\beta$ decreases with $\mu$, so smaller target modes saturate faster.
At $L=8$, the ordering reverses and larger target modes saturate faster, while $L=4$ lies near the transition.
These results show that depth influences not only the size of the phase law exponents, but also the ordering of mode fitting under optimizer state dynamics.
}
\label{fig:depth_decay_shampoo_saturation}
\end{figure}

\Cref{fig:L2,fig:L8} show the evolution of training loss and ranks throughout training for depths $L=2$  and $8$, respectively. In both cases, all optimizers converge to similar rank and loss solutions. But their trajectory shows the difference.
Similarly, \Cref{fig:matrix_factorization_modes_l2,fig:matrix_factorization_modes_l8} show the singular mode trajectories for $L=2$ and $8$, respectively. In both cases, we observe that GD exhibits strong sequential learning and at the same time other optimizers also begin to show stronger bias with depth. 

\begin{figure}[h!]
    \centering
    \includegraphics[width=0.32\linewidth]{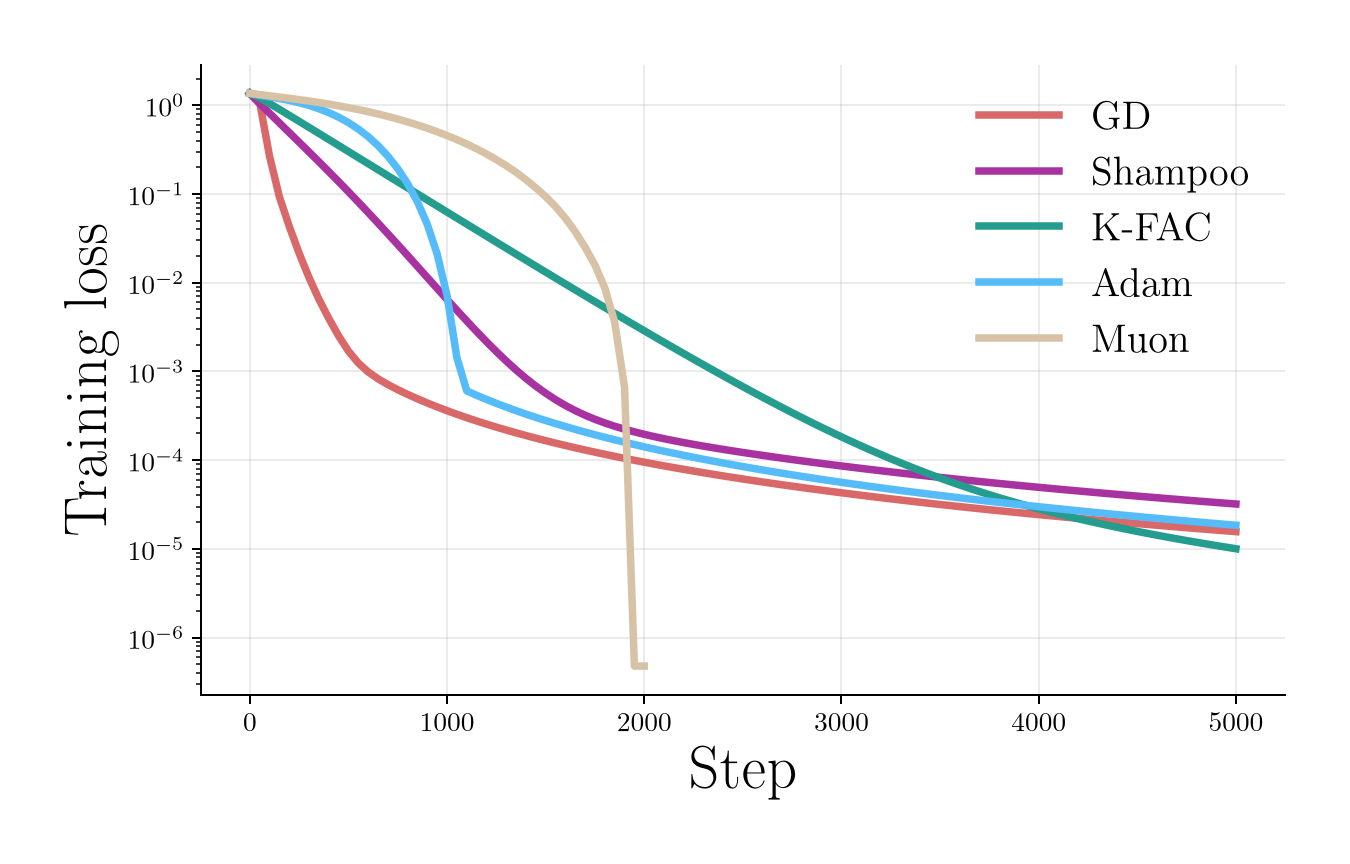}
    \includegraphics[width=0.32\linewidth]{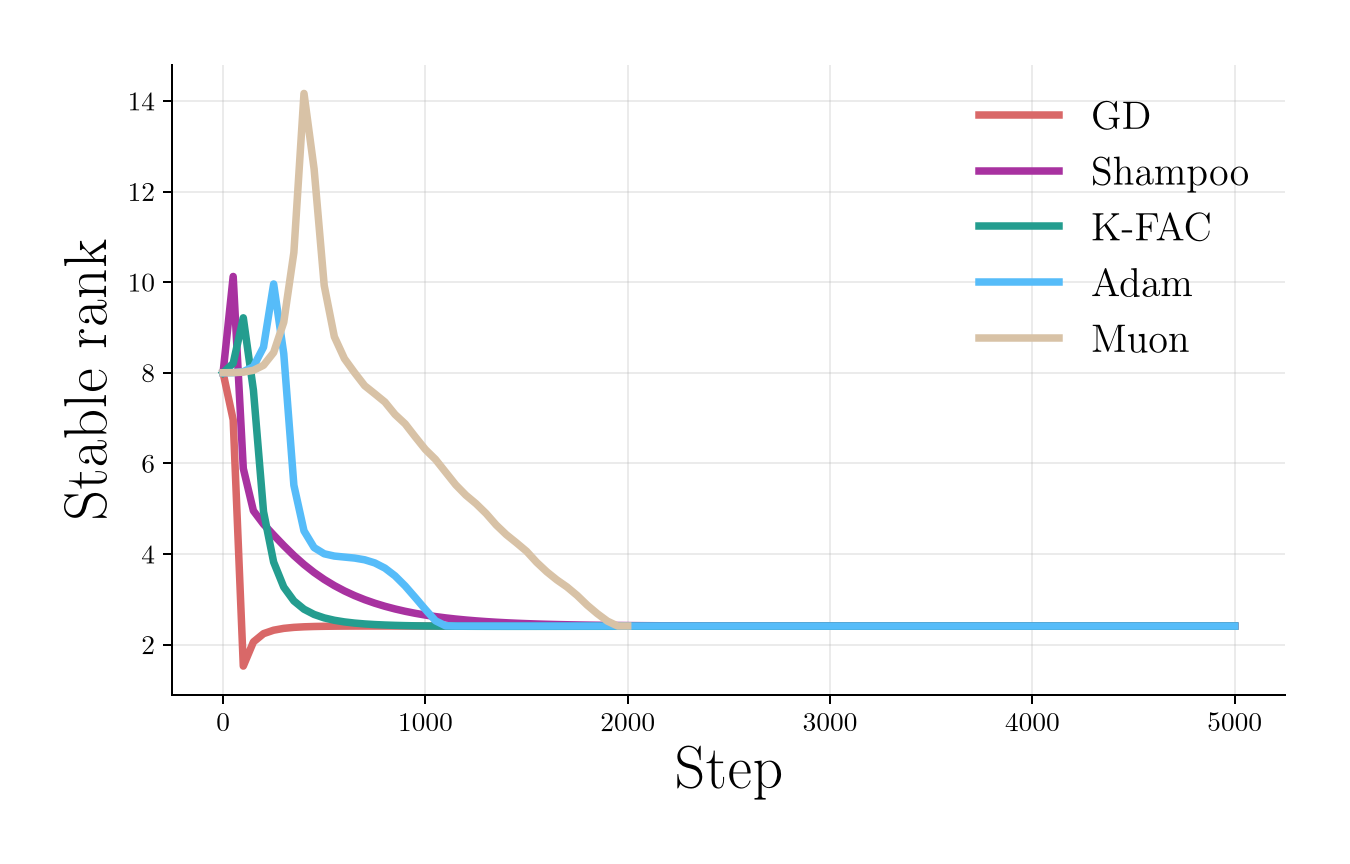}
    \includegraphics[width=0.32\linewidth]{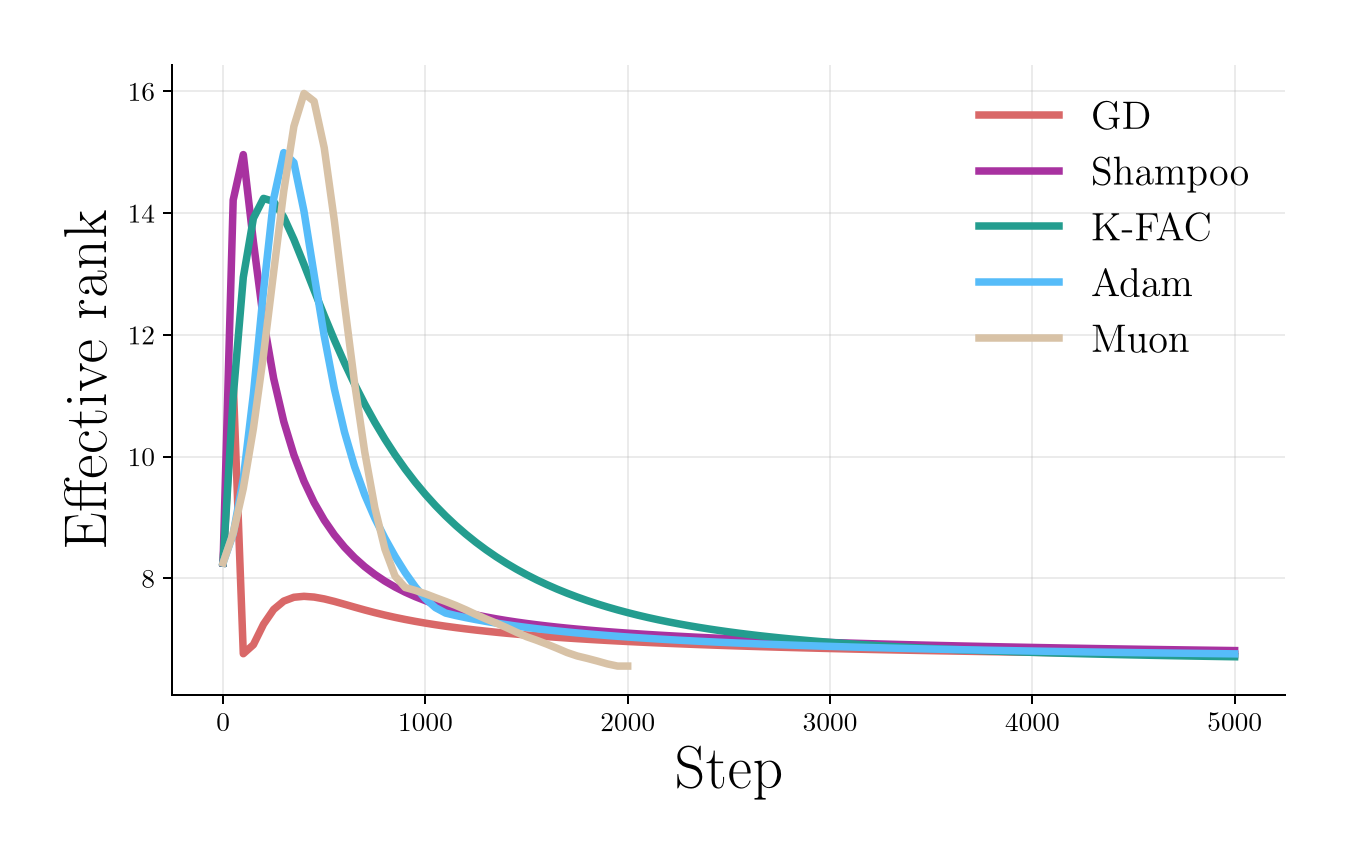}
    \caption{\textbf{Optimization and rank dynamics in full observation matrix factorization.}
All optimizers train the same depth $L=2$ deep linear factorization toward the same rank eight target and reach comparable final training loss.}
    \label{fig:L2}
\end{figure}

\begin{figure}[h!]
    \centering
    \begin{subfigure}{0.33\textwidth}
        \centering
        \includegraphics[width=\linewidth]{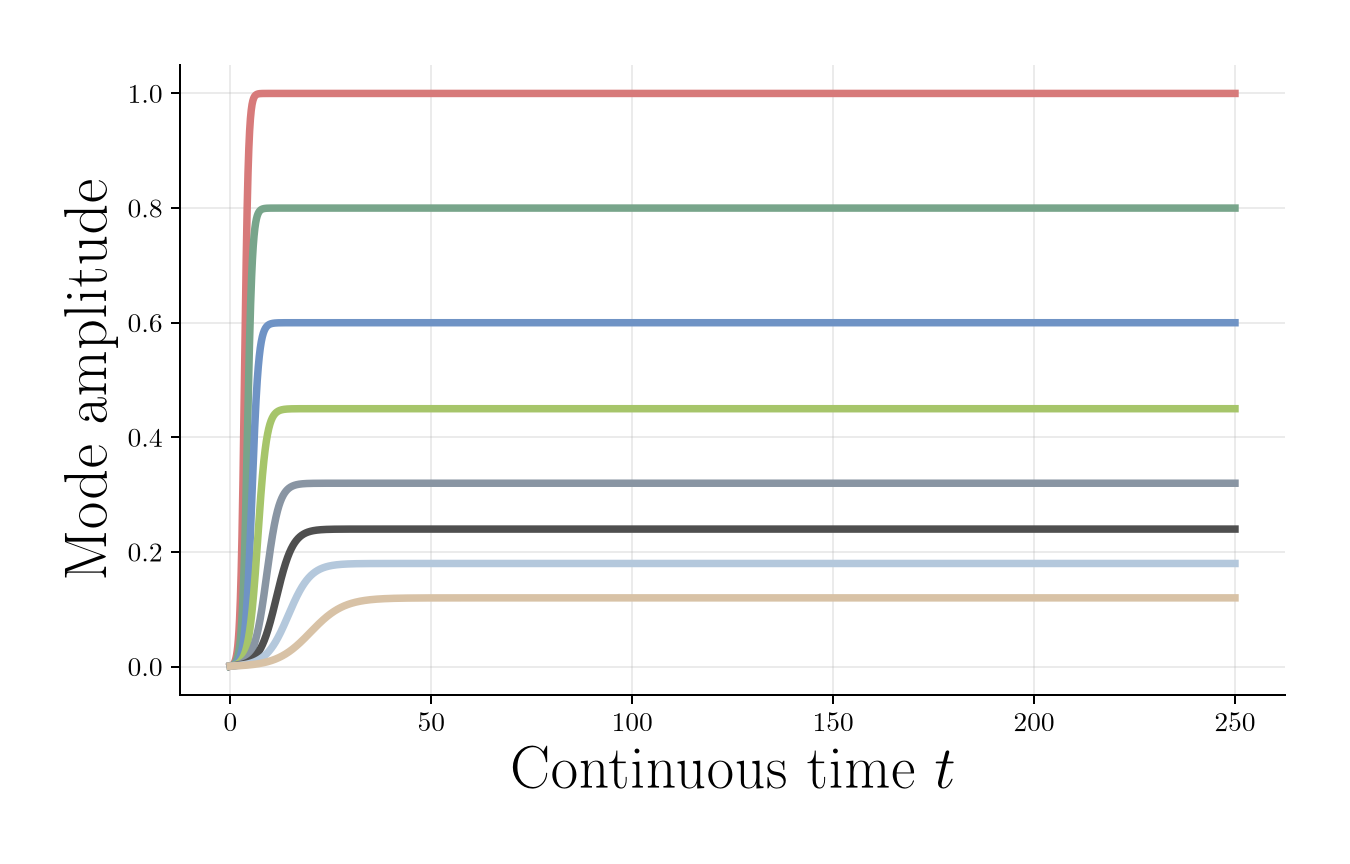}
        \caption{GD}
        \label{fig:gd_mode}
    \end{subfigure}%
    \begin{subfigure}{0.33\textwidth}
        \centering
        \includegraphics[width=\linewidth]{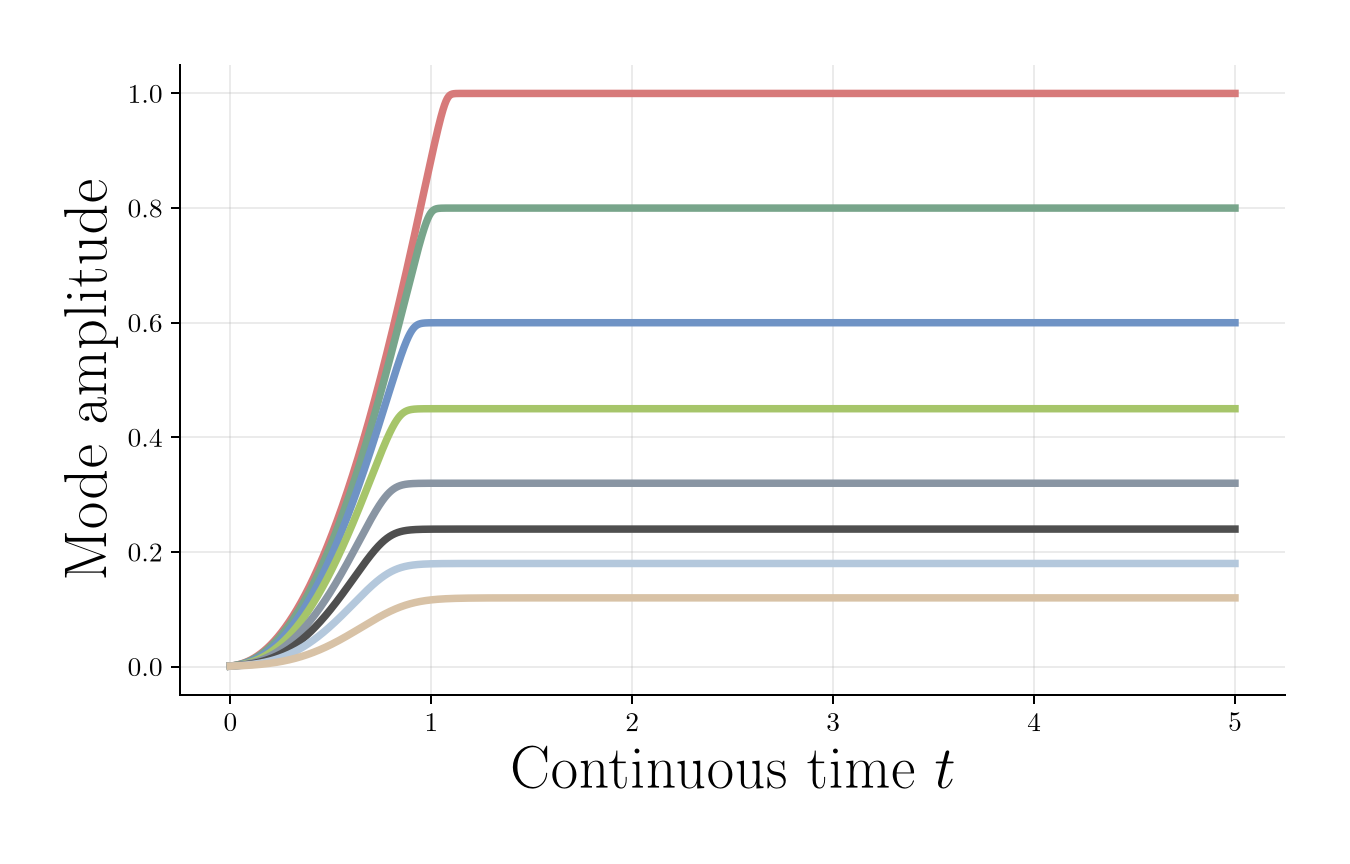}
        \caption{Adam}
        \label{fig:adam_mode}
    \end{subfigure}
    \begin{subfigure}{0.33\textwidth}
        \centering
        \includegraphics[width=\linewidth]{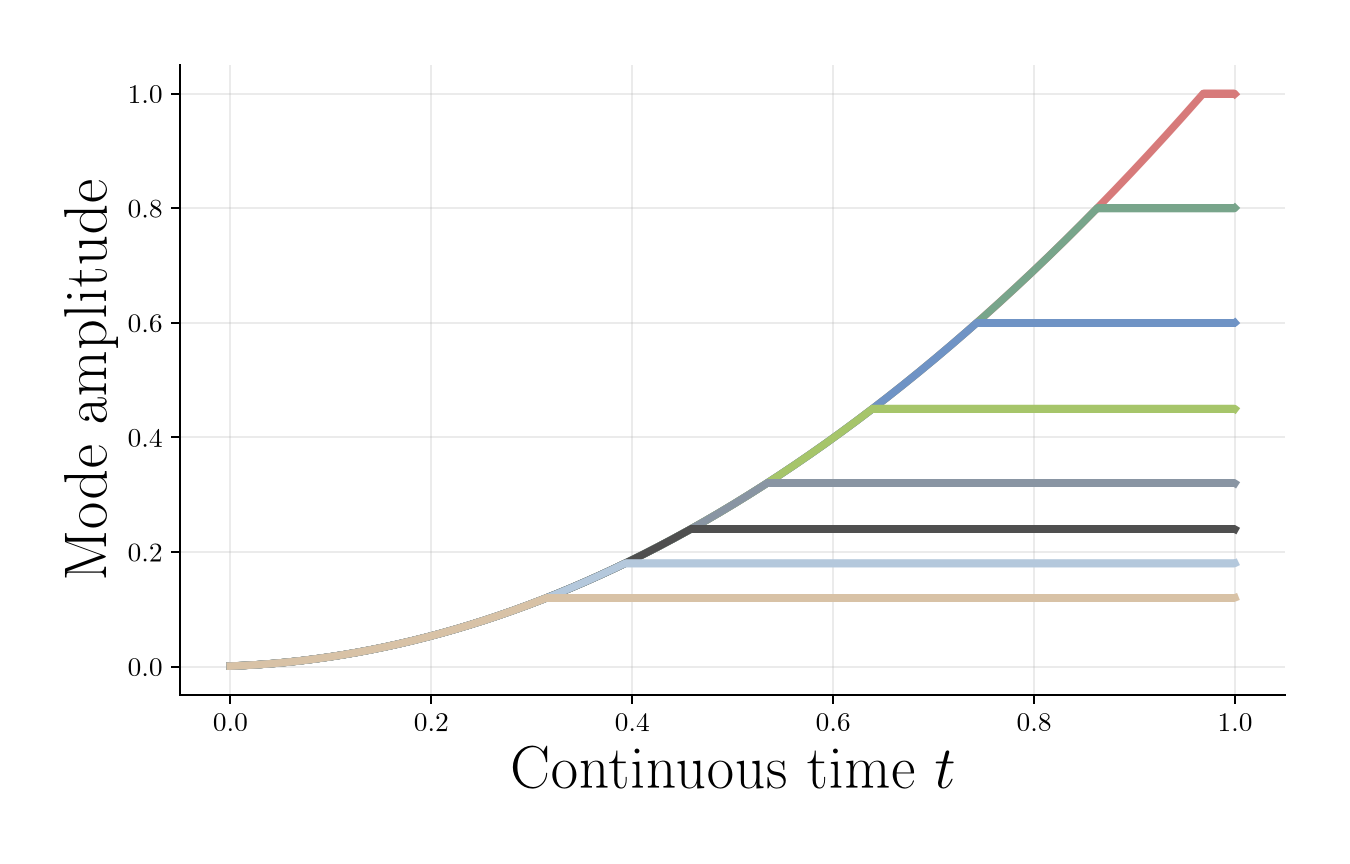}
        \caption{Muon}
        \label{fig:muon_mode}
    \end{subfigure}
    \begin{subfigure}{0.33\textwidth}
        \centering
        \includegraphics[width=\linewidth]{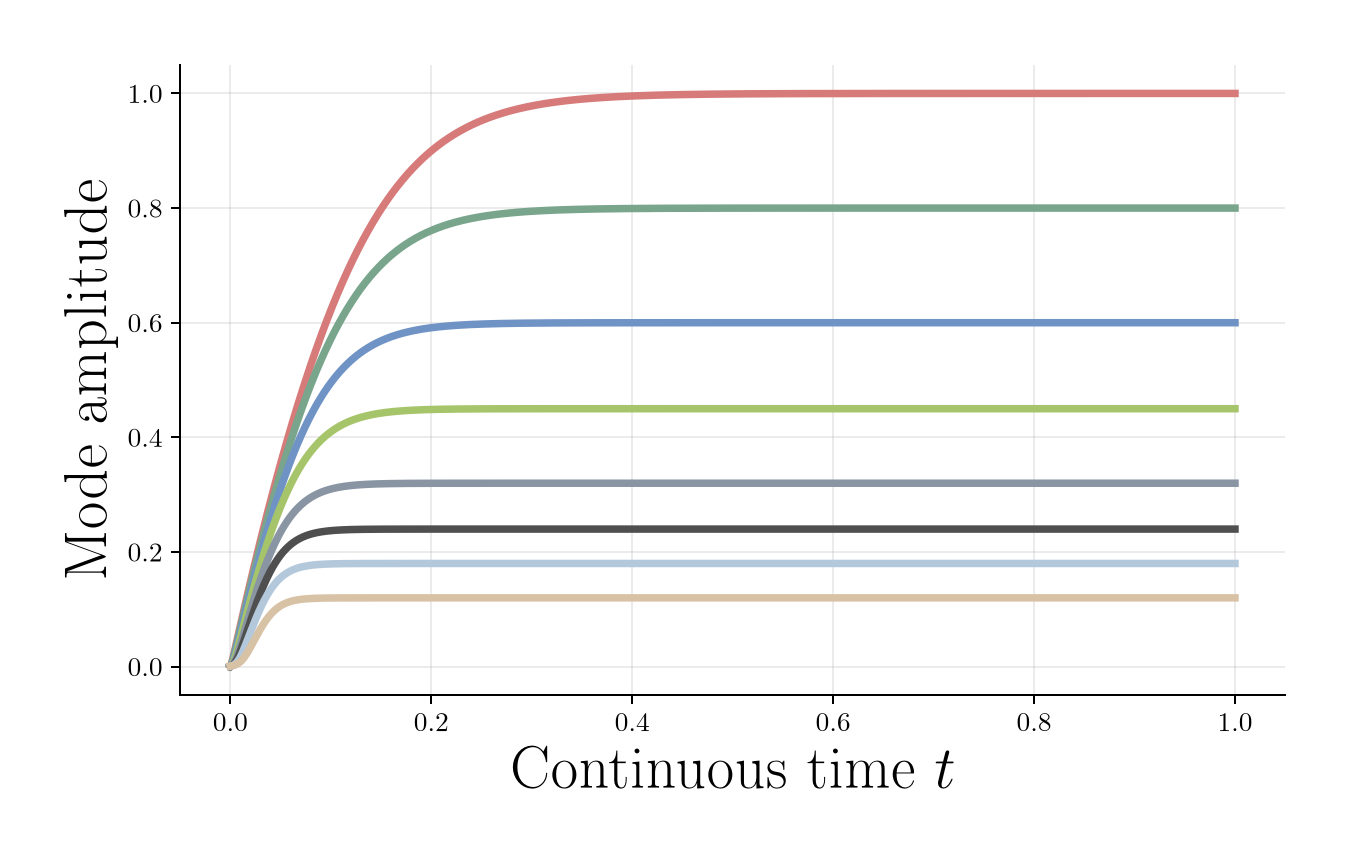}
        \caption{Shampoo}
        \label{fig:shampoo_mode}
    \end{subfigure}
    \begin{subfigure}{0.33\textwidth}
        \centering
        \includegraphics[width=\linewidth]{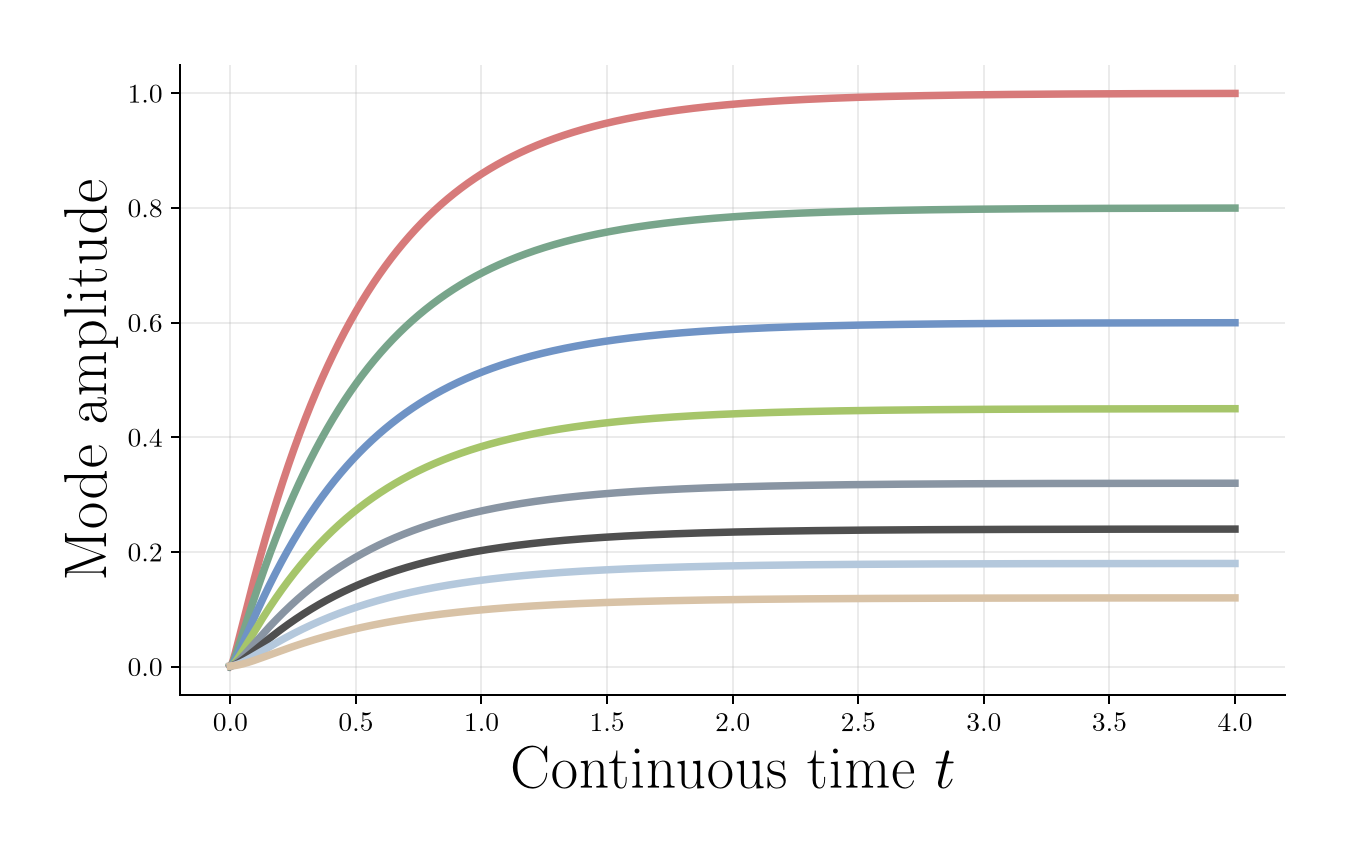}
        \caption{K-FAC (damped)}
        \label{fig:kfac_mode}
    \end{subfigure}
    \caption{\textbf{Singular mode trajectories show optimizer dependent spectral clocks for $L=2$.} }
    \label{fig:matrix_factorization_modes_l2}
\end{figure}

\begin{figure}[h!]
    \centering
    \includegraphics[width=0.32\linewidth]{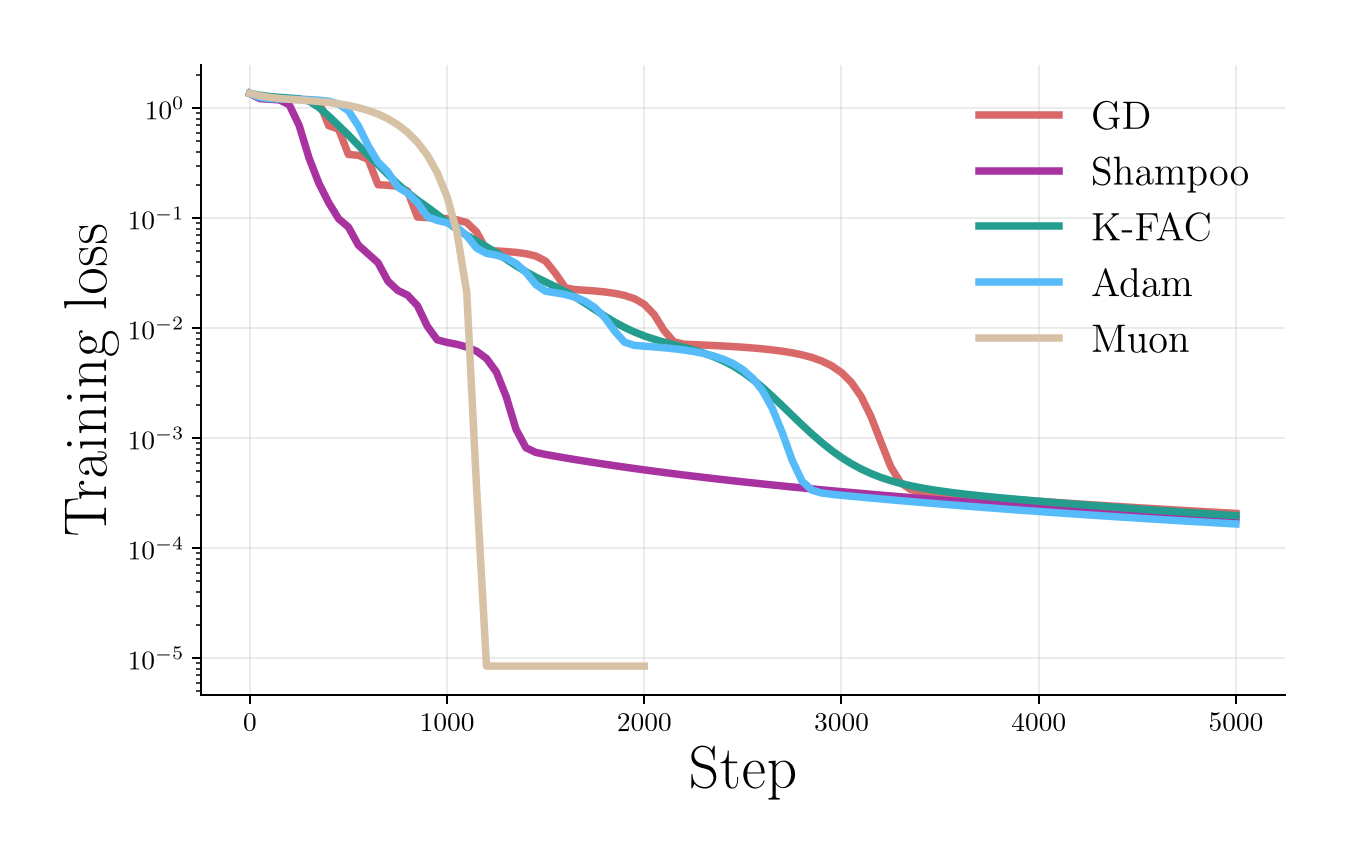}
    \includegraphics[width=0.32\linewidth]{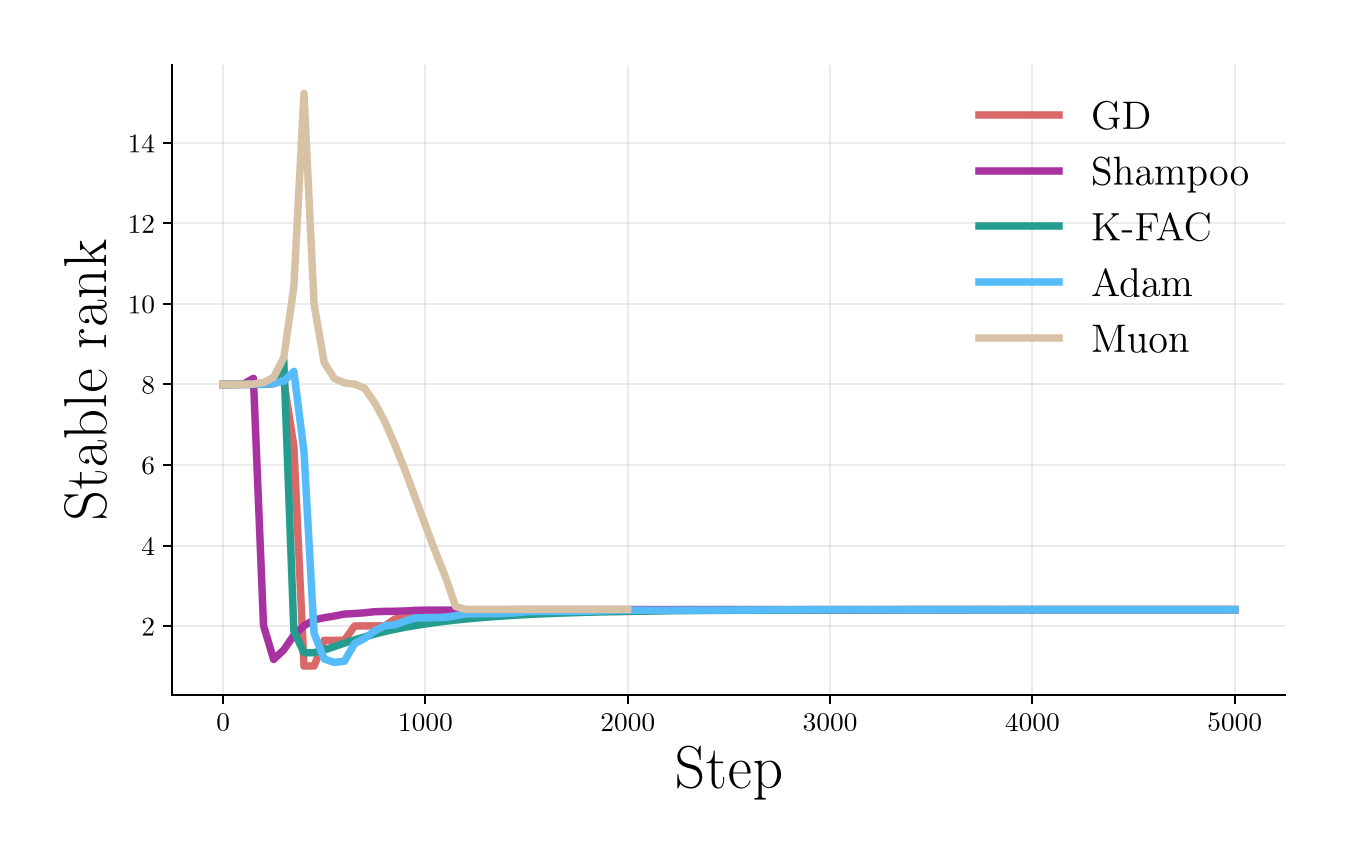}
    \includegraphics[width=0.32\linewidth]{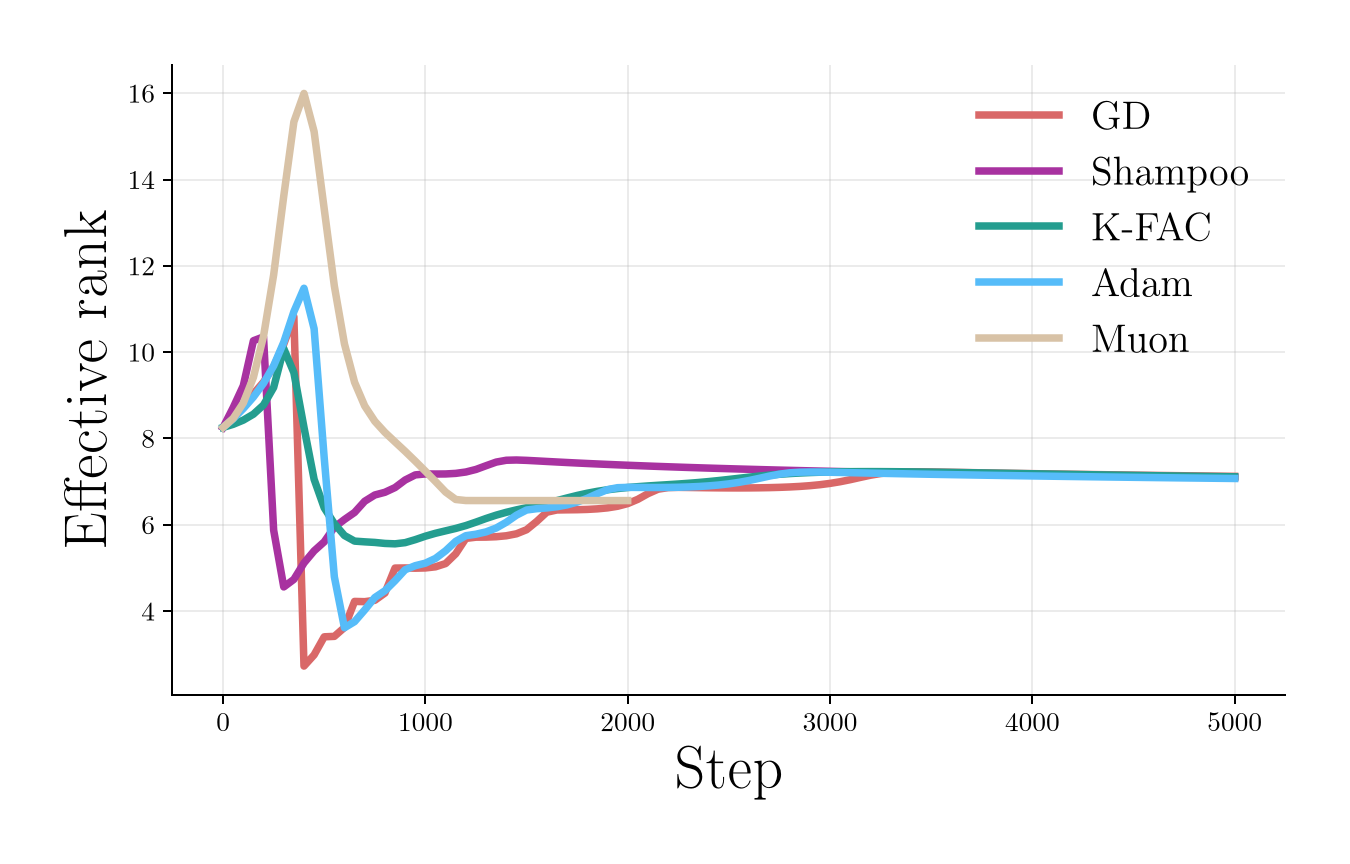}
    \caption{\textbf{Optimization and rank dynamics in full observation matrix factorization.}
All optimizers train the same depth $L=8$ deep linear factorization toward the same rank eight target and reach comparable final training loss.}
    \label{fig:L8}
\end{figure}

\begin{figure}[h!]
    \centering
    \begin{subfigure}{0.33\textwidth}
        \centering
        \includegraphics[width=\linewidth]{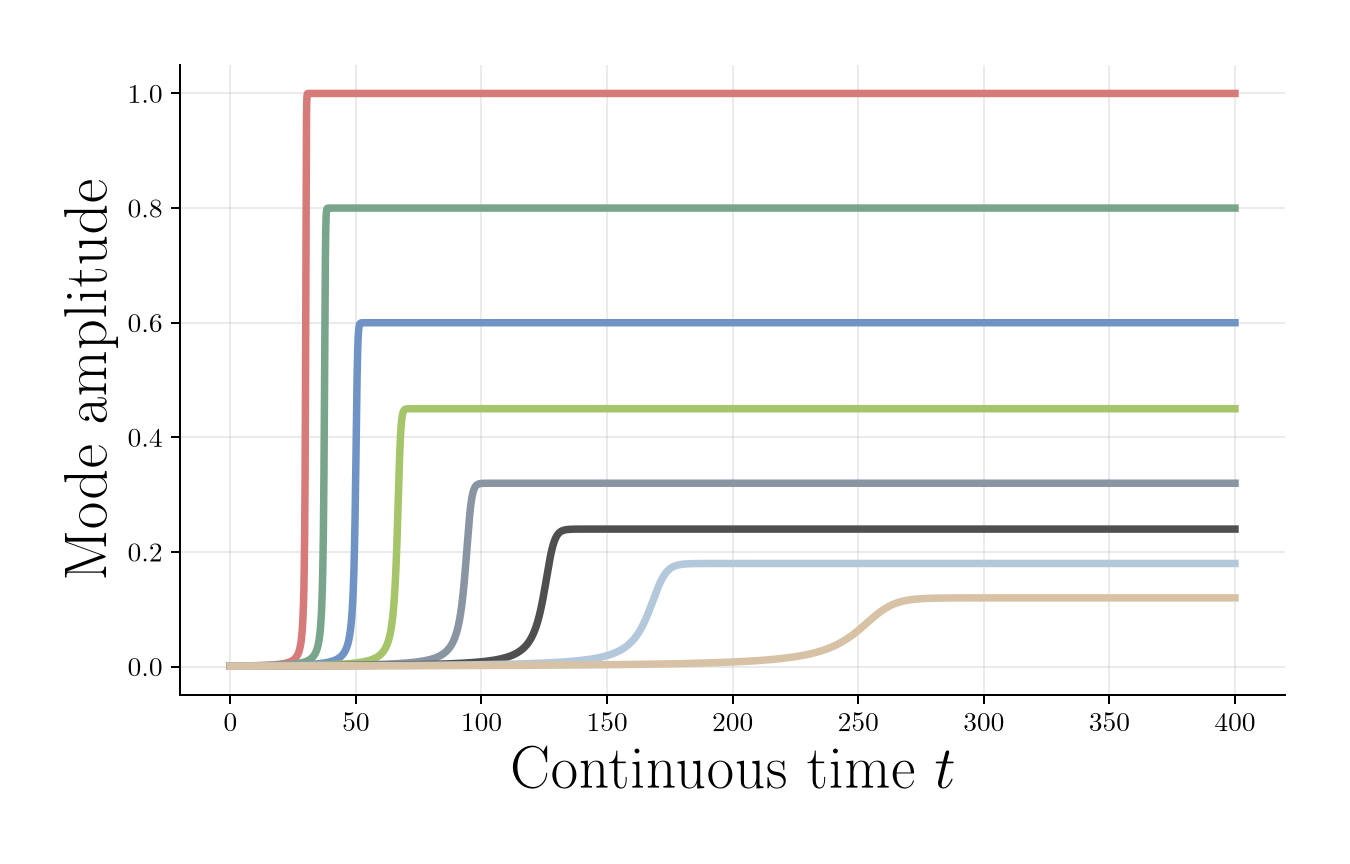}
        \caption{GD}
        \label{fig:gd_mode}
    \end{subfigure}%
    \begin{subfigure}{0.33\textwidth}
        \centering
        \includegraphics[width=\linewidth]{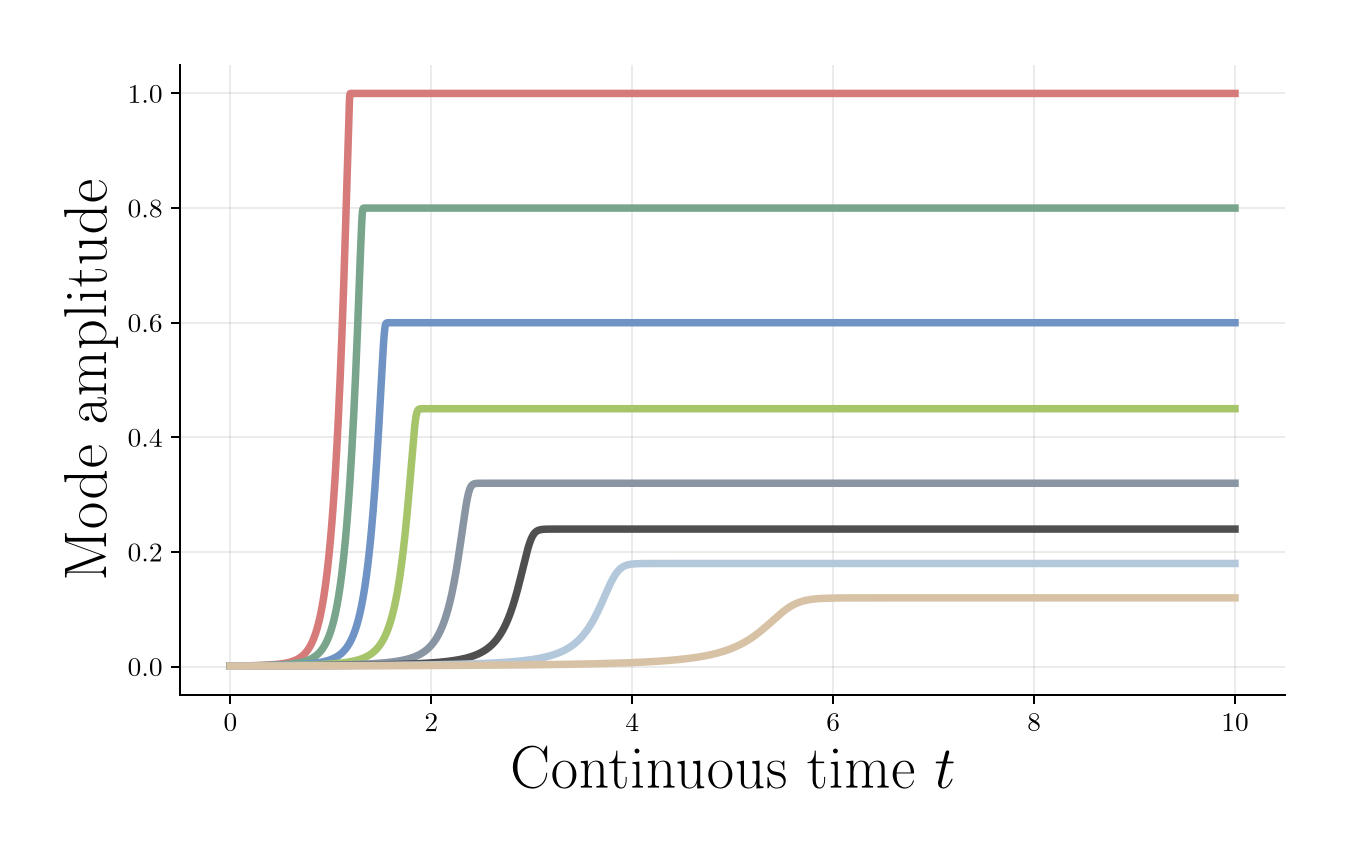}
        \caption{Adam}
        \label{fig:adam_mode}
    \end{subfigure}
    \begin{subfigure}{0.33\textwidth}
        \centering
        \includegraphics[width=\linewidth]{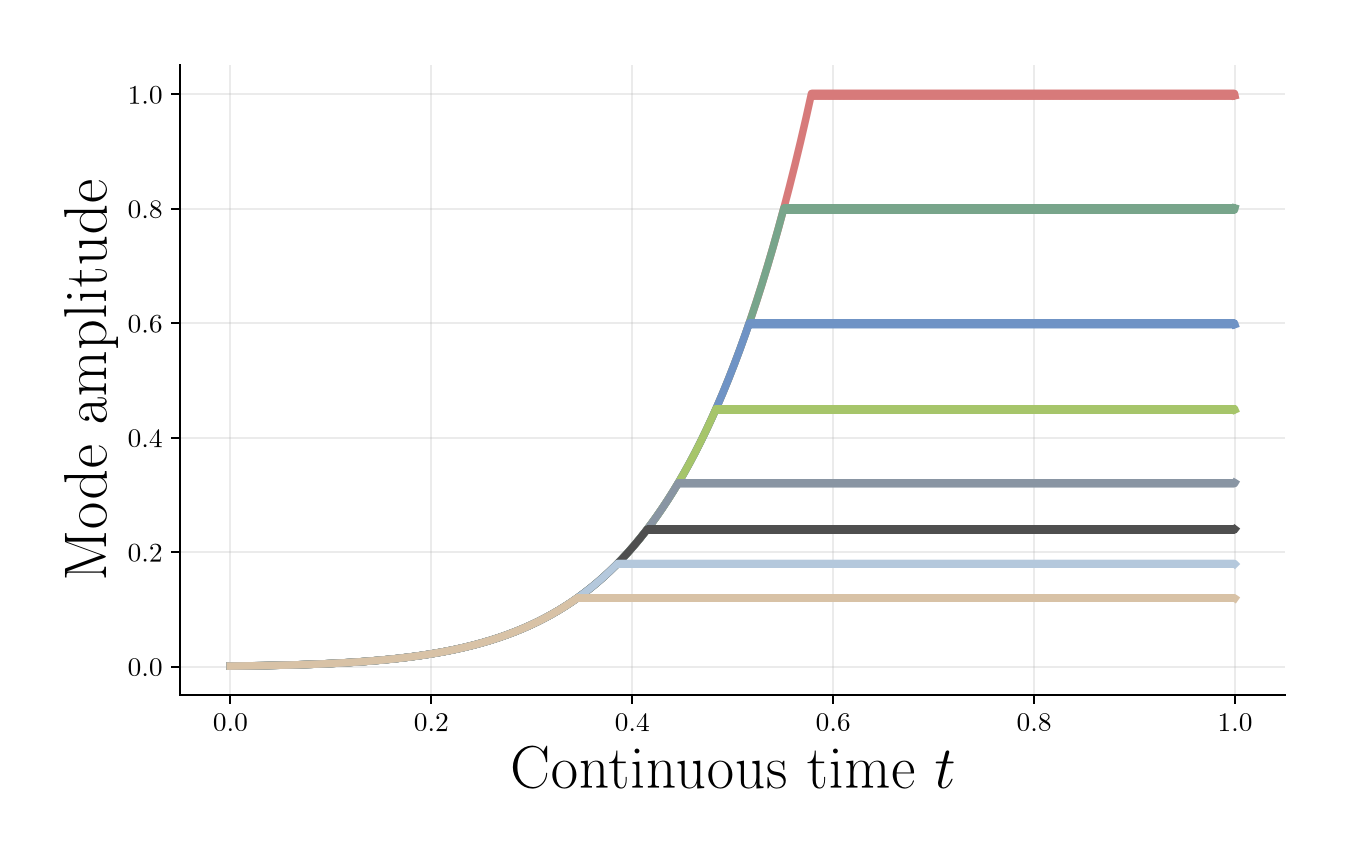}
        \caption{Muon}
        \label{fig:muon_mode}
    \end{subfigure}
    \begin{subfigure}{0.33\textwidth}
        \centering
        \includegraphics[width=\linewidth]{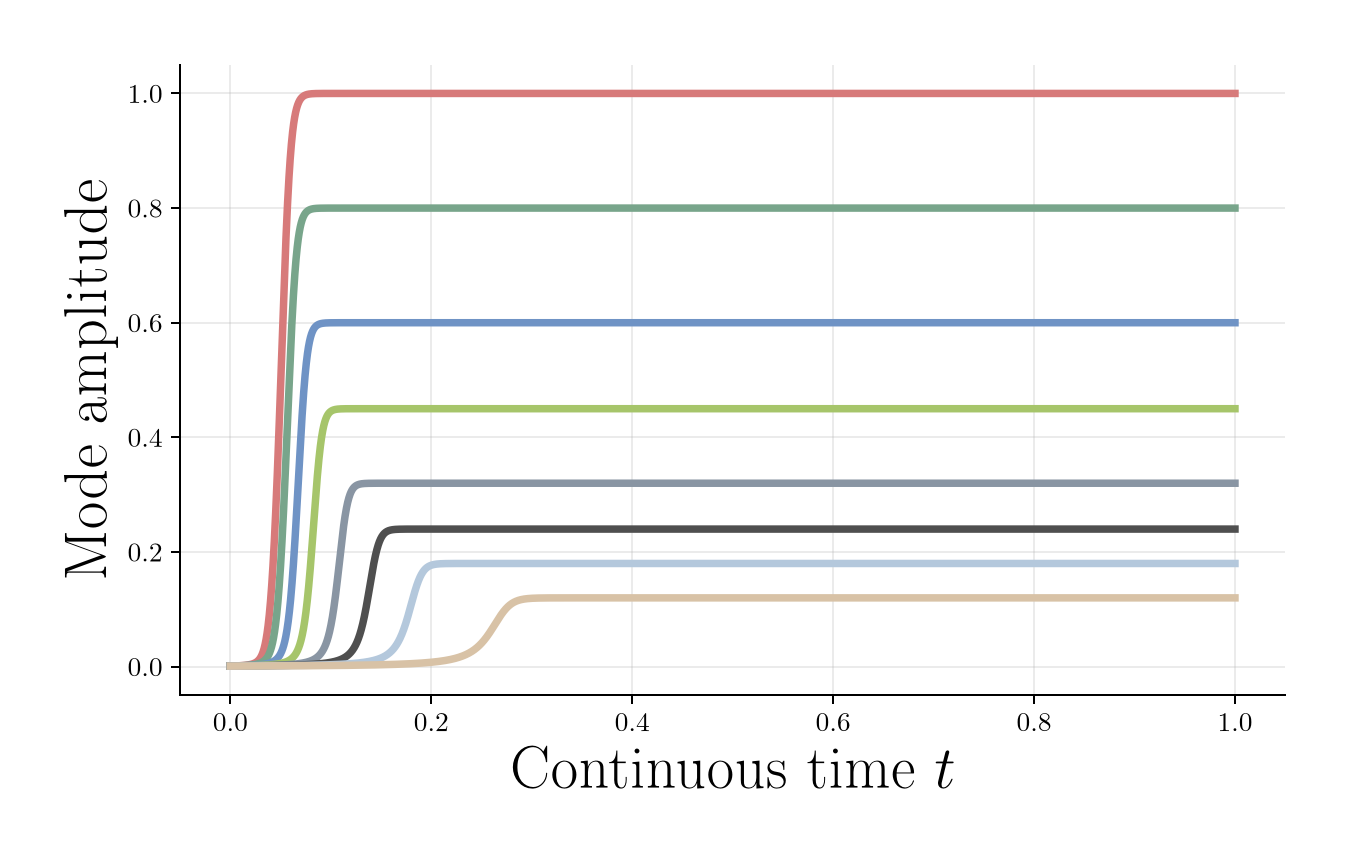}
        \caption{Shampoo}
        \label{fig:shampoo_mode}
    \end{subfigure}
    \begin{subfigure}{0.33\textwidth}
        \centering
        \includegraphics[width=\linewidth]{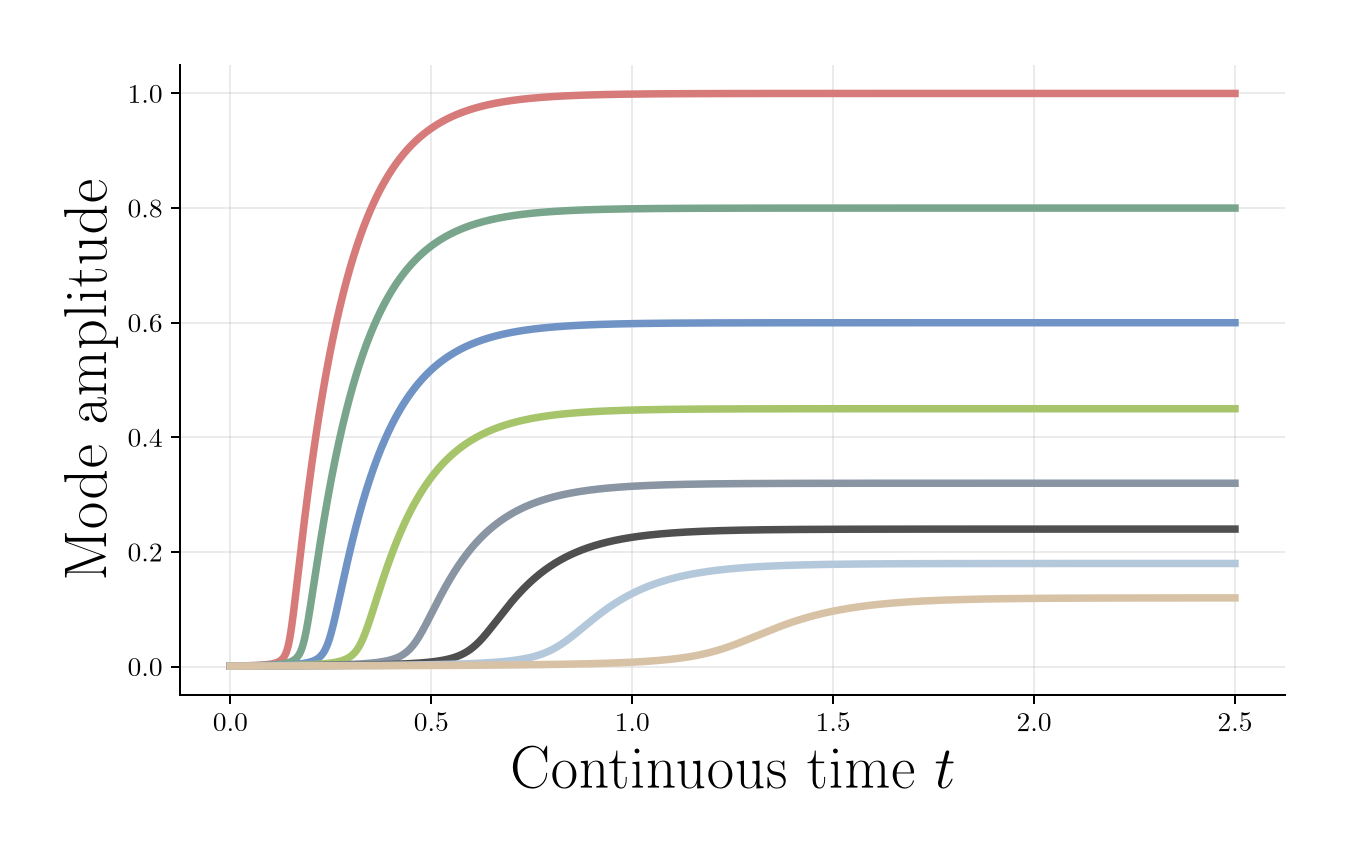}
        \caption{K-FAC (damped)}
        \label{fig:kfac_mode}
    \end{subfigure}
    \caption{\textbf{Singular mode trajectories show optimizer dependent spectral clocks for $L=8$.} 
    }
    \label{fig:matrix_factorization_modes_l8}
\end{figure}

\FloatBarrier

\subsubsection{Adam Phase Transitions}
In addition to the birth exponent plot in \Cref{fig:adam}, here we provide the validation of the theoretical decay exponent and saturation rate in \Cref{fig:adam_decay_beta}. Both align well with the theory.
For saturation rate, we choose to plot normalized residual $x=\dfrac{r\mu^q}{\epsilon}$ vs normalized saturation rate $\dfrac{\beta}{L\mu^p/\epsilon}$ for ease of representation. 
Theoretically,  near the saturation $\sigma \approx \mu$, so $\beta = \dfrac{L \mu^p}{r\mu^q + \epsilon}$. Therefore, the normalized saturation rate reduces to 
$$
\frac{\epsilon}{r\mu^q + \epsilon} = \frac{1}{1+x}.
$$
The normalized representation also clarifies why comparing the raw saturation rates can be misleading.
The limiting rate scales as $1/\epsilon$, so changing $\epsilon$ by several orders of magnitude produces the same change in the vertical scale even when the underlying dynamics are identical after rescaling.
In contrast, the normalized variables remove this trivial dependence.

\begin{figure}[h!]
    \centering
    \includegraphics[width=0.32\linewidth]{figures/iclr/adam/adam_state_transitions.pdf}  
    \includegraphics[width=0.32\linewidth]{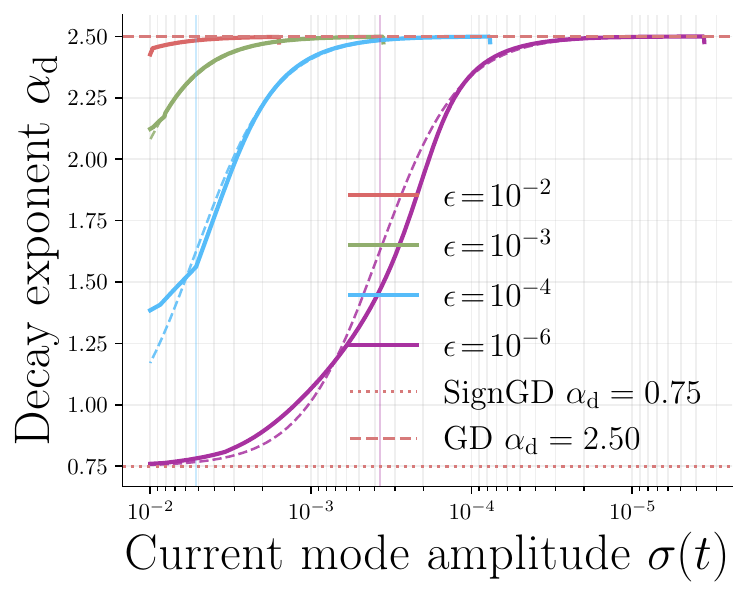}    \includegraphics[width=0.32\linewidth]{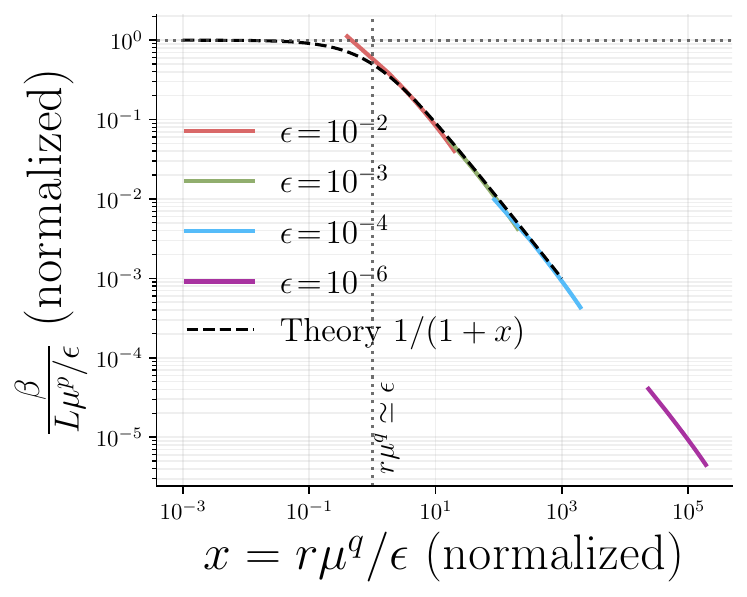}
    \caption{
\textbf{Birth, decay and normalized saturation dynamics of Adam.} Dotted and solid lines mark the theoretical and empirical values for $L=4$, respectively. These values align in both the decay and saturation rate plots across different values of $\epsilon$.
}
    \label{fig:adam_decay_beta}
    \vspace{-0.1cm}
\end{figure}
\FloatBarrier

\subsubsection{Shampoo Phase Transitions}
\Cref{fig:shampoo_phase_transitions} examines how the cumulative Shampoo state changes the birth, decay, and saturation dynamics.
The birth transition from the initial GD exponent $p$ toward the history adapted exponent $p-1$ is discussed in the main text.
Here we focus on the corresponding decay and saturation behavior.
For an unsupported mode with $\mu=0$, different values of $\lambda_{\rm Sh}$ produce different finite amplitude trajectories, but all curves approach the common asymptotic exponent $p+1=2.5$ for $L=4$.
Even in the case of smaller damping, the deviation from $2.5$ is only in the second decimal place.
Thus Shampoo retains a polynomial tail near zero rather than the finite time extinction of ideal Muon.
For a positive target mode, the right panel shows each damping level approaches a distinct constant rate, with the empirical curves closely following the finite amplitude prediction.
The limiting rate depends on the entire accumulated gradient history and this distinguishes Shampoo from memoryless spectral normalization and explains why modes with similar current residuals can saturate at different rates.
\vspace{-0.1cm}
\begin{figure}[h!]
    \centering
    \includegraphics[width=0.32\linewidth]{figures/iclr/shampoo/shampoo_state_transitions.pdf}
    \includegraphics[width=0.32\linewidth]{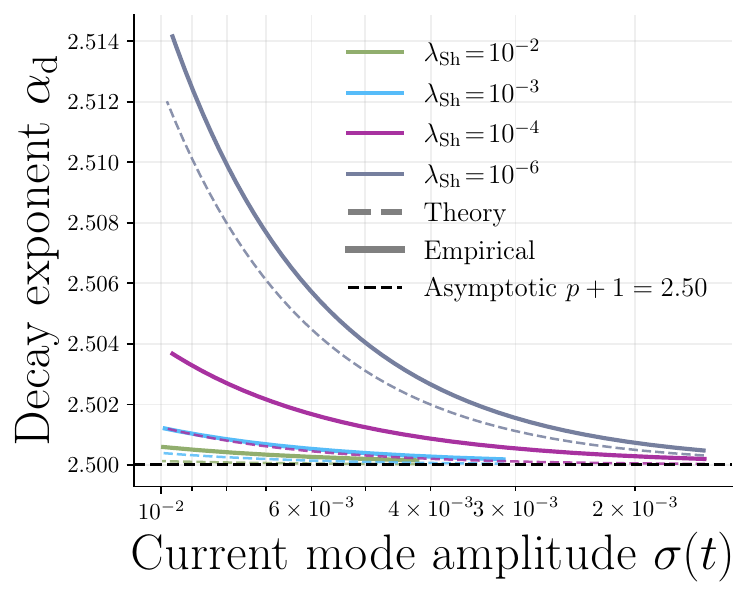}
    \includegraphics[width=0.32\linewidth]{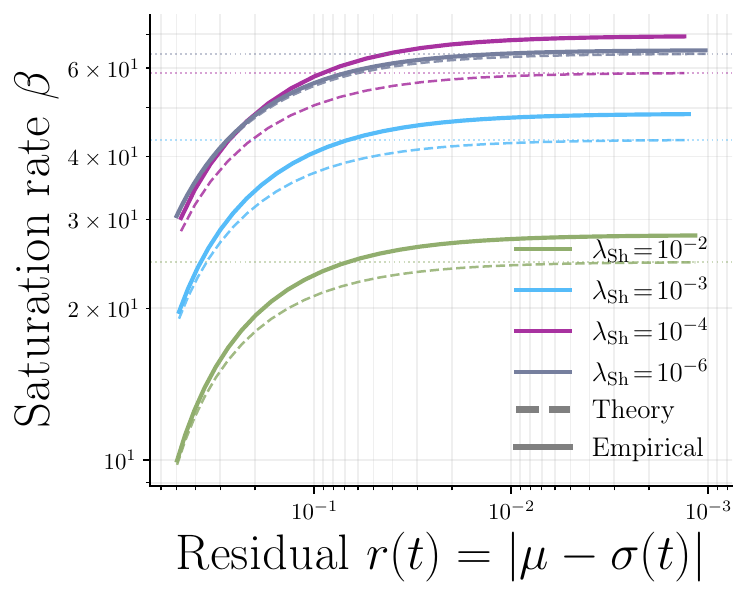}
    \caption{
\textbf{Phase transitions induced by cumulative Shampoo state.}
Solid curves show empirical dynamics and dashed curves show the corresponding aligned theory for depth $L=4$.
Together, the panels show that cumulative Shampoo changes its effective spectral geometry during training while retaining explicit memory of earlier gradient activity.
}
\label{fig:shampoo_phase_transitions}
\end{figure}

\FloatBarrier

\subsubsection{K-FAC Phase Transitions}
\label{sec:exp_kfac_transitions}

\Cref{fig:kfac_phase_transitions} illustrates how additive damping in K-FAC controls the spectral dynamics across all three phases.
For birth, positive damping restores the small mode barrier. 
Reducing $\lambda_{\rm KF}$ shifts this transition toward smaller mode amplitudes, recovering the undamped KFAC regime.

For an unsupported mode, the middle panel confirms that stronger damping causes the polynomial regime to appear at larger amplitudes, following the theory.

Near a positive target, K-FAC's saturation rate is nearly constant along the final approach to the target and approaches the undamped value $L=4$ as $\lambda_{\rm KF}\rightarrow0$.
Thus, K-FAC damping controls not only numerical regularization of the curvature inverse, but also the spectral scale below which GD-like dynamics reappear.

\begin{figure}[h!]
    \centering
    \includegraphics[width=0.32\linewidth]{figures/iclr/kfac/kfac_state_transitions.pdf}
    \includegraphics[width=0.32\linewidth]{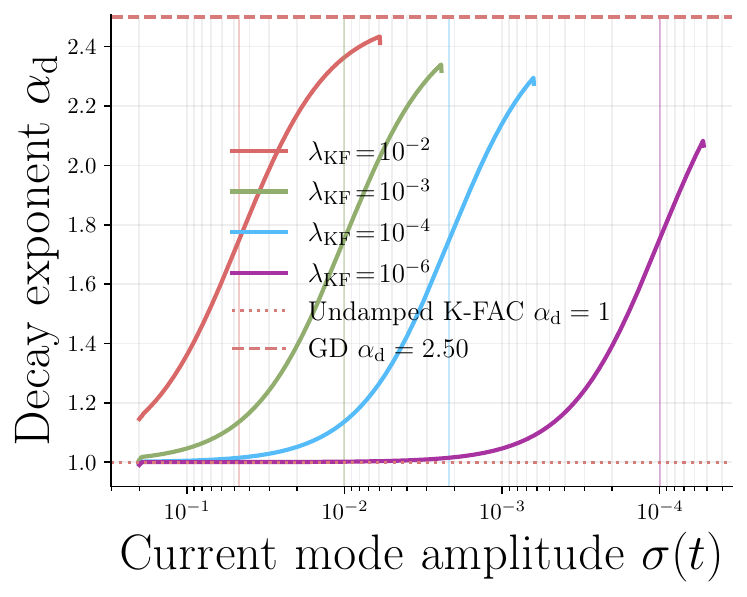}
    \includegraphics[width=0.32\linewidth]{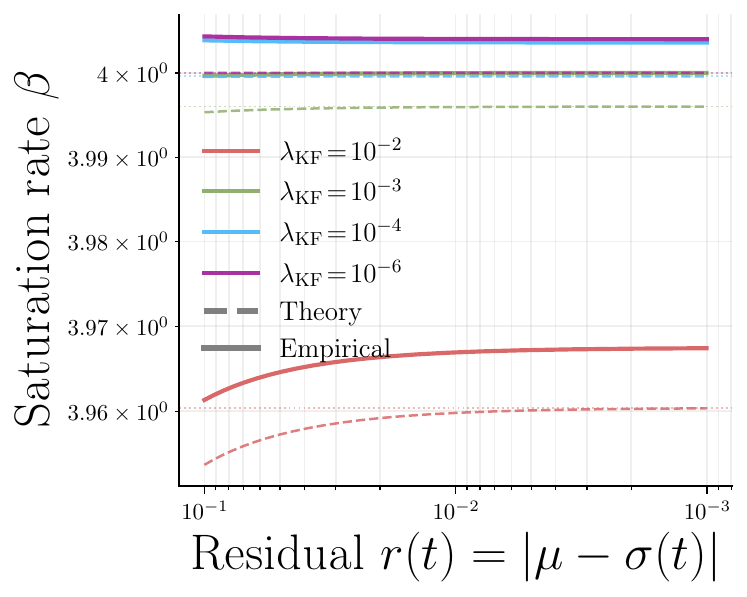}
    \caption{\textbf{Additive K-FAC damping controls the spectral phase transitions.}
Solid curves show empirical dynamics and dashed curves show the corresponding aligned theory for depth $L=4$.
It shows that $\lambda_{\rm KF}$ defines a spectral mobility threshold separating curvature corrected dynamics from GD-like behavior.
}
    \label{fig:kfac_phase_transitions}
\end{figure}

\FloatBarrier

\subsection{Linear Regression}
\label{sec:exp_linear_regression}

We next consider a less structured setting in which the effective predictor is trained by underdetermined linear regression as described in \Cref{app:exp_settings}.
This experiment tests whether the qualitative predictions persist once our simplifying assumptions \Cref{assump:balancedness,assump:aligned_modes,assump:compatible_updates} are removed.

\Cref{fig:linear_regression_summary} shows that the optimizers reach comparable training losses and converge to nearly identical final singular spectra.
This agreement is important: the pronounced differences observed during training are not simply consequences of one method converging to a substantially different level of fit.
Nevertheless, the stable and effective ranks evolve differently.
Thus similar terminal predictors need not imply similar spectral learning dynamics.

\begin{figure}[h!]
    \centering
    \includegraphics[width=0.45\linewidth]{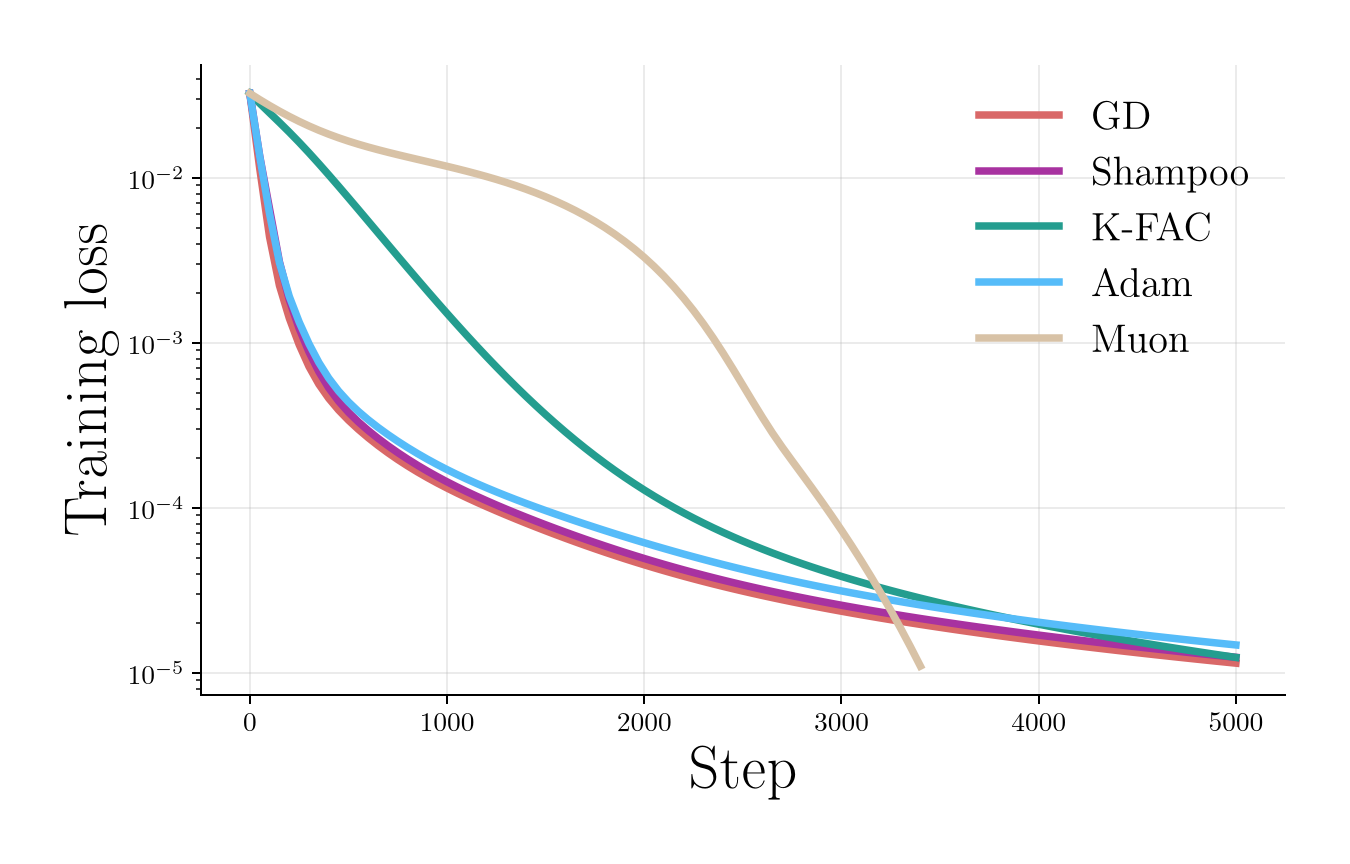}
    \includegraphics[width=0.45\linewidth]{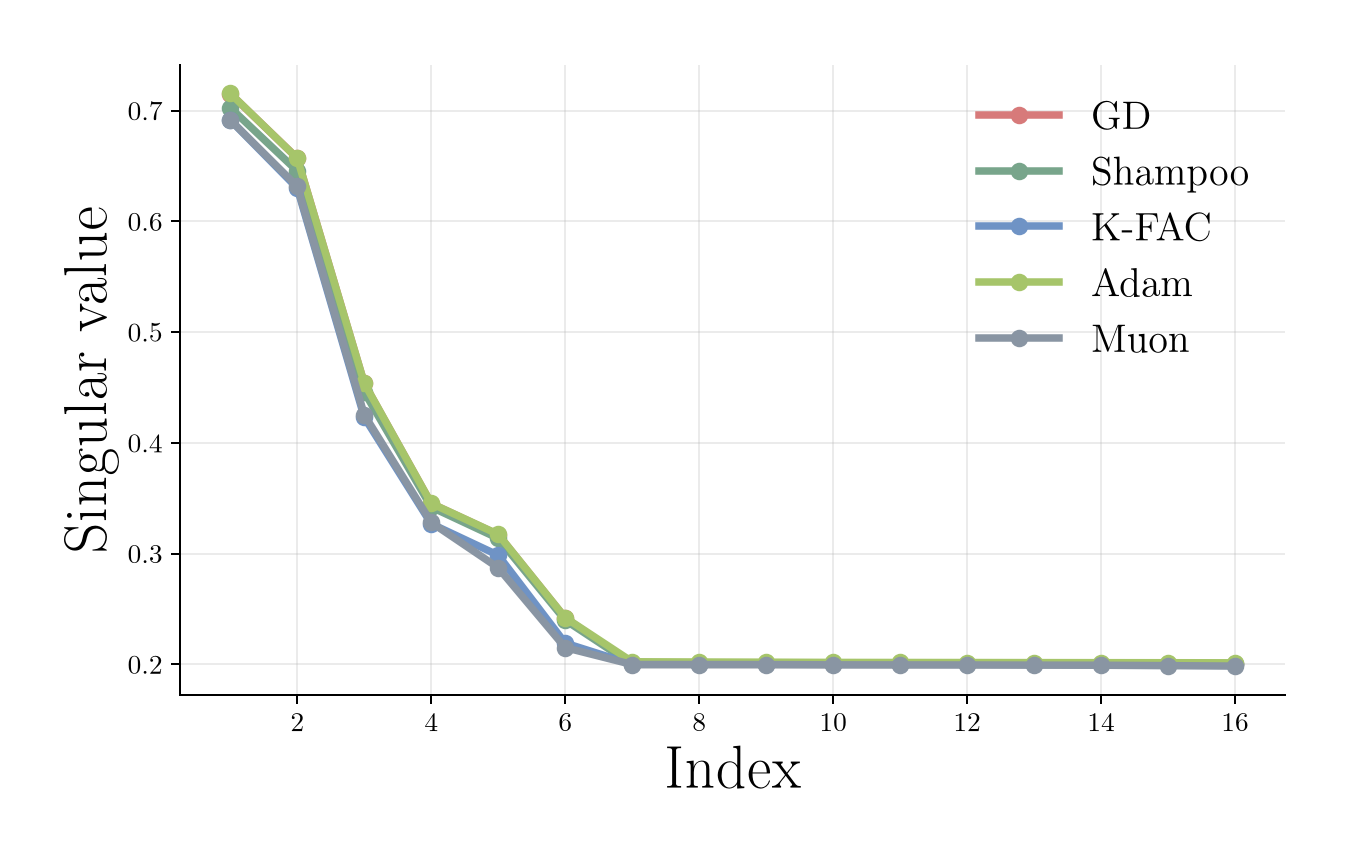}
    \includegraphics[width=0.45\linewidth]{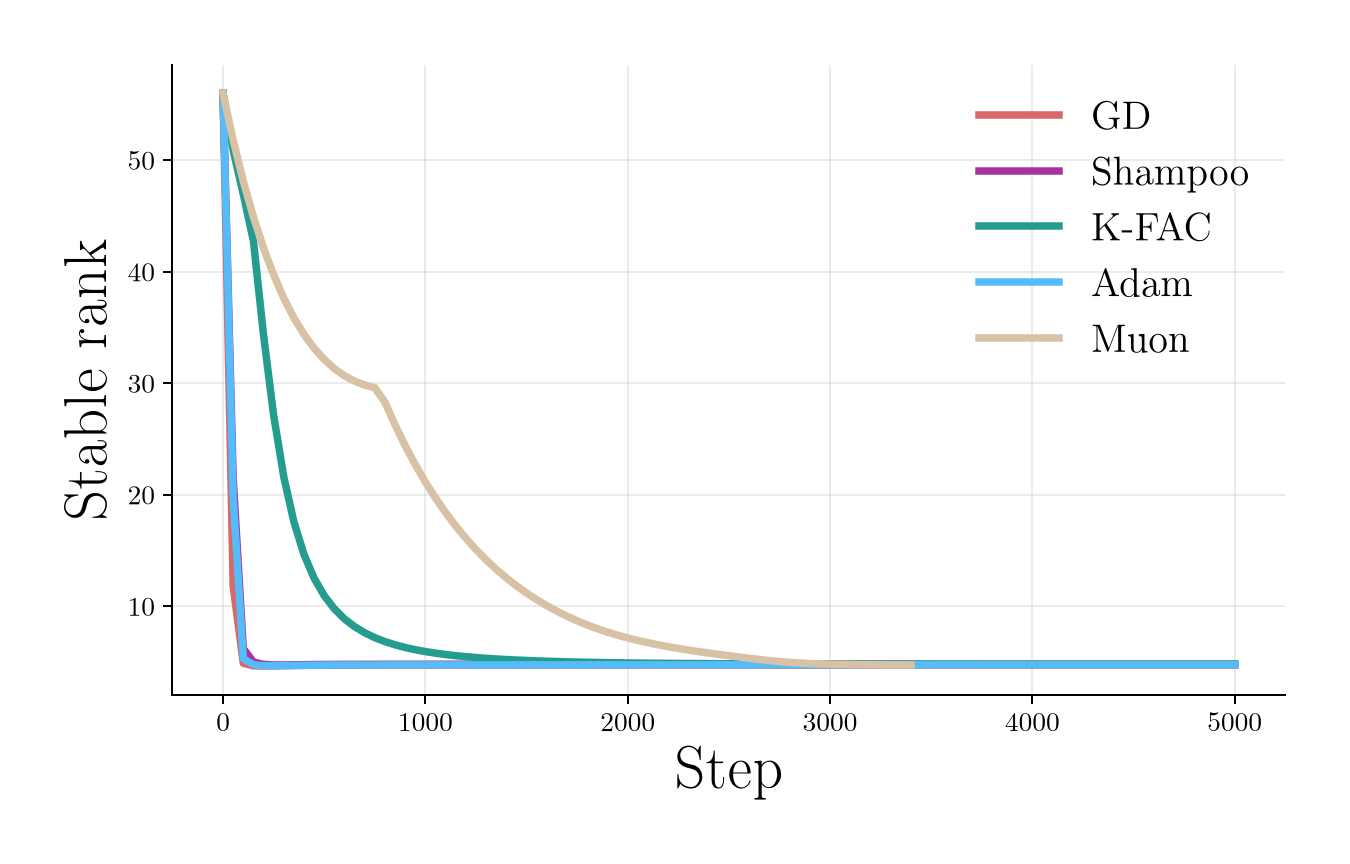}
    \includegraphics[width=0.45\linewidth]{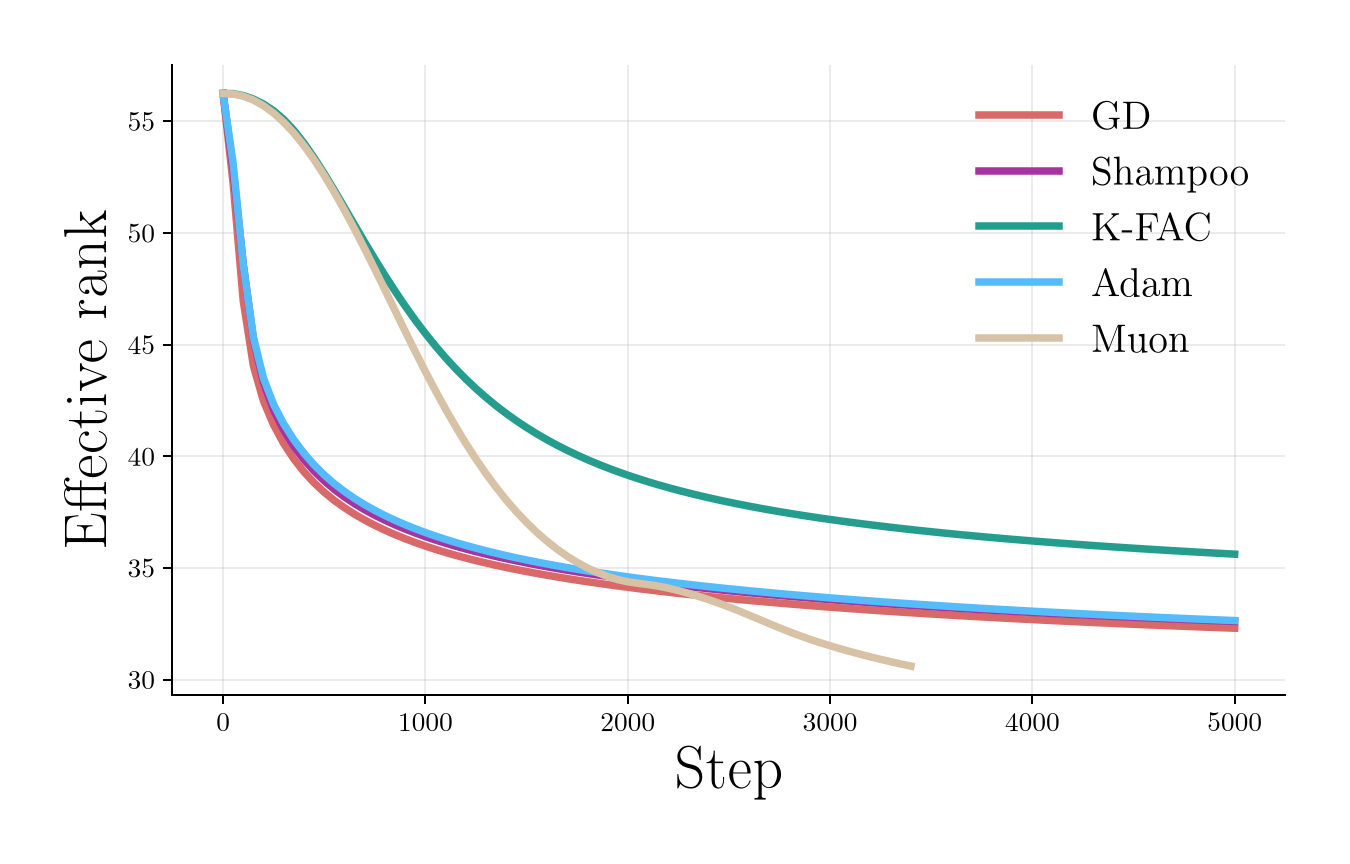}
    \caption{\textbf{Optimizer dependent dynamics persist in underdetermined linear regression.}
We train the same deep linear predictor on an underdetermined, nonwhitened regression problem using GD, Adam, Muon, cumulative Shampoo, and KFAC.
All methods reach comparable final loss and final spectra of $W$ despite substantially different transient dynamics shown by the stable rank and effective rank throughout training.
The optimizer geometry strongly changes how quickly spectral mass concentrates. 
}
\label{fig:linear_regression_summary}
\end{figure}

The singular value trajectories in \Cref{fig:linear_regression_singular_trajectories} make this distinction particularly clear. 
GD exhibits pronounced separation between the emergence times of successive singular values, consistent with the factorization induced small mode barrier derived in the aligned setting.
Adam substantially compresses these delays.
Muon produces a markedly more synchronized growth of the singular spectrum, consistent with normalization reducing mode dependent mobility differences.
Cumulative Shampoo also accelerates weaker spectral components after its preconditioner has accumulated sufficient history, while K-FAC produces comparatively smooth growth across the spectrum through curvature correction.
These effects remain visible even though the singular vectors are not fixed and the data covariance couples the modes.

The experiment therefore supports the interpretation of the theoretical phase laws as mechanisms governing transient spectral learning rather than merely properties of the special aligned model.
Exact exponents depend on the aligned reduction, but the resulting qualitative distinctions 
persist in the more practical underdetermined regression problem.

\begin{figure}[h!]
    \centering
    \begin{subfigure}{0.33\textwidth}
        \centering
        \includegraphics[width=\linewidth]{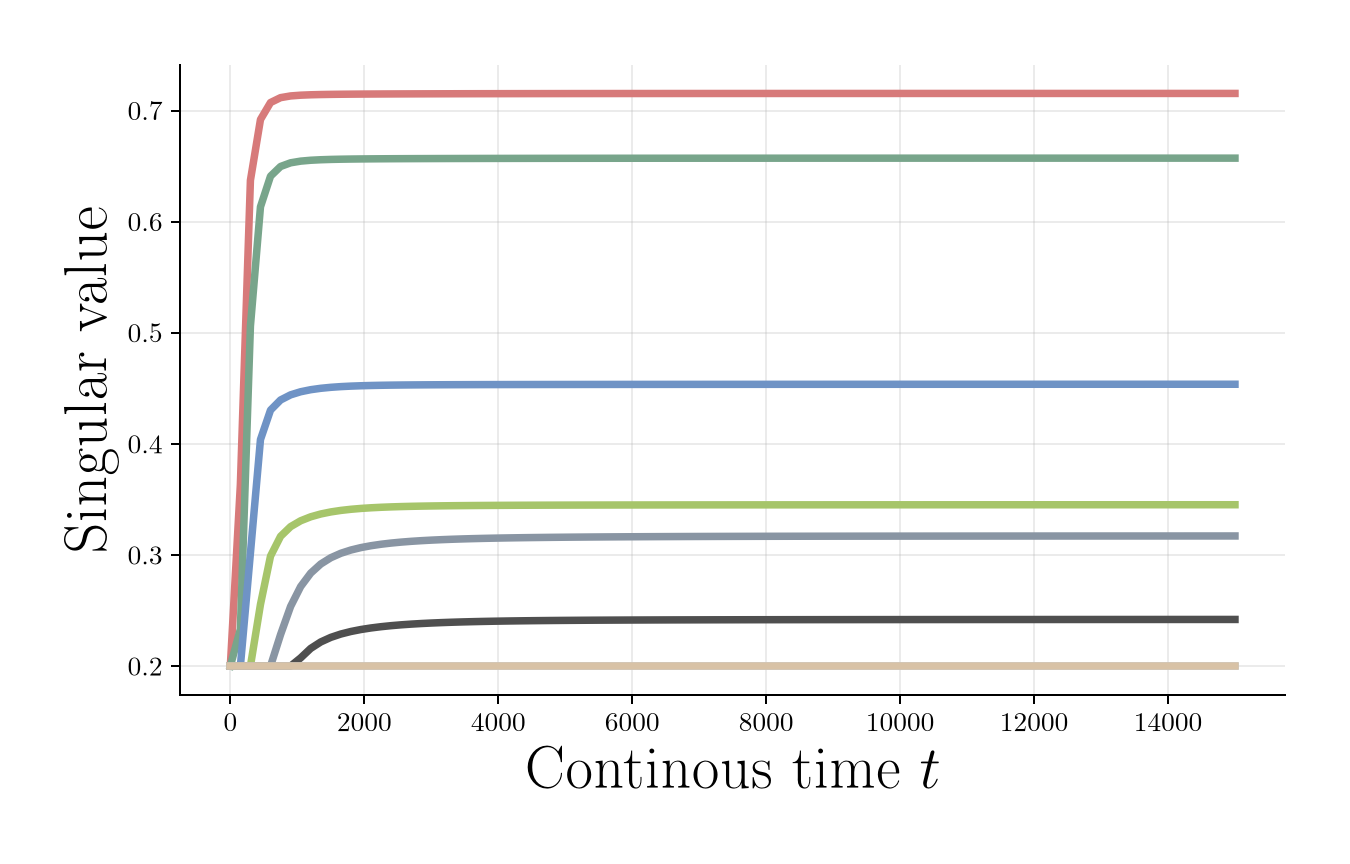}
        \caption{GD}
        \label{fig:gd_mode}
    \end{subfigure}%
    \begin{subfigure}{0.33\textwidth}
        \centering
        \includegraphics[width=\linewidth]{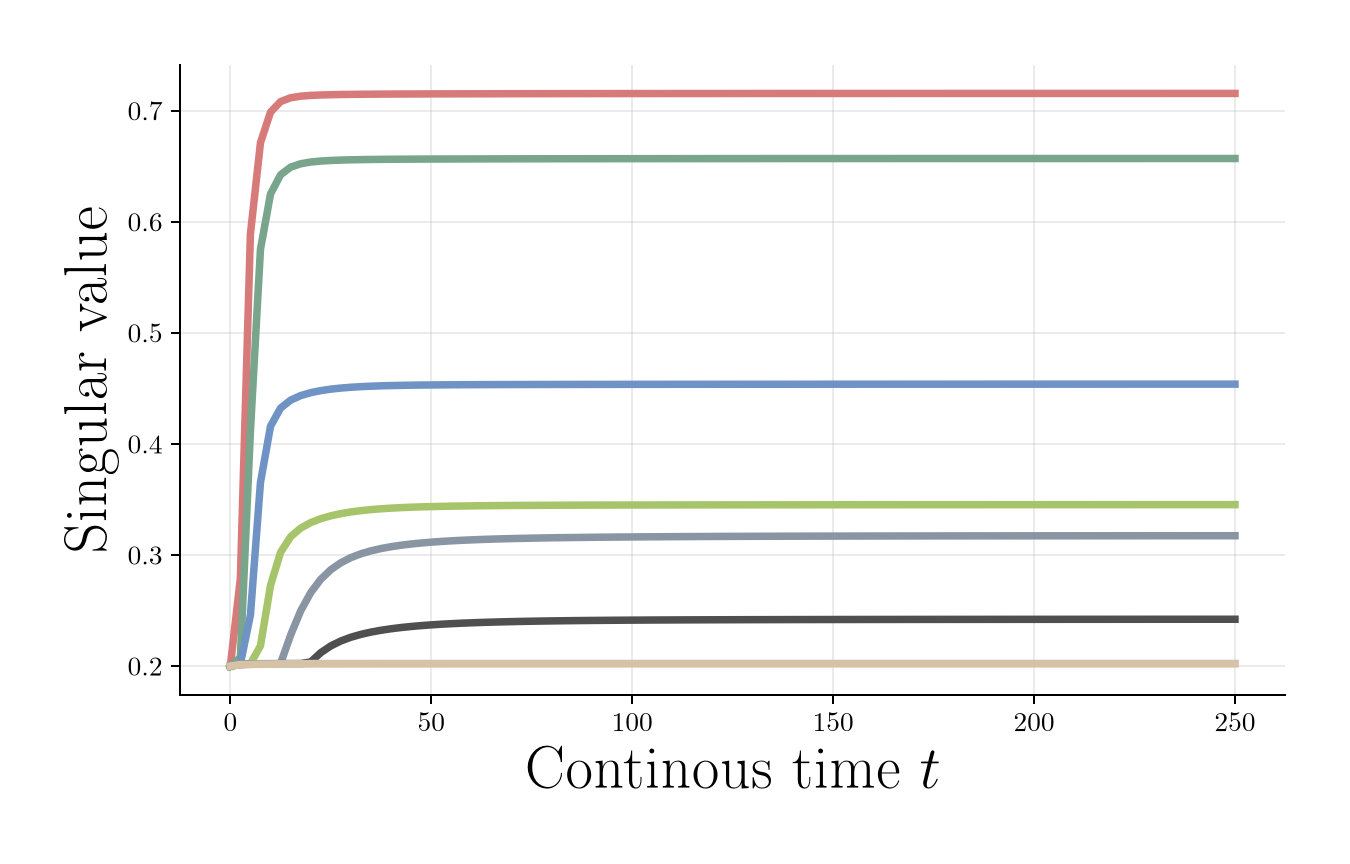}
        \caption{Adam}
        \label{fig:adam_mode}
    \end{subfigure}
    \begin{subfigure}{0.33\textwidth}
        \centering
        \includegraphics[width=\linewidth]{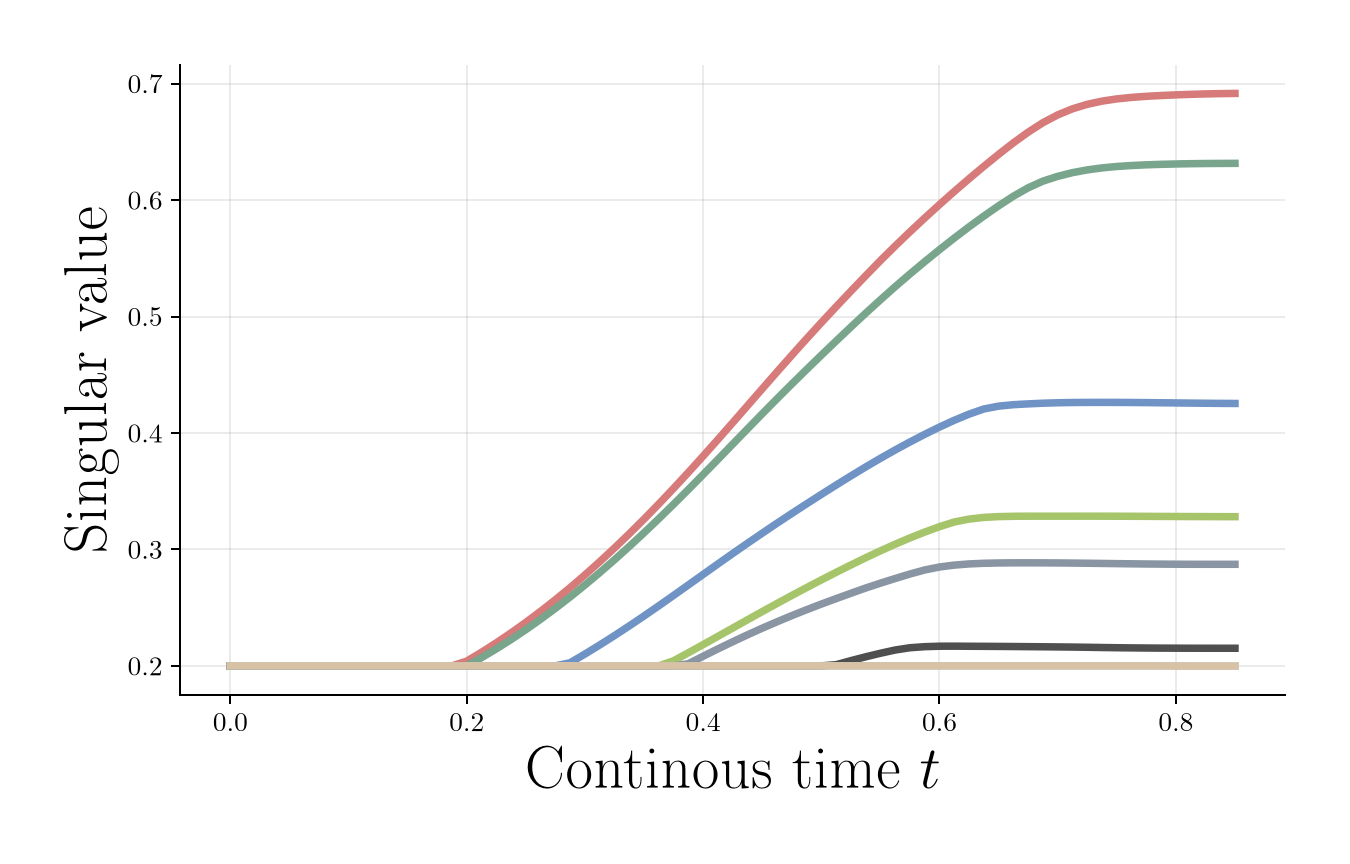}
        \caption{Muon}
        \label{fig:muon_mode}
    \end{subfigure}
    \begin{subfigure}{0.33\textwidth}
        \centering
        \includegraphics[width=\linewidth]{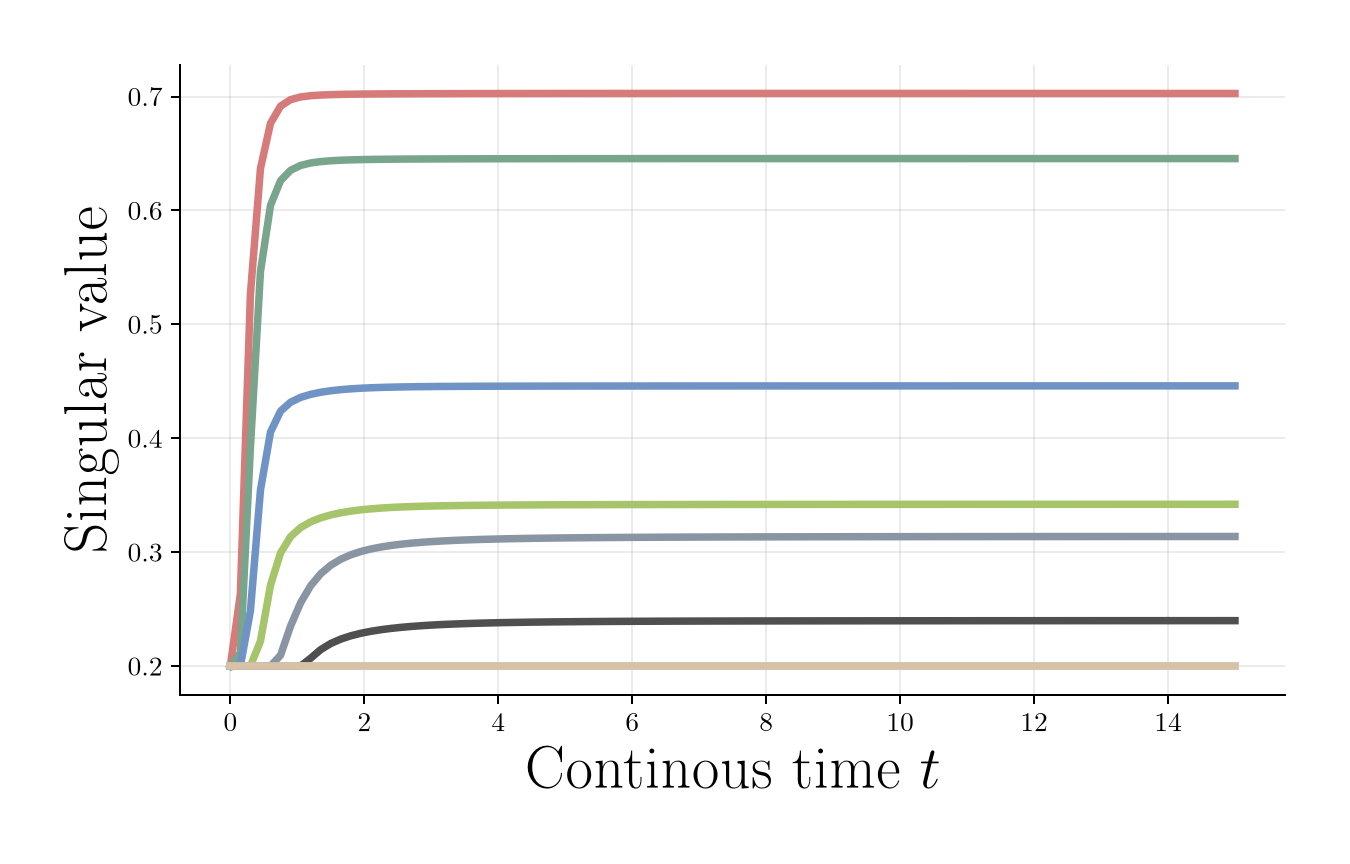}
        \caption{Shampoo}
        \label{fig:shampoo_mode}
    \end{subfigure}
    \begin{subfigure}{0.33\textwidth}
        \centering
        \includegraphics[width=\linewidth]{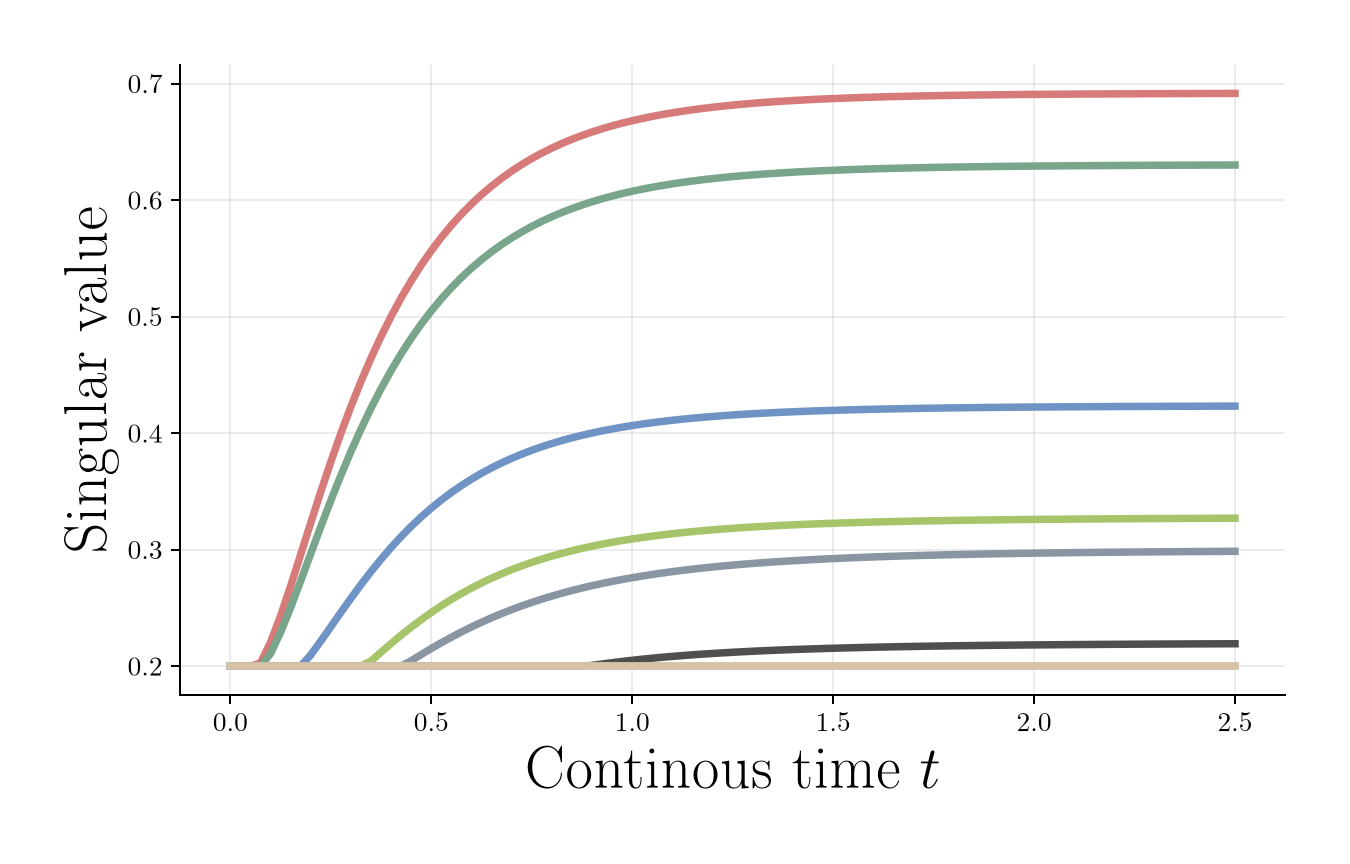}
        \caption{K-FAC}
        \label{fig:kfac_mode}
    \end{subfigure}
    \caption{\textbf{Singular value trajectories in nonaligned underdetermined linear regression.}
Although all methods approach similar final singular spectra, their transient spectrum differ substantially.
The qualitative optimizer dependent dynamics predicted by the aligned phase analysis remain visible even when the empirical data covariance couples singular directions and the exact scalar mode reduction no longer holds.
}
\label{fig:linear_regression_singular_trajectories}
\end{figure}

\FloatBarrier

\end{document}